\documentclass[lettersize,journal]{IEEEtran}
\usepackage[numbers, super]{natbib}
\usepackage{amsmath,amsfonts}
\usepackage{algorithmic}
\usepackage{algorithm}
\usepackage{array}
\usepackage[caption=false,font=normalsize,labelfont=sf,textfont=sf]{subfig}
\usepackage{textcomp}
\usepackage{stfloats}
\usepackage{url}
\usepackage{verbatim}
\usepackage{graphicx}
\usepackage{mathtools}

\usepackage[nice]{nicefrac}
\usepackage{amsmath,amsthm}
\usepackage{subfig}
\usepackage[raggedleft]{sidecap} 
\sidecaptionvpos{figure}{t} 

\usepackage{xr}
\theoremstyle{definition}

\theoremstyle{plain}
\newtheorem{theorem}{Theorem}
\newtheorem{assmp}{Assumption}
\newtheorem{corollary}{Corollary}[theorem]
\newtheorem{claim}{Claim}[theorem]
\newtheorem{lemma}{Lemma}
\newcommand{\phiperp}[1]{\phi^{\perp}_{n+1,#1}}
\newcommand{\phicomp}{\phi^{\perp}_{n+1}}
\newcommand{\proj}[2]{\mathcal{P}_{\mathcal{S}(\bigcup\limits_{i=1}^{#1}\{\phi_i\})}(#2)}
\newcommand{\projalt}[3]{\mathcal{P}_{\mathcal{S}(\bigcup\limits_{#1}^{#2}\{\phi_i\})}(#3)}
\newcommand{\projexp}[2]{\mathcal{P}_{\mathcal{S}(\{\phi_1,...\phi_{#1}\})}(#2)}
\begin{document}

\title{ROBUST ONLINE RECONSTRUCTION OF CONTINUOUS-TIME SIGNALS FROM A LEAN SPIKE TRAIN ENSEMBLE CODE}

\author{Anik Chattopadhyay, Arunava Banerjee, ~\IEEEmembership{Member,~IEEE,}
\thanks{A preliminary version except for the new Section \ref{windowing} has appeared in \citep{10096734}.}}

\markboth{Journal of \LaTeX\ Class Files,~Vol.~14, No.~8, August~2021}%
{Shell \MakeLowercase{\textit{et al.}}: A Sample Article Using IEEEtran.cls for IEEE Journals}


\maketitle

\begin{abstract}
Sensory stimuli in animals are encoded into spike trains by neurons, offering advantages such as sparsity, energy efficiency, and high temporal resolution. This paper presents a signal processing framework that deterministically encodes continuous-time signals into biologically feasible spike trains, and addresses the questions about representable signal classes and reconstruction bounds. The framework considers encoding of a signal through spike trains generated by an ensemble of neurons using a convolve-then-threshold mechanism with various convolution kernels. A closed-form solution to the inverse problem, from spike trains to signal reconstruction, is derived in the Hilbert space of shifted kernel functions, ensuring sparse representation of a generalized Finite Rate of Innovation (FRI) class of signals. Additionally, inspired by real-time processing in biological systems, an efficient iterative version of the optimal reconstruction is formulated that considers only a finite window of past spikes, ensuring robustness of the technique to ill-conditioned encoding; convergence guarantees of the windowed reconstruction to the optimal solution are then provided. Experiments on a large audio dataset demonstrate excellent reconstruction accuracy at spike rates as low as one-fifth of the Nyquist rate, while showing clear competitive advantage in comparison to state-of-the-art sparse coding techniques in the low spike rate regime.
\end{abstract}



\section{Introduction}
\label{sec:intro}
In most animals, sensory stimuli are communicated to the brain via ensembles of discrete, spatio-temporally compact electrical events generated by neurons, known as action potentials or spikes \cite{SpikesRiekeEtAl}. The conversion of continuous-time stimuli to spike trains occurs at an early stage of sensory processing, such as in the retinal ganglion cells in the visual pathway or spiral ganglion cells in the auditory pathway. Nature likely resorts to spike-based encoding due to several advantages: sparsity of representation \citep{Olshausen1996} and energy efficiency \citep{Laughlin2003}, noise robustness \citep{London2005}, high temporal precision \citep{Buzsaki2006} and facilitation of downstream computation \citep{Foldiak1990, Graham}. Olshausen and Field \citep{Olshausen1996} demonstrated how efficient codes could arise from learning sparse representations of natural stimuli, resulting in striking similarities with observed biological receptive fields. Similarly, Smith and Lewicki \citep{Lewicki2002, CSmith2006} showed that auditory filters could be estimated by training a population spike code model with natural sounds.  While  these studies highlight important aspects of spike-based encoding, they rely on existing dictionary learning techniques (e.g. matching pursuit or basis pursuit) to obtain the sparse spike codes, raising the question of the biological feasibility of computing such codes. The formal study of encoding continuous-time sensory stimuli via biologically plausible spiking neurons falls under the field of neural coding. Based on how neural spike responses are represented, studies can be broadly divided into two categories \citep{survey2023encoding}: 1) rate coding, where spike train responses are converted into an average rate, and 2) temporal coding, where the precise timing of spikes convey information about the stimuli. In the rate coding literature, spike responses to stimuli are converted to an average instantaneous rate $r(t)$, and stimulus reconstruction is typically formulated probabilistically by choosing the stimulus $s$ that maximizes the likelihood $P(s|r)$. Rate coding is criticized for losing temporal precision, especially since studies have shown that neurons can exhibit sub-millisecond precision \citep{SejnowskiandMainen}. Temporal coding, on the other hand, although emphasizes the precise timing of individual spikes, in this literature stimulus reconstruction is often formulated  probabilistically (e.g. Bayesian Inference \citep{Pillow2008Spatiotemporal}) or through a linear transformation of spike-responses (e.g. reverse correlation \citep{Rieke1997Spikes}). Probabilistic approaches to reconstruction complicate the development of a deterministic signal processing framework from spike trains, while simple linear transformations may be too restrictive for representing a generalized class of signals. Recent advanced temporal coding schemes, such as those by Sophie et al. \citep{Sophie2017SpikeBySpike}, leverage recurrent networks and have shown near perfect reconstruction for certain signals. However, these techniques involve complex training procedures and do not provide reconstruction guarantees for a generalized class of signals. This paper is therefore motivated to build a signal processing framework that deterministically encodes continuous-time signals into biologically feasible spike trains, addressing questions of representable signal classes and reconstruction bounds.

\textbf{Related Work and Contribution:}
From a signal processing standpoint, the very coarse $\Sigma\Delta$ quantization of bandlimited signals, as investigated in \citep{Daubechies}, effectively represents a spike train encoding. This scheme encodes an oversampled signal into a stream of 1/0 bits by thresholding the cumulative quantization error. Similarly, the Time encoding Machine \citep{lazar} encodes a signal $X(t)$ into a sequence of time instances $\{t_k\}$, corresponding to moments when the signal crosses certain thresholds using an integrate-and-fire mechanism. In both cases, the signal class is bandlimited, and reconstruction guarantee is provided in the oversampled regime above the Nyquist rate. However, evidence suggests that  spike codes can achieve high reconstruction accuracy while being much sparser than warranted by a rate code (e.g., the H1 neuron in the fly \cite{nemenman2008}). Our proposed framework considers a generalized class--Finite Rate of Innovation\citep{vetterli2002sampling} (FRI) signals--not limited to bandlimited signals. It involves an encoding model for a spiking neuron that convolves the input signal with a generalized kernel function and produces spikes via thresholding on that convolved signal. Using a set of generalized convolution kernels for spiking neurons is biologically justified. In sensory pathways, an input signal undergoes multiple graded transformations before reaching a spiking neuron (e.g., in the visual pathway, signals pass through  amacrine, horizontal, bipolar cells, etc., before being converted into spike trains at the retinal ganglion cells). We model this transformation using a linear convolution kernel, while the overall transformation into spike trains remains nonlinear due to the thresholding operation. The use of such generalized kernels in encoding spike trains can also be seen in \citep{lazar2005faithful, lazar2005multichannel}. However, there, the inverse problem from spike trains to the reconstructed signal is solved in the Fourier domain (i.e., using {\em sinc} interpolation), restricting it to the bandlimited class to provide reconstruction guarantees in the oversampled regime. In our framework, we derive a closed-form solution to the inverse problem in the Hilbert space of shifted kernel functions themselves, enabling our framework to represent a generalized FRI class of shifted kernel functions and allowing us to provide reconstruction bounds in the sparse regime. By representing reconstructed signals sparsely using shifted kernel functions, our framework resembles the Convolutional Sparse Coding (CSC) technique \citep{CristinaCSCReview}. However, while CSC encodes signals through a complex optimization procedure, our framework leverages the efficiency of fast, real-time biological signal processing, showing clear computational advantages, especially in the low spike rate regime (section \ref{Experiments}).

Our main contributions are:
1) We present a comprehensive signal processing framework for encoding using biological spike trains and decoding. While reconstruction is not known to be a biological phenomena, it serves as a direct method to evaluate the efficacy of the coding framework rather than using other approaches such as information theory.   
2) In Section \ref{Class}, we derive a closed-form solution to the inverse problem of reconstructing the signal from a spike train ensemble. We identify the feasibility conditions for perfect and approximate reconstruction for a generalized FRI class of signals.
3) In section \ref{windowing}, we establish the robustness of our framework. We first observe that the optimal reconstruction, formulated in Section \ref{Decoding Module}, is poorly conditioned, especially with lengthy input signals producing many spikes. Inspired by biological signal processing, where real-time response to spike-encoded inputs are necessary,  we develop an efficient linear-time iterative version of the optimal reconstruction (formulated in Section \ref{Decoding Module}) that considers only a window of past spikes. We prove that this window-based reconstruction converges to the optimal solution under realistic assumptions.  
4) In Section \ref{Experiments}, we validate our technique through experiments on a large corpus of audio signals. The results show remarkable reconstruction accuracy at a low spike rate, achieving a median of about 20dB SNR at one-fifth of the Nyquist rate. Additionally, we compare our results against a state-of-the-art convolutional sparse coding technique, demonstrating superior accuracy and runtime in the low spike rate regime.

\section{Coding}
\label{Encoding Module}
For encoding, we make the following assumptions:
\begin{enumerate}
    \item 
    \label{asmp1}
We consider the set of input signals $\mathcal{F}$ to be the class of all \emph{ finite-support, bounded functions} (formally, $\mathcal{F} = \{X(t)| t \in [0,\tau], |X(t)|\leq b\}$, for some arbitrary but fixed $\tau, b \in {\mathbb R}^{+}$) that satisfy a finite rate of innovation bound \citep{Vetterli2002}. Naturally, $X(t) \in L^2$, i.e., square integrable.

\item
\label{asmp2}
We assume an ensemble of $m$ spiking neurons $\Phi= \{ \Phi^j| j\in Z^+, 1 \leq j \leq m \} $, each characterized by a continuous kernel function $\Phi^j(t)$, where $\forall j, \Phi^j(t) \in C[0,\tau]$, $\tau \in {\mathbb R}^+$. Also we assume that each kernel $\Phi^j$ is normalized, i.e., $||\Phi^j||_2=1, \forall j$. 
\item 
Finally, we assume that $\Phi^j$ has a time varying threshold $T^j(t)$. The ensemble of kernels $\Phi$ encodes a given input signal $X(t)$ into a sequence of spikes $\{(t_i, \Phi^{j_i})\}$, where the $i$\textit{th} spike is produced by the $j_i$\textit{th} kernel $\Phi^{j_i}$ at time $t_i$ if and only if:
$
\label{spikeConstraint}
	\int X(t) \Phi^{j_i}(t_i-t) dt = T^{j_i}(t_i). 
$
\end{enumerate}
Due to assumptions \ref{asmp1} \& \ref{asmp2}, our framework operates within a Hilbert space $\mathcal{H}$ of bounded-support, square-integrable functions $L^2[0,\tau]$ for some $\tau \in \mathbb R$, endowed with the standard inner product: $\langle f,g\rangle = \int fg, \forall f,g \in \mathcal{H}$. As these functions reside in a Hilbert space, the algebraic operations on the inputs or the kernels (e.g., the use of the projection operator) performed later in this paper are well-defined.

In our implementation a  threshold function is assumed in which the time varying threshold $T^j(t)$ of the $j{th}$ kernel remains constant at $C$ until that kernel produces a spike, at which time an \emph{after-hyperpolarization potential (ahp)} increments the threshold by a value $M$. This increment then returns to zero linearly within a refractory period $\delta$. 
Formally,
\begin{equation}
\label{thresholdeq}
T^j(t) = C + \sum_{t_p^j \in [t-\delta,t]} M(1 -  \frac{t-t_p^j}{\delta} ) \;\;\; (C, M, \delta \in \mathbb R^+)
\end{equation}
where  the sum is taken over all spike times $t_p^j$ in the interval $[t-\delta, t]$ at which the kernel $\Phi^{j}$ generated  a spike.
This threshold function allows a neuron to remain quiescent as long as the signal is uncorrelated with its kernel $\Phi^j$; it starts firing when the correlation reaches a certain threshold and continues to fire at higher threshold levels communicating increasing correlation levels, only inhibited by previous spikes. This phenomenon of probing signals via a sequence of spikes is depicted in fig{\ref{lockstepspikes}} for one kernel. In Eq.\ref{thresholdeq}, the  \textit{ahp} model is assumed to approximate the behavior of a biological neuron which undergoes a brief refractory period, typically lasting 1 msec immediately after producing a spike. In our model this is achieved by setting the value of $M$ to a value much higher than the supremum norm of the input signal. As will be evident from subsequent sections, such a model of \textit{ahp} not only bounds the interspike intervals thus ensuring the stability of the spiking framework, it also allows our model to spike at dynamically changing threshold levels starting from a baseline threshold at $C$ for varying correlation levels between the input signal and the kernels, leading to better representational capacity of our model. In what follows, we first clarify notations used throughout this paper for improved readability. Following that, we present a set of corollaries based on the assumptions of our encoding model.

\textbf{Notations followed in the paper:}
\begin{itemize}
    \item $t_i$:  Time of  occurrence of the $i$th spike.
    \item $T_i$:  The threshold value.
    \item $X(t)$ or $X$: The input signal. For ease of notation we drop the time $t$ as function argument and simply indicate the input as $X$ instead of $X(t)$.
    \item $X^*$: The reconstructed signal.
    \item $m$: The number of kernels that comprise our framework.
    \item $N$: The total number of spikes produced by the system.
    \item $(\phi^{j_i}, t_i)$: Tuple denotes the $i$th spike produced by kernel $\phi^{j_i}$ at time $t_i$.
    \item $\Phi^{j_i}(t_i - t)$ or $\phi_i$: The kernel function producing the $i$th spike, inverted and shifted to the time of the spike's occurrence $t_i$. For brevity's sake, this function is alternatively termed as the $i$th spike instead of the tuple notation $(\phi^{j_i}, t_i)$ and is denoted via the shorthand $\phi_i$ whenever appropriate. Also, the mathematical definition of the term "spike" must not be confused  with the real physical object representing the elicitation of a neuron's action potential. These distinctions should be clear from the context. 
    \item $\mathcal{S}(V)$: Subspace spanned by $V$ in a Hilbert space $\mathcal{H}, V \subseteq \mathcal{H}$.
    \item $\mathcal{P}_v(u)$: In a Hilbert space $\mathcal{H}$, the projection of $u$ on a vector $v$  for $u, v \in \mathcal{H}$.
    \item $\mathcal{P}_{\mathcal{S}(V)}(u)$: In a Hilbert space $\mathcal{H}$ the projection of $u$  on the subspace $\mathcal{S}(V)$  for $u\in \mathcal{H}, V\subseteq \mathcal{H}$. Note that this notation is similar to the above 
    notation of $\mathcal{P}_v(u)$  except here the projection is
    taken w.r.t. a subspace $\mathcal{S}(V)$ instead of a single 
    vector $v$.  This should be clear from the context.
\end{itemize}

\begin{corollary}
\label{convolutionCorollary}
    Let $\Phi^j $ be a function in $ C[0, \tau], \text{ where }\tau \in \mathbb R^{+}$ and $||\Phi^j||_2=1$. Let $X(t) \in \mathcal{F} = \{f(t)\mid t \in [0,\tau'], |f(t)|\leq b\},$ where $b, \tau' \in \mathbb R^+$, be the input to our model. Then: 
    (a) The convolution $C^j(t)$ between $X(t)$ and $\Phi^j(t)$, defined by $C^j(t) = \int X(t') \Phi^{j}(t-t') dt'$ for $ t\in [0, \tau+\tau']$, is a bounded and continuous function. Specifically, one can show that $|C^j(t)| < b\sqrt{\tau}$  for all $t \in [0, \tau+\tau']$ .
    (b) Suppose the parameter $M$ in the Eq. \ref{thresholdeq} is chosen such that $M > 2b\sqrt{\tau}$. Then the interspike interval between any two spikes produced by the given neuron $\Phi^j$ is greater than $\frac{\delta}{2}$.
\end{corollary}
\textbf{Proof:}
The proof has been detailed in Appendix \ref{convolutionCorollarySupp} since the corollary is intuitively evident. $\hfill\Box$    
\begin{corollary}
\label{spikerateCorollary}
    Suppose the assumptions of the Corollary \ref{convolutionCorollary} hold true. Then, the spike rate generated by our framework for any input signal $X(t)$ is bounded. Consequently, the maximum number of preceding spikes that can overlap with any given spike is bounded  above by a constant value.
\end{corollary}
\begin{proof}
    Based on Corollary \ref{convolutionCorollary}, the total spike rate is bounded above by $\frac{2m}{\delta},$ where $m$ is the number of kernels employed by our model. Since each kernel $\Phi^j \in C[0, \tau]$ is compact support,  the maximum number of preceding spikes any given spike can overlap with is bounded above by $\tau \frac{2m}{\delta}$, where $\tau$ is the maximum length of support for any kernel $\Phi^j$.
\end{proof}
\begin{corollary}
\label{thresholdInnerProduct}
 Let $X(t)$ be an input signal and $\Phi^j$ be a kernel. Suppose the convolution between $X(t)$ and $\Phi^j(t)$ at time $t_p$ is denoted by $C^j(t_p) = \langle X(t), \Phi^j(t_p-t)\rangle >0$, and let the absolute refractory period be $\delta$ as modeled in Eq. \ref{thresholdeq}. If the baseline threshold $C$ is set such that $0 < C \le C^j(t_p)$, then the kernel $\Phi^j$ must produce a spike in the interval $[t_p-\delta, t_p]$, according to the threshold model defined in Eq. \ref{thresholdeq}.   
\end{corollary}
 \textbf{Proof:} We prove by contradiction, utilizing  the continuity of the convolution function, $C^j(t)$. \textbf{Case 1:} No spike is produced by the kernel $\Phi^j$ before or at $t_p$. By definition $C^j(0) =0$. However, since $0 < C \leq C^j(t_p)$, the continuous function $C^j(t)$ must intersect the baseline threshold $C$ between $t=0$ and $t=t_p$ by the intermediate value theorem. This contradiction implies that the kernel $\Phi^j$ must produce a spike prior to or at time $t_p$. \textbf{Case 2:} Assuming spikes occurred before or at time $t_p$ by the kernel $\Phi^j$, let $t_l$ be the time of last spike produced by $\Phi^j$ before or at $t_p$. Suppose $t_l < t_p -\delta$ to ensure no spike in the interval $[t_p-\delta, t_p]$. Then, by Eq. \ref{thresholdeq}, the threshold of kernel $\Phi^j$ at $t_p$ is $T^j(t_p) = C$. However, $C^j(t_p) = \langle X(t), \Phi^j(t_p-t)\rangle  \geq C$. Since $C^j(t_l) = T^j(t_l)$ and $C^j(t)$ is continuous, and considering Eq. \ref{thresholdeq}, where the \textit{ahp} raises by a high value $M$ at $t_l$ and then linearly decreases to  $C$ before $t_p$, the intermediate value theorem implies that $C^j(t)$ must cross the threshold between $t_l$ and $t_p$. However, since $t_l$ was the last spike of  $\Phi^j$ before $t_p$, this contradicts our assumption. Hence, there must be a spike in $[t_p-\delta, t_p]. \hfill\Box$

\begin{assmp}
    \label{assumption1}
    $||\mathcal{P}_{\mathcal{S}(\{\phi_1, ..., \phi_{n-1}\})}(\phi_n)|| \leq \beta < 1$, for some $\beta \in \mathbb{R}$, $\forall n \in \{1,..., N \}$, where $N$ is the total number of spikes produced by the system. In words, the norm of the projection of every spike onto the span of all previous spikes is bounded from above by some constant strictly less than 1. 
\end{assmp}
\textbf{Justification:} Corollary \ref{convolutionCorollary} suggests that for appropriately chosen parameters to the threshold Eq. \ref{thresholdeq}, the spikes produced by the same kernel are sufficiently disjoint in time. Therefore, each new spike $\phi_n$ comes with a component that is disjoint in time with respect to the spikes produced by the same kernel. Since the biological kernels of our framework are causal in nature, due to the disjoint component in time a new spike maintains an orthogonal component with respect to all the previous spikes by the same kernel. For spikes produced by different kernels, we observe that different biological kernels correspond to different frequency responses (e.g., it has been observed that the responses of the auditory nerves can be well approximated by a bank of linear gammatone filters \citep{Lewicki2002, pattersonFilters}). Since there are only finitely many kernels in our framework, this leads to the fact that a new spike is poorly represented by the previous spikes produced by other kernels. Overall, a new spike $\Phi_n$, for appropriately chosen \textit{ahp} parameters in Eq. \ref{thresholdeq} will not be fully represented by previous spikes either due to disjoincy in time or frequency.  Hence, the overall set of spikes grows as a {\em linearly independent} set. The technical need for this assumption will become clear in later sections. 
\begin{figure}[h]
\hspace*{-10pt} 
\vspace{-1cm}
 	\includegraphics[width=100mm]{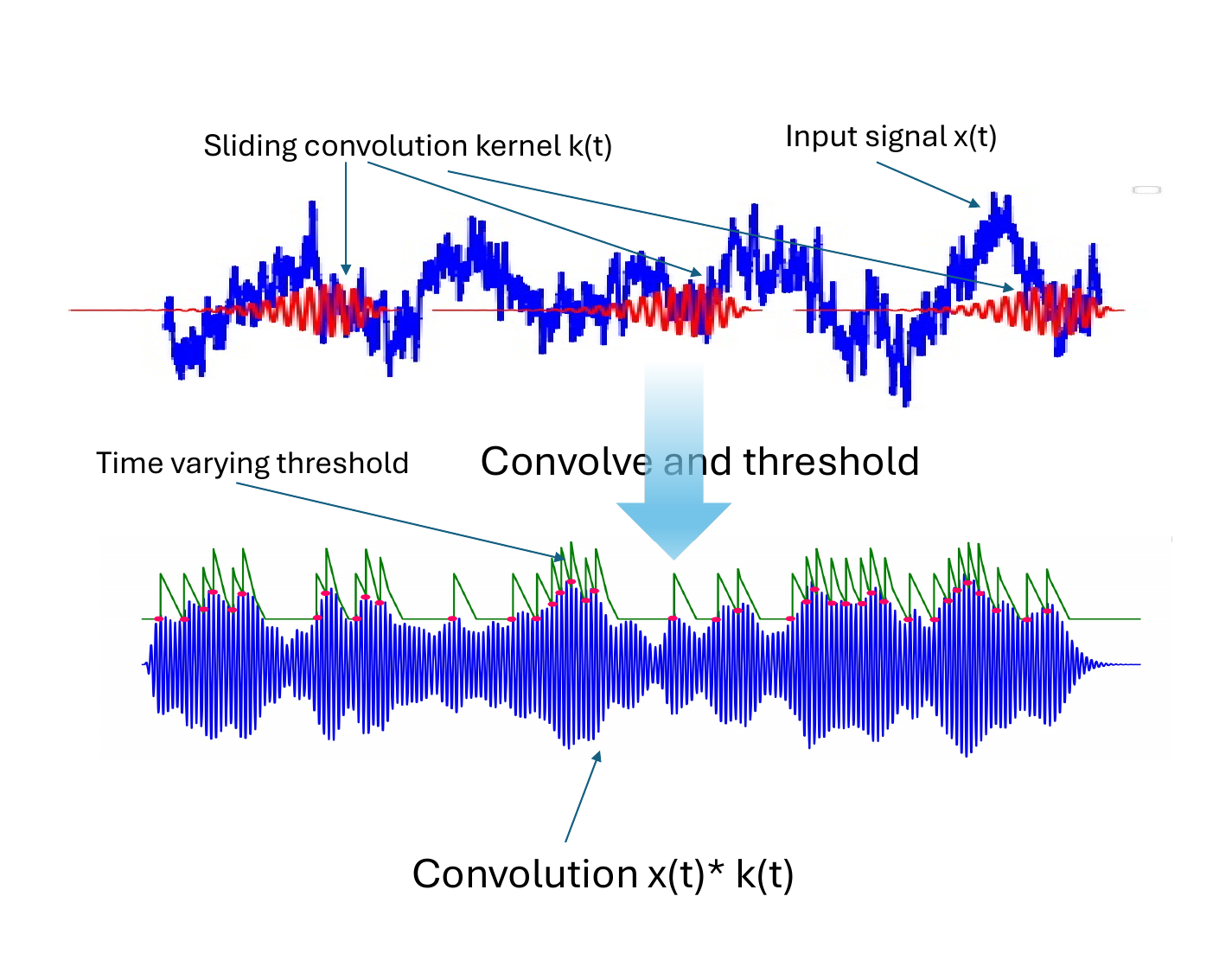}
 	\caption{The convolve and threshold mechanism described in the coding model for a single kernel. Top: a sample signal (in blue) is shown overlayed with a convolution kernel (in red). Below: the result of convolution in blue and the threshold function for the kernel in green. Spikes times are marked at the threshold crossing level with red dots.}%
 	\label{lockstepspikes}
  \vspace{-10pt}
\end{figure}
\section{Decoding}
\label{Decoding Module}

The objective of the decoding module is to reconstruct the original signal from the encoded spike trains. Considering the prospect of the invertibility of the coding scheme, we seek a signal that satisfies the same set of constraints as the original signal when generating all spikes apropos the set of kernels in ensemble $\Phi$. Recognizing that such a signal might not be unique, we choose the reconstructed signal as the one with minimum $L2$-norm. Formally, the reconstruction $X^*$ of the input signal $X$ is formulated to be the solution to:\\
\begin{equation}
\label{optimizationproblem}
\begin{aligned}
	& X^{*}= \underset{\Tilde{X}}{\text{argmin}}
	||\Tilde{X}||_2^2 \\
	& \text{s.t.}
	\int \Tilde{X}(\tau)\Phi^{j_i}(t_i-\tau)d\tau = T^{j_i}(t_i); 
	1 \le i \le N
\end{aligned}
\end{equation}
where  $\{(t_i, \Phi^{j_i})|i \in \{1,...,N\}\}$ is the set of all spikes generated by the encoder.
The choice of $L2$ minimization 
is in congruence with the dictum of energy efficiency in biological systems. The assumption is that, of all signals, the one with the minimum energy that is consistent with the spike trains is desirable.
Also, an $L2$ minimization in the objective of (\ref{optimizationproblem}) reduces the convex optimization problem to a solvable linear system of equations as described below.
\section{Signal Class for Perfect Reconstruction}
\label{Class}
We observe that in general the encoding of $L^{2}[0,T]$ signals into spike trains is not an injective map; the same set of spikes can be generated by different signals so as to result in the same convolved values at the spike times. Naturally, with a finite and fixed ensemble of kernels $\Phi$, one cannot achieve perfect reconstruction for all $L^{2}[0,T]$ signals. Assuming, additionally, a finite rate of innovation, as $\mathcal{F}$ was previously defined changes the story. We now restrict ourselves to a subset $\mathcal{G}$ of $\mathcal{F}$ defined as 
$
\label{restricted class}
\mathcal{G} = \{X| X \in \mathcal{F}, X= \sum_{p=1}^{N} \alpha_p \Phi^{j_p}(t_p-t), j_p \in \{1,...,m\}, \alpha_p \in {\mathbb R}, t_p \in {\mathbb R}^{+}, N \in Z^{+}\}
$
and address the question of reconstruction accuracy. Essentially $\mathcal{G}$ consists of all linear combinations of arbitrarily shifted inverted kernel functions. 
$N$ is bounded above by the total number of spikes that the ensemble $\Phi$ can generate over $[0,T]$. 
For the class $\mathcal{G}$ the \emph{perfect reconstruction theorem} is presented below. The theorem is proved with the help of two lemmas.\\
\vspace{-15pt}
\begin{theorem}{(Perfect Reconstruction Theorem)}
\label{PerfectReconsThm}
Let $X \in \mathcal{G}$ be an input signal. Then for appropriately chosen time-varying thresholds of the kernels, the reconstruction $X^{*}$, resulting from the proposed coding-decoding framework is accurate with respect to the $L2$ metric, i.e., $||X^{*}-X||_2 = 0$.\\ 
\end{theorem}
\vspace{-18pt}
\begin{lemma} 
\label{reconsLemma}
The solution $X^*$ to the reconstruction problem Eq. (\ref{optimizationproblem}) can be written as:
$
\label{reconstructionequation}
X^* = \sum_{i=1}^{N} \alpha_i \Phi^{j_i}(t_i-t)
$
where the coefficients $\alpha_i \in {\mathbb R}$ can be uniquely solved from a system of linear equations if the set of spikes $\{\phi_i = \Phi^{j_i}(t_i-t)\}_{i=1}^{N}$ produced is \emph{linearly independent}.\\
\end{lemma}
\vspace{-15pt}
\textbf{Proof:}
An argument similar to that of the Representer Theorem \citep{scholkopf2001} on Eq.(\ref{optimizationproblem}) directly results in:
$
\label{eq:reconsSignalEq}
X^* = \sum_{i=1}^{N} \alpha_i \Phi^{j_i}(t_i-t)
$
where the $\alpha_{i}$'s are real valued coefficients. This holds true because any component of $X^*$  orthogonal to the span of the ${\Phi^{j_i}(t_i-t)}$'s does not contribute to the convolution (inner product) constraints. In essence, $X^*$ is an orthogonal projection of $X$ on the span of the spikes $\{\phi_i = \Phi^{j_i}(t_i-t)|i \in \{1,2,...,N \}\}$. Hence, the coefficients can be derived by solving the linear system:
$
\label{alphaeq}
P\alpha=T$
where $P$ is the $N\times N$ Gram matrix of the spikes, i.e., $[P]_{ik} = \langle \Phi^{j_i}(t_i-t),  \Phi^{j_k}(t_k-t) \rangle$, and  $T = \langle T^{j_1}(t_1), \ldots, T^{j_N}(t_N)\rangle^T$. Furthermore, the system has a unique solution if $P$ is invertible, which is the case if the set of spikes $\{\phi_i\}_{i=1}^{N}$ is linearly independent. This in turn follows from Assumption \ref{assumption1}. Notably, when the $P$-matrix is non-invertible we still obtain a unique reconstruction $X^*$ (see Appendix \ref{perfectReconsSupp}) by calculating the $\alpha$'s via the pseudo-inverse of $P$, but in such a case the system becomes ill-conditioned, an issue that is analyzed separately in section \ref{windowing}. $\hfill\Box$
\\
\vspace{-15pt}
\begin{lemma}
\label{th2}
Let $X^*$ be the reconstruction of an input signal $X$ with $\{\phi_i\}_{i=1}^{N}$ being the set of generated spikes. Then, for any arbitrary signal $\Tilde{X}$ within the span of the spikes given by $\Tilde{X} = \sum_{i=1}^{N} a_i \phi_i ,\; a_i \in \mathbb{R}$, the following holds: $||X- X^*|| \leq ||X-\tilde{X}||$.\\    
\end{lemma}
\vspace{-15pt}
\textbf{Proof:}
Follows straightforwardly from the fact that $X^*$ is an orthogonal projection on the span (see Appendix \ref{perfectReconsSupp}). $\hfill\Box$

Exploring further, for a given input signal $X$, if $S_1$ and $S_2$ are two sets of spike trains where $S_1 \subset S_2$ produced by two different kernel ensembles, the second a superset of the first, then Lemma 2 further implies that the reconstruction due to $S_2$ is at least as good as the reconstruction due to $S_1$ because the reconstruction due to $S_1$ is in the span of the shifted kernel functions of $S_2$ as  $S_1 \subset S_2$. This immediately leads to the conclusion that for a given input signal the more kernels we add to the ensemble the better the reconstruction, provided the kernels maintain linear independence.

\textbf{Proof of Theorem \ref{PerfectReconsThm}:}
The proof of the theorem follows directly from Lemma 2. Since the input signal $X \in \mathcal{G}$, let $X$ be given by:
$
    X= \sum_{p=1}^{N} \alpha_p \Phi^{j_p}(t_p-t) \hspace{10pt}(\alpha_p \in {\mathbb R}, t_p \in {\mathbb R}^{+}, N \in Z^{+})
$.
Assume that the time varying thresholds of the kernels in our kernel ensemble $\Phi$ are set in such a manner that the following conditions are satisfied:
$
    \langle X, \Phi^{j_p}(t_p-t) \rangle = T^{j_p}(t_p) \hspace{10pt} \forall{p \in \{1,...,N\}}
$
i.e., each of the kernels $\Phi^{j_p}$ at the very least produces a spike at time $t_p$ against $X$ (regardless of other spikes at other times). Clearly then $X$ lies in the span of the set of spikes generated by the framework. Applying Lemma 2 it follows that:
$
    ||X- X^*||_2 \leq ||X-X||_2 = 0 
$.$\hfill\Box$

\section{Approximate Reconstruction}
Theorem \ref{PerfectReconsThm} stipulates the conditions under which perfect reconstruction is feasible in the purview of our framework. Specifically the theorem shows the ideal conditions---when the input signal lies in the span of shifted kernel functions and the spikes are generated at certain desired locations---where perfect reconstruction is attainable. However, under realistic scenarios such conditions may not be feasible and hence the need for quantification of reconstruction error as the system deviates from the ideal conditions. For example, even though corollary \ref{spikerateCorollary} shows that a spike can be produced arbitrarily close to the desired location by setting the \textit{ahp} parameters $C$ and the $\delta$ of Eq. \ref{thresholdeq} at reasonably low values, it begs the question to what extent the reconstruction suffers due to small deviations in spike times. Likewise, the input signal may not perfectly fit in the signal class $\mathcal{G}$, i.e. the input may not be exactly representable by the kernel functions due to the presence of  internal or external noise. Under such non-ideal scenarios how much the reconstruction suffers is addressed in the following theorem.


\begin{theorem}{(Approximate Reconstruction Theorem).}
\label{approxReconsThm}
Let the input signal $X$ be represented as $X = \sum_{i=1}^{N} \alpha_i f^{p_i}(t_i-t)$, where $\alpha_i \in \mathbb{R}$ and $f^{p_i}(t)$ are bounded functions on finite support. Assume that there is at least one kernel function $\Phi^{j_i}$ in the ensemble for which $||f^{p_i}(t) - \Phi^{j_i}(t)||_{2} < \delta$ for all $i \in \{ 1,...,N\}$. Additionally, assume each kernel $\Phi^{j_i}$ produces a spike within a $\gamma$ interval of $t_i$,
for some $\delta \text{,} \gamma \in \mathbb{R^+}$,  for all $i$. Further, assume the functions $f^{p_i}$ satisfy a frame bound type of condition: $
\sum_{k \neq i} \langle f_{p_i}(t-t_i),f_{p_k}(t-t_k) \rangle \; \leq  \eta \; \forall \, i \in \{1,...,N\},
$ and that the kernel functions are Lipschitz continuous. Under such conditions, the $L^2$ error in the reconstruction $X^{*}$ of the input $X$ has bounded SNR.
\end{theorem}
\vspace{-7pt}
\textbf{Proof:} The proof follows from continuity arguments and the use of bounds on the eigen values of the Gram matrix $P$ (see Appendix \ref{approxReconsThmSupp}). 

\section{Stability of the solution and Windowed Iterative Reconstruction:}
\label{windowing}
Theorem \ref{approxReconsThm} shows that even under non-ideal conditions,our technique can keep the reconstruction error in check through suitable parameter choices. However, this may increase spike rates, as implied by Corollaries \ref{convolutionCorollary} and \ref{thresholdInnerProduct}. Higher spike rates exacerbate the condition number of the $P$ matrix referred to in Lemma \ref{reconsLemma}. Instabilities in $P$ render the solutions practically unusable in applications with finite precision floating point representations, due to quantization error. 
The \textit{ahp} partially mitigates this issue for finite-sized $P$ matrices by ensuring linear independence among the spikes. The inhibitory effect of the \textit{ahp}, as alluded to in assumption \ref{assumption1},  results in spikes that are sufficiently disjoint in time, leading to production of a linearly independent set of spikes. But the condition number can get progressively worse as we process longer signals and the size of $P$ grows arbitrarily large.
The following Theorem establishes a relation between the condition number of the $P$ matrix and the spike count, revealing that, in the worst case, despite realistic assumptions about spike non-overlap, the condition number can deteriorate exponentially.
\vspace{-5pt}
\begin{theorem}[Condition Number Theorem]
\label{conditionNumberThm}
Let $\{P_k\}$ denote the set of all Gram matrices corresponding to any set of $k$ successive spikes $\{\phi_1, \ldots, \phi_k\}$, i.e., $P_k[i,j] = \langle \phi_i, \phi_j \rangle$, where each spike satisfies the Assumption \ref{assumption1}: $||\mathcal{P}_{S(\{\phi_1, \ldots, \phi_{i-1}\})}(\phi_i)|| \leq \beta < 1$, $\forall i \in \{2, \ldots, k\}$. If $C_k$ denotes the least upper bound on the condition number of the class of matrices $\{P_{k}\}$, then $(1-\beta^2)^{-k+1}  \leq C_k \leq (1 + (k-1)\beta) (\frac{1-\beta^2}{2})^{-k+1} \\$.
\end{theorem}
\vspace{-14pt}
\textbf{Proof:} The condition number of $P_k$ is defined as 
$\frac{\Lambda_{max} (P_k)}{\Lambda_{min} (P_k)}$, where $\Lambda_{max}(P_k)$ and $\Lambda_{min}(P_k)$ denote the maximum and minimum eigen values of $P_k$.
First we find the infimum $L_k$ on $\Lambda_{min}$ over all $\{P_k\}$, the class of all gram-matrices of $k$ successive spikes. 
We find the infimum inductively on $k$. By definition,
\vspace{-5pt}
\begin{align}
 &   \Lambda_{min}(P_k) = \min_{e, ||e||=1} || \Sigma_{i=1}^{k} e_i\phi_i||^2  \nonumber\\
 & (\text{where } e = [e_1,..., e_k]^T, \text{a k-vector})\nonumber\\
 &   =  \min_{e, ||e||=1} [e_k^2 + 2e_k\langle \phi_k ,  \Sigma_{i=1}^{k-1} e_i\phi_i\rangle + ||\Sigma_{i=1}^{k-1} e_i\phi_i||^2 ] \nonumber\\
 &   \geq \min_{e, ||e||=1} [e_k^2 - 2|e_k| \beta\sqrt{1-e_k^2} ||\Sigma_{i=1}^{k-1} \frac{e_i}{\sqrt{1- e_k^2}}\phi_i|| + \nonumber\\
 & \;\;\;\;||\Sigma_{i=1}^{k-1} \frac{e_i}{\sqrt{1- e_k^2}}\phi_i||^2] \;\; \;(\text{since, } ||\mathcal{P}_{S\{\phi_1, ..., \phi_{i-1}\}}(\phi_i)|| \leq \beta )\nonumber\\
 & \geq \min_{e, ||e||=1} [e_k^2 - 2|e_k| \beta z\sqrt{1-e_k^2}+ (1-e_k^2)z^2] \label{construction}\\
 & \hspace{2cm}(\text{denoting, } z = ||\Sigma_{i=1}^{k-1} \frac{e_i}{\sqrt{1- e_k^2}}\phi_i||) \nonumber\\
 &\text{Now we set $|e_k| = \cos{\theta}, |e_k| \leq 1$ to obtain:}\nonumber\\
 & \Lambda_{min}(P_k) \geq \min_{e, ||e||=1} [\cos^{2}{\theta} - 2 \beta z\cos{\theta}\sin{\theta}+ z^2\sin^2{\theta}] \nonumber\\
 & \geq \min_{e, ||e||=1} [\frac{1+z^2}{2} + \frac{1-z^2}{2}\cos{2\theta} - \beta z \sin{2\theta}] \nonumber\\
 & \geq \min_{e, ||e||=1} [\underbrace{\frac{1+z^2}{2} - \sqrt{(\frac{1-z^2}{2})^2 + \beta^2 z^2}}_\text{g($z^2$)}] \label{costheta}\\
 & \text{But, } z^2 = ||\Sigma_{i=1}^{k-1} \frac{e_i}{\sqrt{1- e_k^2}}\phi_i||^2 \geq L_{k-1} \nonumber\\
 & (\text{Since, } \Sigma_{i=1}^{k-1} \frac{e_i^2}{1- e_k^2} = 1 \text{ we get this inductively}) \nonumber\\
  &\text{Since for $|\beta| < 1$ the expression $g(z^2)$ in \ref{costheta} is a monotonic} \nonumber\\
 &\text{in  $z^2$, and $L_k$ is the infimum of $\Lambda_{min}(P_k)$, we may 
 write,}\nonumber\\
 & L_k \geq [\frac{1+L_{k-1}}{2} - \sqrt{(\frac{1-L_{k-1}}{2})^2 + \beta^2 L_{k-1}}] \label{ineq}\\
 & \text{But $L_k$ is a lower bound of $\Lambda_{min}$, and all the inequalities } \nonumber\\
 & \text{above  are tight. Specifically, Eq. \ref{construction} \& \ref{costheta} show how given} \nonumber\\
 & \text{ $P_{k-1}$, a Gram matrix of $k-1$ successive spikes with  $\Lambda_{min}$} \nonumber\\
 & \text{$= L_{k-1}$, one can choose $\phi_k$ and $e$ to result in a matrix $P_k$, so} \nonumber\\
 & \text{that $\Lambda_{min}(P_k)$ achieves the lower bound of \ref{ineq}. Therefore, the}\nonumber\\
 & \text{inequality of \ref{ineq} can be turned into an equality.} \nonumber\\
 \nonumber 
\end{align}
\begin{align}
  & L_k = [\frac{1+L_{k-1}}{2} - \sqrt{(\frac{1-L_{k-1}}{2})^2 + \beta^2 L_{k-1}}]\nonumber\\
 &L_k= [\frac{(1-\beta^2)L_{k-1}}{\frac{1+L_{k-1}}{2} + \sqrt{(\frac{1-L_{k-1}}{2})^2 + \beta^2 L_{k-1}}}] \label{recurrence}\\
 & \text{Since, $|\beta| \geq 0$, setting  $|\beta| = 0$ in denominator of \ref{recurrence},}\nonumber\\
 & L_k \leq (1-\beta^2)L_{k-1} \label{lesser}\\
 & \text{Again, $L_1 =1$ and using induction we can get $L_k \leq 1$.} \nonumber\\
 & \text{Therefore, setting $L_{k-1} =1, \beta=1$ in denominator of \ref{recurrence},}\nonumber\\
 & L_k \geq \frac{(1-\beta^2)L_{k-1}}{2} \label{greater}\\ 
 &\Rightarrow \frac{(1-\beta^2)L_{k-1}}{2} \leq L_k\leq (1-\beta^2)L_{k-1} \;\;\;\;\;\;\;\;\text{(Using \ref{greater} \& \ref{lesser})}\nonumber\\
 & \Rightarrow (\frac{1-\beta^2}{2})^{k-1} \leq L_k \leq (1-\beta^2)^{k-1} \label{L_k_exp}\\
 & \text{Eq. \ref{L_k_exp} establishes a bound on the infimum of $\Lambda_{min}$. To } \nonumber\\ 
 & \text{complete the proof and establish a bound on $C_k$ we need to }\nonumber\\
 & \text{show a bound on the supremum of $\Lambda_{max}(P_k)$, call it $U_k$.} \nonumber\\
 & \text{A bound on $U_k$ can be shown as follows:}\nonumber\\
 & \Lambda_{max} (P_k) \leq \sup_{i} (P_k[i,i] + \sum_{i\neq j}|P_k[i,j]|) \nonumber\\ 
 & \hspace{1cm}(\text{using Gershgorin Circle Theorem}) \nonumber\\
 & = \sup_{i} \{ \langle \phi_i, \phi_i \rangle +  \sum_{i\neq j}|\langle \phi_i, \phi_j \rangle| \} \leq (1 + (k-1)\beta)\nonumber\\
 &\hspace{1cm} (\text{Since $||\mathcal{P}_{S\{\phi_1, \ldots, \phi_{i-1}\}}(\phi_i)|| \leq \beta \Rightarrow |\langle \phi_i, \phi_j \rangle| \leq \beta$}) \nonumber\\
 & \Rightarrow 1 \leq U_k \leq (1 + (k-1)\beta) \label{U_k_exp}\\
 &\hspace{1cm} (\text{$1 \leq U_k$ is trivial because $\Lambda_{max} =1$ for $P_k = I$})\nonumber\\
 & \text{Combining \ref{L_k_exp} \& \ref{U_k_exp} we get:}\nonumber\\
& (1-\beta^2)^{-k+1}  \leq C_k \leq (1 + (k-1)\beta) (\frac{1-\beta^2}{2})^{-k+1} \hspace{1cm}\Box\nonumber\\
 \nonumber
\end{align}
The above theorem provides a tight upper bound on the condition number of the $P$-matrix and clearly shows how the condition number can degrade  even if the spikes are sufficiently disjoint in time, i.e. $\beta \approx 0$. 
This is where the combined effect of causality of the kernels and the \textit{ahp} comes to our defense.
We observe that the addition of the $n+1^{th}$ spike on an existing set of $n$ spikes can only affect the solution substantially within a finite time window which in turn is facilitated by the effect of the ahp, which ensures that new spikes maintain reasonable separation with the previous spikes and therefore have a fading effect on the reconstruction back into the past. This observation is consistent with biological systems where an animal has to respond in real-time and therefore if an additional spike changes the conceived reconstruction too far into the past it's of little use. This simple observation enables us to encode and reconstruct signals in an online mode within a finite window of past spikes, leading to remarkable efficiency of our proposed coding scheme. In this section we first provide a mathematical formulation of our window-based reconstruction scheme and then we derive conditions under which the window-based reconstruction converges to the optimal reconstruction formulated in Lemma \ref{reconsLemma}. As it turns out, such conditions are easily met in the context of naturally occurring biological kernels, enabling our biologically motivated coding scheme to result in a very efficient and robust solution, as evidenced by our experimental results (section \ref{Experiments}).

\textbf{Windowed iterative reconstruction:} Lemma \ref{reconsLemma} establishes the fact that the reconstruction $X^*$ by our framework is essentially the projection onto the span of all spikes, i.e., $X^* = \projexp{N}{X}$. This observation enables us to formulate the reconstruction iteratively by updating an existing reconstruction on a set of $n$ spikes $\{\phi_i\}_1^n$ with each new incoming spike $\phi_{n+1}$, instead of solving the $P\alpha = T$ equation for the full set of spikes as shown in lemma \ref{reconsLemma}. The iterative update of the reconstruction then follows from the formula:\\
\vspace{-20pt}
\begin{align}
\label{iterativeRecons}
    \proj{n+1}{X} &= \proj{n}{X} + \langle X, \phi^{\perp}_{n+1}\rangle \frac{\phi^{\perp}_{n+1}}{||\phi^{\perp}_{n+1}||^2}
\vspace{-10pt}
\end{align}
where $\phi^{\perp}_{n+1}$ is the orthogonal complement of the additional $n+1$-th spike with respect to the span of all previous spikes, i.e. $\phi^{\perp}_{n+1} = \phi_{n+1} - \mathcal{P}_{\mathcal{S}\{\phi_1, ..., \phi_{n}\}}(\phi_{n+1})$. 
The above iterative scheme Eq. \ref{iterativeRecons}, motivates us to formulate a windowed iterative reconstruction. 
Here, when a new spike $\phi_{n+1}$ appears, instead of calculating its orthogonal complement with respect to the span of all previous spike, the orthogonal projection of $\phi_{n+1}$ is computed with respect to the span of $w$ previous spikes, where $w$ is chosen as a fixed window size. Mathematically, for the $(n+1)$-th incoming spike, we define the \emph{windowed iterative reconstruction}, $X^{*}_{n+1,w}$, for an input signal $X$ with window size $w$ iteratively as:
\vspace{-8pt}
\begin{align}
\label{windowingEq}
    X^{*}_{n+1,w} &=
    X^{*}_{n,w} + \langle X, \phi^{\perp}_{n+1,w}\rangle \frac{\phi^{\perp}_{n+1,w}}{||\phi^{\perp}_{n+1,w}||^2}
\end{align}
\vspace{-15pt}
\\
where $\phi^{\perp}_{n+1,w}$ is defined as:
$$\phi^{\perp}_{n+1,w} = \phi_{n+1} - \projalt{n-w+1}{n}{\phi_{n+1}}$$
The idea is that $\phiperp{w}$ closely approximates $\phicomp$ for reasonably large window size $w$, allowing us to formulate an iterative reconstruction based only on a finite window of $w$ spikes rather than inverting a large $P$-matrix of size $N\times N$ as formulated in lemma \ref{reconsLemma}. Eq. \ref{windowingEq} involves computing $\phiperp{w}$ for each new spike, derived by inverting the $w\times w$ gram matrix corresponding to the previous $w$ spikes. Since $w$ is chosen as a finite constant independent of $N$, it speeds up the decoding process and holds the condition number of the solution in check as per Theorem \ref{conditionNumberThm}. The following Theorem \ref{windowthm} establishes how the window-based solution formulated in Eq. \ref{windowingEq} converges to the optimal solution of Eq. \ref{reconstructionequation} under an assumption feasible in the context of spikes produced by our framework via biological kernels. The assumption is:
\begin{assmp}
\label{assumption2}
   $||\mathcal{P}_{\mathcal{S}(\{\phi_1,\ldots, \phi_N\} \setminus \{\phi_n\})}(\phi_n)|| \leq \beta < 1$, for some $\beta \in \mathbb{R},\forall n \in \{1, ..., N\}$. In words, the norm of the projection of each spike onto the span of the remaining spikes is bounded above by a constant strictly less than 1.
\end{assmp}
\textbf{Justification:} Assumption \ref{assumption2} extends Assumption \ref{assumption1} to the future. For large $N$, it is possible to construct a spike sequence where each spike satisfies Assumption \ref{assumption1}, yet the norm of the projection of an individual spike onto the set of remaining spikes (including both past and future spikes), i.e.  $||\mathcal{P}_{\mathcal{S}(\bigcup_1^{N}\{\phi_i\} \setminus \{\phi_n\})}(\phi_n)|| \rightarrow 1$. An example illustrating this scenario is shown in Fig. \ref{sineKernelsOverlap}, where each spike corresponds to one full cycle of a sine wave, and for any given spike $\phi_n$ the half wave of its tail (head) precisely overlaps with the half wave of the head (tail) of its previous (next) spike $\phi_{n-1}$ ($\phi_{n+1}$). One can show that in such a scenario  $||\mathcal{P}_{\mathcal{S}(\bigcup_1^{N}\{\phi_i\} \setminus \{\phi_n\})}(\phi_n)||$ converges to 1 for large values of $n$ and $N$ (see Appendix \ref{sineKernelSupp} for details). This convergence occurs because within the compact support of a spike, the spike is fully represented by the overlapping components of the neighboring spikes. Conversely, if a spike is poorly represented within its compact support by the overlapping components of its neighboring spikes, it satisfies the condition of Assumption \ref{assumption2}. Specifically, if within its compact support, a spike $\phi_n$ can produce a component $\hat{\phi_n}$, orthogonal to all overlapping parts of the neighboring spikes with $||\hat{\phi_n}|| >0$, then $\phi_n$ satisfies the condition of Assumption \ref{assumption2}, because $\hat{\phi_n}$ is orthogonal to the overlapping parts of the neighboring spikes and hence $\hat{\phi_n}$ is orthogonal to every spike other than $\phi_n$. This observation enables us to assert Assumption \ref{assumption2} for spikes produced by our framework via the biological kernels (e.g. the gammatone kernels corresponding to the auditory processing \citep{pattersonFilters}) which inherently exhibit causal and fading memory properties, so that the overlapping parts of the neighboring spikes do not represent the given spike within its support. Figure \ref{overlapSimilarity} visually illustrates this property using the dot product of overlapping portions of gammatone kernels. We show that due to the \textit{ahp}, overlapping spikes corresponding to the same gammatone kernel poorly represent each other as they temporally separate. Additionally, spikes corresponding to gammatones of different frequencies poorly represent each other, irrespective of temporal shifts. 
Thus, with appropriately chosen \textit{ahp} parameters, the spikes produced by our framework's biological kernels are poorly represented within their supports by overlapping components of neighboring spikes, either due to temporal disjoincy or frequency mismatch. Since every spike overlaps with only finitely many other spikes (corollary \ref{spikerateCorollary}), Assumption \ref{assumption2} holds.\\
An important consequence of Assumption \ref{assumption2} is that the norm of the projection of every unit vector in a subspace of a finite partition of spikes onto the remaining subspace of spikes is strictly less than $1$. This condition forms the basis for establishing the convergence of \emph{windowed iterative reconstruction} to optimal reconstruction in our subsequent windowing theorem \ref{windowthm}. We formally state the condition in Lemma \ref{windowCondLemma} before presenting Windowing theorem \ref{windowthm}.
\vspace{-5pt}
\begin{lemma}
\label{windowCondLemma}
    Let $S = \{\phi_i\}_{i=1}^{N}$ denote the set of spikes generated by our framework, satisfying Assumption \ref{assumption2}, i.e., $\forall n \in \{1,..., N\}$, $||\mathcal{P}_{\mathcal{S}(\bigcup_{i=1}^{N}\{\phi_i\} \setminus \{\phi_n\})}(\phi_n)|| \leq \beta$, where $\beta \in \mathbb{R}$ is a constant strictly less than 1. Consider a subset $V \subseteq S$ of a finite size $d$, $d < N$. Then, for every $v \in \mathcal{S}(V)$ with $||v|| = 1$,  $\exists \beta_d < 1$, such that $||\mathcal{P}_{\mathcal{S}(S \setminus V)}(v)|| \leq \beta_d$ where $\beta_d$ is a real constant that depends on $\beta$ and $d$. Specifically, we can show that $\beta_d^2 \leq (1+ \frac{1-\beta^2}{d^2 \beta^2})^{-1} < 1$.
\end{lemma}
\vspace{-3pt}
\textbf{Proof:} Proof in Section \ref{windowCondLemmaSupp} of \citep{supplementary_key}. $\hfill \Box$
\vspace{-5pt}
\begin{SCfigure*}
  \caption{
   Illustration of why a spike produced by our framework using gammatone kernels is poorly represented by other overlapping spikes within its temporal support. \textbf{(a)} Shows a diagram of five spikes; the central spike (blue dashed line within the rectangle) overlaps with two past spikes (red) and two future spikes (green). Despite the overlaps, the distinct shapes of the gammatone kernels result in a poor representation of the central spike, leaving a component orthogonal to its neighbors (shown in black). \textbf{(b)} Details the interaction between two spikes, $\phi_1$ and $\phi_2$ (red and green curves), both using a gammatone kernel with a center frequency of 500 Hz. It examines how well $\phi_1$ is represented by the overlapping tail of $\phi_2$ (denoted $\phi_2^{tail}$, green dashed line), focusing on representation within this overlapping support only. The graph in blue measures the dot product $\frac{|\langle \phi_1, \phi_2^{tail}\rangle|}{||\phi_1|| ||\phi_2^{tail}||}$ as a function of the time lag \(t\) between the two spikes, illustrating how the representation deteriorates rapidly as the lag increases, thus demonstrating the poor representational quality induced by the \textit{ahp-effect}. \textbf{(c)} A similar plot of interaction between two spikes, $\phi_1$ and $\phi_2$ as in \textbf{(b)}, except here the center frequencies of the corresponding gammatone kernels are different (500 Hz for $\phi_1$ and 400 Hz for $\phi_2$). The blue graph represents the dot product $\frac{|\langle \phi_1, \phi_2^{tail}\rangle|}{||\phi_1|| ||\phi_2^{tail}||}$ as a function of time shift \(t\), highlighting systematic poor representation due to frequency differences. Each spike's time of occurrence is marked by a red dot.}
  	\includegraphics[width=70mm, height=90mm]{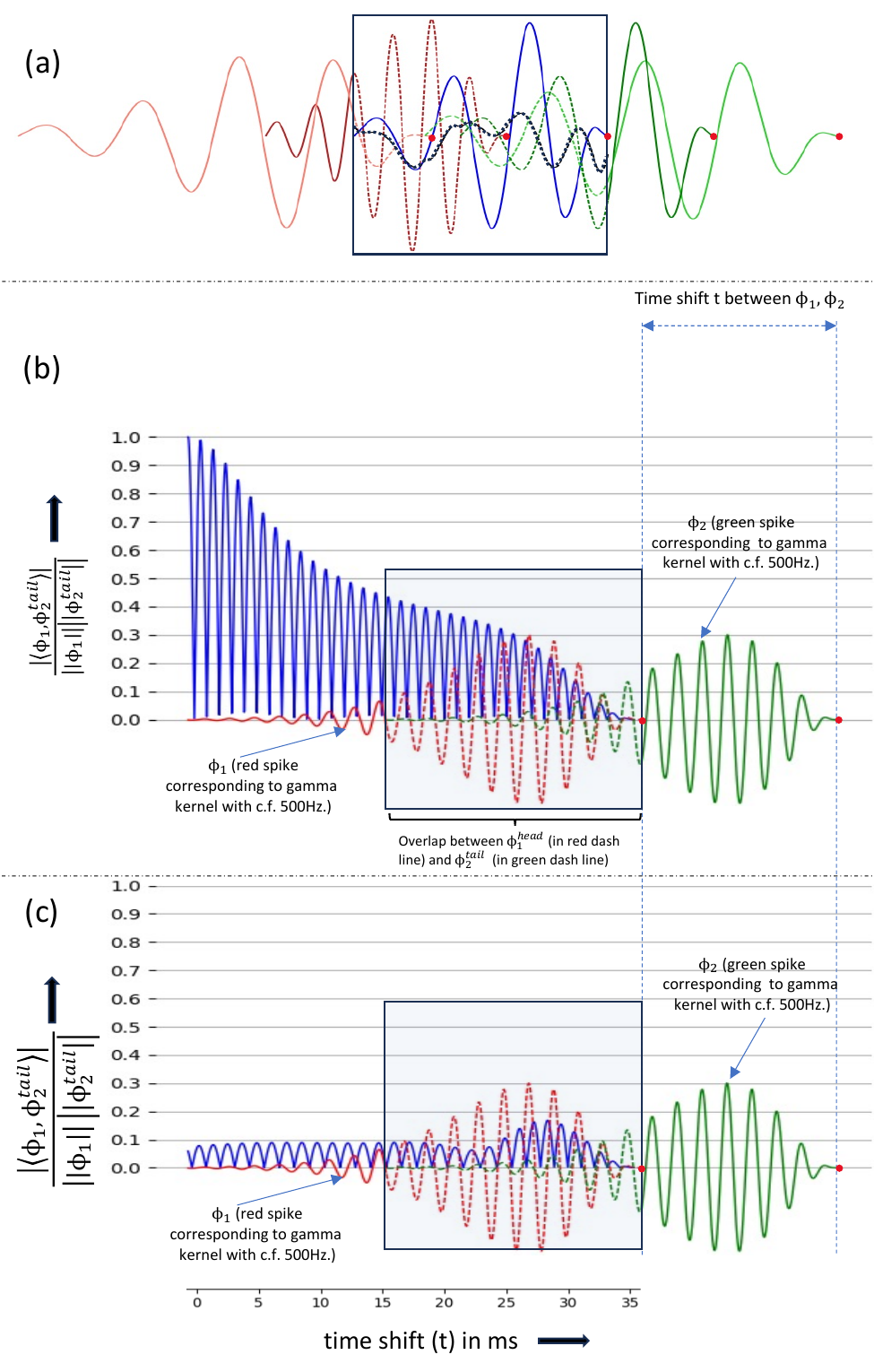}
 	\label{overlapSimilarity}
\end{SCfigure*}
\begin{theorem}[Windowing Theorem] 
\label{windowthm}
For an input signal $X$ with bounded $L_2$ norm, suppose our framework produces a set of $n+1$ successive spikes $S = \{\phi_1, ..., \phi_{n+1}\}$, sorted by their time of occurrence and satisfying Assumption \ref{assumption2}. The error in the iterative reconstruction of $X$ with respect to the last spike $\phi_{n+1}$ due to windowing, as formulated in Eq. \ref{windowingEq}, is bounded. Specifically, 
\vspace{-5pt}
\begin{align}
 & \forall \epsilon >0, \exists w_0 > 0 \text{ s.t. } ||\mathcal{P}_{\phi^{\perp}_{n+1,w}}(X)- \mathcal{P}_{\phi^{\perp}_{n+1}}(X)|| < \epsilon, \nonumber\\
& \forall w\geq w_0\text{ and } w \leq n \label{windowThmEqun}   
\\
& \hspace{-10pt}\text{ where  $w_0$ is independent of $n$ for arbitrarily large $n \in \mathbb N$.}\nonumber\\
\nonumber 
\end{align}
 \vspace{-40pt}
\\
\end{theorem}
\textbf{Import and Proof Idea}:  The Theorem implies that the error from windowing can be made arbitrarily small by choosing a sufficiently large window size $w_0$, independent of $n$, when $n$ is arbitrarily large. At first glance, one might think that the condition in Eq. \ref{windowThmEqun} is trivially satisfied by choosing $w_0 = n$, i.e., a window inclusive of all spikes. However, the key aspect of the theorem is that $ w_0 $ should be independent of $ n $ when $ n $ is arbitrarily large. This allows us to use the same window size regardless of the number of previous spikes, even for large signals producing many spikes, thereby maintaining the condition number of the overall solution as per Theorem \ref{conditionNumberThm}. Our proof demonstrates this by showing that the reconstruction error converges geometrically as a function of the window size, depending only on the spike rate which in turn depends on the \textit{ahp} parameters in Eq. \ref{thresholdeq}, but not on $n$.
The proof hinges on a central lemma showing that the $L_2$ norm of the difference between $\phi_{n+1,w}^{\perp}$ and  $\phi_{n+1}^{\perp}$ decreases steadily as \( w \) increases, based on the assumptions stated in the theorem. This ensures that \( \phi_{n+1,w}^{\perp} \) becomes a good approximation of $\phi_{n+1}^{\perp}$. The proof of the theorem then follows by establishing a bound on $ \|\mathcal{P}_{\phi^{\perp}_{n+1,w}}(X) - \mathcal{P}_{\phi^{\perp}_{n+1}}(X)\| $ for a given choice of $w$ for any bounded input $X$.
The lemma is provided below:\\
\vspace{-10pt}
\begin{lemma}
\label{windowLemma}
Under the conditions of Theorem \ref{windowthm}, for any $\delta > 0$, there exists $w_0 \in \mathbb{N}^+$ such that $\|\phi^{\perp}_{n+1,w} - \phi^{\perp}_{n+1}\| < \delta\; \forall w \geq w_0, w \leq n$, where the choice of $w_0$ is independent of $n$ for arbitrarily large $n \in \mathbb{N}^+$.
\end{lemma}
\textbf{Proof Idea:} The proof leverages Corollary \ref{spikerateCorollary}, which states that the maximum number of spikes overlapping in time is bounded by a constant $d \in \mathbb{N}^+$, dependent on the \textit{ahp} parameters. Using this corollary, we partition the set of all spikes in time into a chain of subsets, where each subset overlaps only with its neighboring subsets and is disjoint from all others. Each subset contains at most $d$ spikes. The proof then demonstrates that the error in approximating $\phi^{\perp}_{n+1}$ due to windowing, i.e., $\|\phi^{\perp}_{n+1,w} - \phi^{\perp}_{n+1}\|$, decreases faster than a geometric sequence as more of these partitions are included within the window. This convergence is illustrated schematically in Figure \ref{windowedOverlap}.\\
\textbf{Proof of Lemma \ref{windowLemma}:}
    Let the set of spikes $S = \{\phi_1, ..., \phi_n\}$ be partitioned into a chain of subsets of spikes $v_1, ..., v_m$ in descending order of time, defined recursively. The first subset $v_1$ consists of all previous spikes overlapping with the support of $\phi_{n+1}$. Recursively, $v_{i+1}$ is the set of spikes overlapping with the support of any spike in $v_{i}$ for all $i \geq 1$. This process continues until the first spike $\phi_1$ is included in the final subset $v_m, \text{ where } m\leq n$. An example of this partitioning is illustrated in Figure \ref{windowedOverlap}.
    The individual spikes in each partition are indexed as follows:\\
    \vspace{-20pt}
    \begin{align*}
        v_i = \{ \phi_{p_i}, ..., \phi_{p_{i-1}-1}\}, \forall i \leq m \\ 
        \text{ where } 1 = p_m < ...<p_1 < p_0 = n+1. 
    \end{align*}  
    That is, the $i^{th}$ partition $v_i$ consists of spikes indexed  from $p_i$ to $p_{i-1} -1$. Since the spikes $\phi_1, ... , \phi_n$ are  sorted in order of their occurrence time, if both $\phi_{p_i} \text{ and } \phi_{p_{i-1}-1}$ are in partition $v_i$, then by construction,  $\phi_j \in v_i$ for all $p_i \leq j \leq p_{i-1}-1$. \\
    \vspace{-15pt}
    \begin{claim}
        \label{partitionClaim}
             The number of spikes in every partition, $v_i \; \forall i \in \{1, ..., m\}$, is bounded by some constant $d \in \mathbb{N}^+$.           
    \end{claim}
    \textbf{Proof:} This corollary follows from the observation that all spikes in a given partition $v_{i+1}$ overlap in time with at least one spike from the preceding partition $v_i$, specifically the spike whose support extends furthest into the past, for all $ i \in \{2,..,m\}$ (see figure \ref{windowedOverlap}). In the case of the partition $v_1$, each spike overlaps with $\phi_{n+1}$. Then the proof  follows from the corollary \ref{spikerateCorollary}. $\hfill\Box$\\  
    Now, let  $V_1, ... , V_m$  be subspaces spanned by the subsets of spikes $v_1, ..., v_m $ respectively.
    Before proceeding with the rest of the proof, we introduce the following notations. 
    \begin{align*}
        \Tilde{\phi_i} = \phi_i - \mathcal{P}_{\mathcal{S} (\bigcup\limits_{j = i+1}^{n}\{ \phi_{j}\})}(\phi_i) \;\;\;
        \forall i \in \{1,..., n\}
    \end{align*}
    \vspace{-10pt}\\
    where $ \Tilde{\phi_i}$ denotes the orthogonal complement of the spike $\phi_i$ with respect to all the future spikes up to $\phi_n$. We also denote,  
    \[
        \Tilde{V_k} = \mathcal{S}(\bigcup\limits_{j=p_k}^{p_{k-1} -1}\{\Tilde{\phi}_j\}) \;\;\;
        \forall k \in \{1,...,m\}
    \]   
    where $\Tilde{V_k}$ is the subspace spanned by spikes in the partition $v_k$, with each spike orthogonalized with respect to all future spikes.
    \begin{claim}
        \label{orthovectorsclaim} For the subspaces defined on the partitions as above, the following holds:
            \[\Tilde{V_k} = \{ x- \mathcal{P}_{\sum_{j =1}^{k-1} V_j} (x) | x \in V_k\} =  
            \mathcal{S}\{\Tilde{\phi}_{p_{k}},..., \Tilde{\phi}_{p_{k-1} -1} \}.\]        
    \end{claim}
    \textbf{Proof:} The proof of Claim \ref{orthovectorsclaim} follows from the properties of orthogonal projection in a Hilbert space (see \citep{supplementary_key}). $\hfill\Box$\\
        Following the Claim, we define the subspace $U_{k}$ as below:
        \begin{align}
            &U_{k} = \sum_{i=k}^{m} \Tilde{V}_i = \mathcal{S}(\bigcup\limits_{j=1}^{p_{k-1} -1} \{\Tilde{\phi}_{j}\})
            \label{lhs}\\
            & = \{ x- \mathcal{P}_{\sum_{j =1}^{k-1} V_j} (x) | x \in \sum_{i=k}^{m} V_i \} \label{rhs}
        \end{align}
        where the equality between \ref{lhs} and \ref{rhs} follows from Claim \ref{orthovectorsclaim}. Lastly, define $\phi^\perp_{n+1, v_k}$ as:
        $$\phiperp{v_k} = \phi_{n+1} - \mathcal{P}_{\sum_{i =1}^{k} V_i}(\phi_{n+1})$$
        i.e. $\phiperp{v_k}$ is the orthogonal complement of $\phi_{n+1}$ with respect to window of spikes up to partition $v_k$. Now we proceed to quantify the norm of the difference of $\phicomp$ and $\phiperp{v_k}$.  For that denote $e_k$ as follows:       
         \begin{align}
            & e_k = ||\phiperp{v_k} - \phicomp||  \nonumber\\
            & = ||\mathcal{P}_{\sum_{i =1}^{k} V_i} (\phi_{n+1}) -\mathcal{P}_{\sum_{i =1}^{m} V_i} (\phi_{n+1})|| \nonumber\\
            & = ||\mathcal{P}_{\sum_{i =1}^{k} V_i} (\phi_{n+1}) -\mathcal{P}_{\sum_{i =1}^{k} V_i +U_{k+1}} (\phi_{n+1})|| (\text{by def. of $U_{k}$})\nonumber\\
            & = ||\mathcal{P}_{\sum_{i =1}^{k} V_i} (\phi_{n+1}) -(\mathcal{P}_{\sum_{i =1}^{k} V_i } (\phi_{n+1}) + \mathcal{P}_{U_{k+1}} (\phi_{n+1}))|| \nonumber\\
            & \hspace{2cm}(\text{Since by construction $U_{k+1} \perp \Sigma_{i=1}^{k}V_i$}) \nonumber\\
            & \Rightarrow e_k= ||\mathcal{P}_{U_{k+1}} (\phi_{n+1})) ||  \label{e_k_exp}\\
            & \text{Note that by definition $e_m =0$ and by Assumption \ref{assumption2} we get,} \nonumber\\
            & e_k = ||\mathcal{P}_{U_{k+1}} (\phi_{n+1})) || \leq |\beta| < 1, \;\forall k \in \{1,...,m\} \label{boundary} \\
            &\text{Now Assume that:} \nonumber\\
            & \mathcal{P}_{U_{k+1}} (\phi_{n+1})) = \alpha_k \Tilde{\psi}_k, \text{where } \alpha_k \in \mathcal{R}, \Tilde{\psi}_k \in U_{k+1} \nonumber\\
            &\text{It follows from definition of $U_{k+1}$ that $\Tilde{\psi}_k$ is of the form:} \nonumber\\
            \nonumber
        \end{align}
        \begin{align}   
        & \Tilde{\psi}_k =  \psi_k - \mathcal{P}_{\sum_{j =1}^{k} V_i } (\psi_k) \text{ for some } \psi_k \in \sum_{i =k+1}^{m} V_i \nonumber\\ 
            &\text{\emph{w.l.o.g.}  assume $||\psi_k|| =1$ by appropriately choosing $\alpha_k$.}\nonumber\\
            & \text{Since $\alpha_k \Tilde{\psi}_k$ is projection of $\phi_{n+1}, \alpha_k = \frac{\langle \phi_{n+1}, \Tilde{\psi}_k\rangle}{||\Tilde{\psi}_k||^2}$} \label{coeff}\\
            & \text{For $k>0$ combining \ref{coeff} and \ref{e_k_exp} we obtain,}\nonumber\\
            & e_k = ||\alpha_k \Tilde{\psi}_k|| = |\alpha_k|||\Tilde{\psi}_k|| = \frac{|\langle \phi_{n+1}, \Tilde{\psi}_k\rangle|}{||\Tilde{\psi}_k||}\nonumber\\
            & = \frac{|\langle \phi_{n+1}, \psi_k - \mathcal{P}_{\sum_{j =1}^{k} V_j } (\psi_k)\rangle|}{||\Tilde{\psi}_k||}\nonumber\\
            & = \frac{|\langle \phi_{n+1}, \mathcal{P}_{\sum_{j =1}^{k} V_j} (\psi_k)\rangle|}{||\Tilde{\psi}_k||} 
            =\frac{|\langle \phi_{n+1}, \mathcal{P}_{\sum_{j =1}^{k-1} V_j + \Tilde{V}_k} (\psi_k)\rangle|}{||\Tilde{\psi}_k||}
            \nonumber\\
            & \text{(for $k>0, \psi_k \in\sum_{i =k+1}^{m} V_i \perp \phi_{n+1}$ due to disjoint support)} \nonumber\\
            & =\frac{|\langle \phi_{n+1}, \mathcal{P}_{\sum_{j =1}^{k-1} V_j} (\psi_k)+ \mathcal{P}_{\Tilde{V}_k} (\psi_k)\rangle|}{||\Tilde{\psi}_k||} \;\text{(by def. $\Tilde{V}_k \perp \sum_{j =1}^{k-1} V_j$)}\nonumber\\
            & \Rightarrow e_k = \frac{|\langle \phi_{n+1}, \mathcal{P}_{\Tilde{V}_k} (\psi_k)\rangle|}{||\Tilde{\psi}_k||} \;
            \text{($k>0,\psi_k \in \sum_{i=k+1}^{m} V_i \perp \sum_{j=1}^{k-1} V_j$)}
            \label{finalE_k}\\ 
            & \text{Likewise, } e_{k-1} = ||\mathcal{P}_{U_{k}} (\phi_{n+1})) ||  = ||\mathcal{P}_{U_{k+1}+ \Tilde{V}_k} (\phi_{n+1})) ||\nonumber\\
            & \Rightarrow e_{k-1}^2 = ||\mathcal{P}_{U_{k+1}} (\phi_{n+1})) ||^2 + ||\mathcal{P}_{\Tilde{V}_k} (\phi_{n+1})) ||^2
            \nonumber\\
            & \hspace{2cm} 
            \text{(Since by construction $U_{k+1} \perp \Tilde{V}_k$)} \nonumber\\
            & \Rightarrow e_{k-1}^2 = e_k^2 +  ||\mathcal{P}_{\Tilde{V}_k} (\phi_{n+1})) ||^2 \label{ek11}\\
            & \text{Assume that } \mathcal{P}_{\Tilde{V}_k} (\phi_{n+1}) = \beta_k \Tilde{\theta}_k, \text{ where }  \Tilde{\theta}_k \in \Tilde{V}_k, \beta_k \in \mathcal{R} \nonumber\\
            & \Rightarrow ||\mathcal{P}_{\Tilde{V}_k} (\phi_{n+1})) || = ||\beta_k \Tilde{\theta}_k|| = \frac{|\langle \phi_{n+1}, \Tilde{\theta}_k \rangle|}{||\Tilde{\theta}_k||} \label{ptheta}\\
            &\text{From \ref{ek11} \& \ref{ptheta} we get, } e_{k-1}^2 = e_k^2 + \frac{|\langle \phi_{n+1}, \Tilde{\theta}_k \rangle|^2}{||\Tilde{\theta}_k||^2} \label{e_k-1}\\
            &\text{Again, for $k >0$ we further analyze $e_k$ from \ref{finalE_k} to obtain:}\nonumber\\
            & e_k =  \frac{|\langle \phi_{n+1}, \mathcal{P}_{\Tilde{V}_k} (\psi_k)\rangle|}{||\Tilde{\psi}_k||} 
            \nonumber\\
            &= \frac{|\langle \phi_{n+1}, \mathcal{P}_{\mathcal{S}(\{\Tilde{\theta}_k\})} (\psi_k)
            + \mathcal{P}_{\Tilde{V_k}\setminus{\mathcal{S}(\{\Tilde{\theta}_k\})}} (\psi_k)\rangle|}{||\Tilde{\psi}_k||}
            \label{lastline}\\ \nonumber
            \vspace{-20pt}
             &\text{The expression \ref{lastline} is essentially written by breaking $\psi_k$ }\nonumber\\
             &\text{into two mutually orthogonal subspaces: $\mathcal{S}(\{\Tilde{\theta}_k\})$, subspace }\nonumber\\
             &\text{of $\Tilde{V_k}$ spanned by $\Tilde{\theta}_k$, and $\Tilde{V_k}\setminus\mathcal{S}(\{\Tilde{\theta}_k\}$, the subspace of $\Tilde{V_k}$}\nonumber\\
            &\text{orthogonal to the subspace $\mathcal{S}(\{\Tilde{\theta}_k\})$. Also, observe that} \nonumber\\ 
             & \phi_{n+1} \perp \Tilde{V_k}\setminus\mathcal{S}(\{\Tilde{\theta}_k\}) \text{ since } \mathcal{P}_{\Tilde{V}_k} (\phi_{n+1})) = \beta_k \Tilde{\theta}_k \text{. Therefore,}\nonumber\\   
             & e_k = \frac{|\langle \phi_{n+1}, \mathcal{P}_{\mathcal{S}(\{\Tilde{\theta}_k\})} (\psi_k)|}{||\Tilde{\psi}_k||}\;\; 
             = \frac{|\langle \phi_{n+1}, \frac{\langle \psi_k,\Tilde{\theta}_k\rangle \Tilde{\theta}_k}{||\Tilde{\theta}_k||^2}\rangle|}{||\Tilde{\psi}_k||}
             \nonumber\\
                & \Rightarrow e_k= \frac{|\langle \phi_{n+1}, \Tilde{\theta}_k\rangle|}{||\Tilde{\theta}_k||} \frac{|\langle \psi_k,\frac{\Tilde{\theta}_k}{||\Tilde{\theta}_k||}\rangle|}{||\Tilde{\psi}_k||}\label{e_k}\\
             &\text{Combining \ref{e_k-1} and \ref{e_k} we obtain:}\nonumber\\ 
             & e_{k-1}^2 = e_k^2 + e_k^2 \frac{||\Tilde{\psi}_k||^2}{|\langle \psi_k, \frac{\Tilde{\theta}_k}{||\Tilde{\theta}_k||}\rangle|^2}\label{sequenceRel}\\
            \nonumber
        \end{align}
        \begin{align}                       
            &\text{Since both $\psi_k$ and $\frac{\Tilde{\theta}_k}{||\Tilde{\theta}_k||}$ are unit norm $|\langle \psi_k, \frac{\Tilde{\theta}_k}{||\Tilde{\theta}_k||}\rangle| \leq1$.}\nonumber\\
            & \Rightarrow e_{k-1}^2 \geq e_k^2(1+ ||\Tilde{\psi}_k||^2)\label{geoDrop}\\\nonumber
            \nonumber
        \end{align}
        \vspace{-30pt}\\
        Eq. \ref{geoDrop} demonstrates that the sequence $\{e_k\}$ converges geometrically
        to 0 when $(1+ ||\Tilde{\psi}_k||^2) > 1$, i.e. $||\Tilde{\psi}_k|| >0 $. To complete the proof of Lemma \ref{windowLemma}, we need to establish a positive lower bound on $||\Tilde{\psi}_k||$, ensuring the convergence of the sequence ${e_k}$. The following Corollary provides the necessary lower bound on $||\Tilde{\psi}_k||$.
        \vspace{-5pt}
        \begin{claim}
            \label{psiknormboundClaim}
            Following Claim \ref{partitionClaim}, if the number of spikes in each partition $v_i$ is bounded by $d$, then $|| \Tilde{\psi}_k||^2 \geq 1-  \beta_d^2$, where $\beta_d >0 $ is as defined in Lemma \ref{windowCondLemma}.
        \end{claim}   
        \vspace{-5pt}
        \textbf{Proof:} Since $\psi_k$ is a unit vector in $\sum_{i =k+1}^{m} V_i $ and has disjoint support with respect to the future partitions $v_{k-1},..., v_1$, except for the partition the $v_k$, the projection of $\psi_k$ on $V_k$ and the application of Claim \ref{partitionClaim} and Lemma \ref{windowCondLemma} lead to the result. See \citep{supplementary_key} for details. $\hfill\Box$\\    
Finally, combining \ref{geoDrop} with Claim \ref{psiknormboundClaim}, we obtain:
\vspace{-5pt}
        \begin{equation}            
         e_{k-1}^2 \geq e_k^2(1+ (1-  \beta_d^2)) \Rightarrow e_k^2 \leq \frac{e_{k-1}^2}{\gamma^2}\label{finalGeoSeq}
        \end{equation}
        \vspace{-13pt}
        \\
    where $\gamma^2 = (1+ (1-  \beta_d^2))$ is a constant strictly greater than 1. Since $e_m =0$, Eq. \ref{finalGeoSeq} shows that the sequence $\{e_k\}_{k=1}^m$ converges to 0 faster than geometrically. 
    Thus, $\forall \delta >0, \exists k_0 \in \mathbb{N}^+ \text{such that } e_k < \delta, \text{ for all } k\geq k_0 \text{ and } m\geq k$. Given the geometric drop in Eq. \ref{finalGeoSeq} and the bound $e_1 <1$ (Eq. \ref{boundary}), for arbitrarily large $m$ (hence $n$) the choice of $k_0$ is independent of $m$ and depends only on $\beta_d$, which is determined by the \textit{ahp} parameters in Eq. \ref{thresholdeq}). Since the number of spikes in each partition is bounded by d (Claim \ref{partitionClaim}), choosing a window size $w_0 = k_0*d$ ensures $\forall \delta >0, \exists w_0 \in \mathbb{N}^+  \text{ such that }  ||\phi^{\perp}_{n+1,w} -\phi^{\perp}_{n+1}|| < \delta \; \text{ for all } w \geq w_0 \text{ and } w \leq n$. For arbitrarily large $n$, the choice of $w_0$ is independent of $n$.$\hfill\Box$
    \vspace{5pt}\\
    \textbf{Proof of Theorem \ref{windowthm}}: Having established a bound on the
    norm of the difference between \(\phi^{\perp}_{n+1}\) and \(\phi^{\perp}_{n+1,w}\), we need 
    to bound the norm of the difference between the projections
    of the input signal \(X\) with respect to these vectors.
    Specifically, we seek to bound \(||\mathcal{P}_{\phi^{\perp}_{n+1,w}}(X) - \mathcal{P}_{\phi^{\perp}_{n+1}}(X)||\)
    based on the window size. We use the following notations:
    $
    \Tilde{X} = \mathcal{P}_{\mathcal{S}(\{\phi^{\perp}_{n+1,w}, \phi^{\perp}_{n+1}\})}(X),
    X_u = \mathcal{P}_{\phi^{\perp}_{n+1}}(X), X_v = $
    $\mathcal{P}_{\phi^{\perp}_{n+1,w}}(X)$, so that $\Tilde{X}, X_u$ and $X_v$ lie in a plane (fig. \ref{windowedProjection}).    
    Assume the angle between $\Tilde{X}$ and $\phi^{\perp}_{n+1}$ is $a$, and the angle between $\Tilde{X}$ and $\phi^{\perp}_{n+1,w}$ is $b$. Hence, the angle between $\phi^{\perp}_{n+1,w}$ and $\phi^{\perp}_{n+1}$ is $a-b$ (fig. \ref{windowedProjection}). Note that  $X$ may not lie in the same plane as $X_u$ and $X_v$; thus, we consider the projection $\Tilde{X}$ of $X$ on the plane. We denote $p_w =\phi^{\perp}_{n+1,w} - \phi^{\perp}_{n+1}$. Further analysis shows  $p_w = \mathcal{P}_{\mathcal{S}(\{\Tilde{\phi}_1,..., \Tilde{\phi}_w\}}(\phi^{\perp}_{n+1,w})$, and thus $p_w \perp \phi^{\perp}_{n+1}$ (see \citep{supplementary_key} for details). Therefore, from figure \ref{windowedProjection} we observe:
    \vspace{-5pt}
    \begin{equation}
    \sin^2(a-b) = \frac{||p_w||^2}{||\phiperp{w}||^2} \label{sinValue}            
    \end{equation}
    \vspace{-10pt}
    \\
    Now we can quantify the norm of the difference in two the projections of $X$ w.r.t. $\phicomp$ and $\phiperp{w}$ as follows:
    \vspace{-5pt}
    \begin{align}
      &||\mathcal{P}_{\phi^{\perp}_{n+1,w}}(X)- \mathcal{P}_{\phi^{\perp}_{n+1}}(X)|| = ||X_u - X_v||^2 \nonumber\\
      &=||X_u||^2 + ||X_v||^2 -2 ||X_u||||X_v||cos(a-b) \nonumber\\
      &=||\Tilde{X}||^2[ \cos^2a + \cos^2b - 2\cos a\cos b\cos(a-b) ] \;\; (\text{see fig. \ref{windowedProjection}})\nonumber\\
      &=||\Tilde{X}||^2 \sin ^2 (a-b) = ||\Tilde{X}||^2\frac{||p_w||^2}{||\phiperp{w}||^2}  \nonumber\\
      & \;\; (\text{Using trigonometric identities \& eq. \ref{sinValue}, see \citep{supplementary_key} for details}) \nonumber\\
      & \leq \frac{||X||^2}{1- \beta ^2} ||\phiperp{w}-\phicomp||^2 \;\;\ (\text {Since $||\Tilde{X}|| \leq||X|| $ and} \nonumber\\
      & \quad \text {$||\phiperp{w}||^2 \geq ||\phicomp||^2 \geq (1-\beta^2)$ by assumption}) \nonumber\\\nonumber
    \end{align}
    \vspace{-30pt}
    \\
    Finally, since the input signal has a bounded $L_2$ norm (i.e. $||X||$ is bounded),  setting $\delta = \epsilon\frac{\sqrt{1-\beta^2}}{||X||}$ in Lemma \ref{windowLemma} gives a window size $w_0$ such that $||\mathcal{P}_{\phi^{\perp}_{n+1,w}}(X)- \mathcal{P}_{\phi^{\perp}_{n+1}}(X)|| \leq \frac{||X||}{\sqrt{1- \beta ^2}} ||\phiperp{w}-\phicomp|| < \epsilon, \forall \epsilon >0, \forall w \geq w_0\text{ and } w\leq n $. The choice of $w_0$ is independent of $n$ for arbitrarily large $n \in \mathbb{N}$. This completes the proof of the theorem. $\hfill\Box$
\begin{figure}[h]
 	\centering
  	{{\includegraphics[width=90mm, height=50mm]{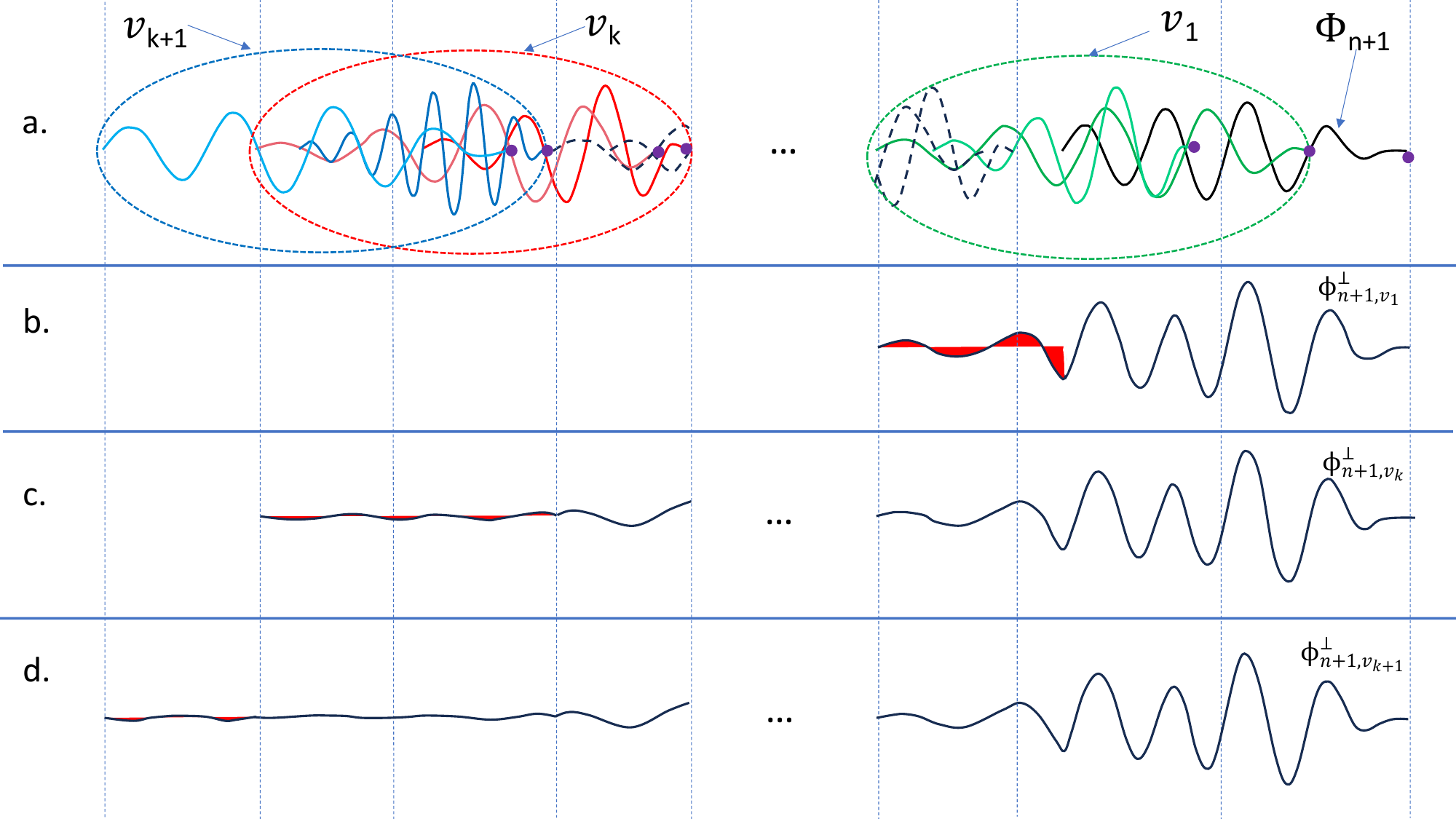}}}
 	\caption{Illustration of Lemma \ref{windowLemma}. The diagram shows the convergence of the windowed orthogonal complement $\phiperp{w}$ of spike $\phi_{n+1}$ to $\phicomp$ by orthogonalizing $\phi_{n+1}$ across the partitions of spikes. (a) Displays all spikes up to $\phi_{n+1}$ (black), with partitions circled: $v_1$ (green), $v_k$ (red), and $v_{k+1}$ (blue). Spikes contained in each partition are shaded accordingly, with the time of each spike marked by a purple dot. (b), (c), and (d) show orthogonal complements $\phiperp{v_1}$, $\phiperp{v_k}$, and $\phiperp{v_{k+1}}$ respectively. The support of $\phiperp{w}$ extends as more partitions are included, with the extending tail for each additional partition highlighted in red. This tail's diminishing energy as more partitions are added illustrates Lemma \ref{windowLemma}.}
 	\label{windowedOverlap}
\end{figure}
\begin{figure}[t]
 	\centering
   \vspace{-30pt}
  	{{\includegraphics[width=90mm, height=50mm]{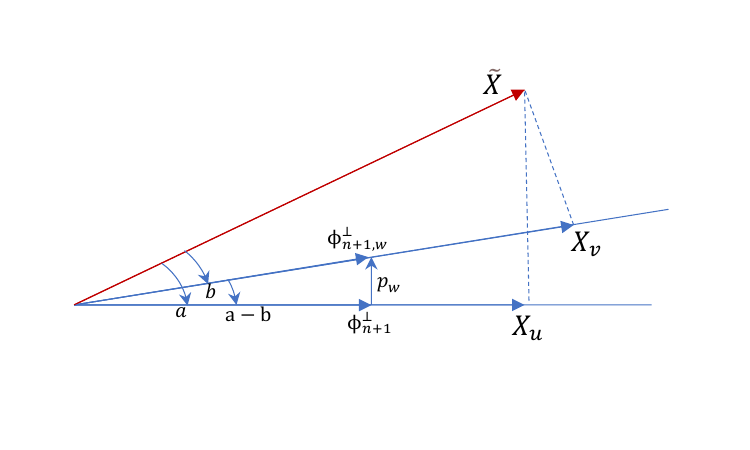}}}
 \vspace{-50pt}
 	\caption{Figure illustrating the vector projections of the input signal \(X\) onto vectors \(\phi^{\perp}_{n+1}\) and \(\phi^{\perp}_{n+1,w}\). The red vector represents \(\Tilde{X}\), the projection of \(X\) within the plane formed by \(\phi^{\perp}_{n+1}\) and \(\phi^{\perp}_{n+1,w}\). The vectors \(\phi^{\perp}_{n+1}\) and \(\phi^{\perp}_{n+1,w}\), as well as the projections \(X_u\) and \(X_v\) of \(\Tilde{X}\) onto them, are indicated in blue. The vector \(p_w\), representing the difference between \(\phi^{\perp}_{n+1,w}\) and \(\phi^{\perp}_{n+1}\), is also shown in blue. The angles \(a\) between \(\Tilde{X}\) and \(X_u\), \(b\) between \(\Tilde{X}\) and \(X_v\), and \(a-b\) between \(\phi^{\perp}_{n+1,w}\) and \(\phi^{\perp}_{n+1}\) are marked.}
 	\label{windowedProjection}
\end{figure}
\begin{figure}[h]
 	\centering
    \vspace{-15pt}
  	{{\includegraphics[width=90mm, height=25mm]{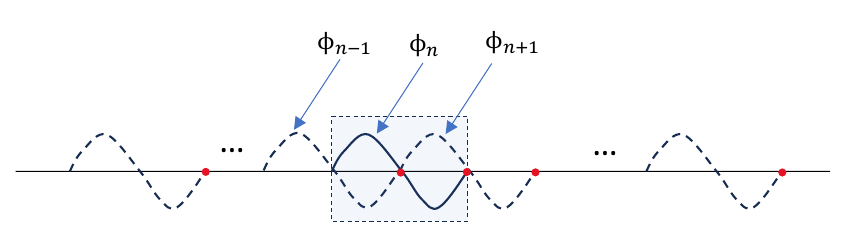}}}
   \vspace{-10pt}
 	 \caption{The scenario illustrating the need for Assumption \ref{assumption2}. See text for details. For derivation of $||\mathcal{P}_{\mathcal{S}(\bigcup_1^{N}\{\phi_i\} \setminus \{\phi_n\})}(\phi_n)|| \rightarrow 1$ see \citep{supplementary_key}.}
 	\label{sineKernelsOverlap}
  \vspace{-20pt}
\end{figure}

\section{Experiments on Real Signals} \label{Experiments}
The proposed framework was tested on audio signals.
\textbf{Dataset:}
 We chose the Freesound Dataset Kaggle 2018, an audio dataset of natural sounds referred in \citep{dataset}, containing 18,873 audio files. All audio samples in this dataset are provided as uncompressed PCM 16bit, 44.1kHz, mono audio files. 
 \textbf{Set of Kernels:} 
We chose gammatone filters ($at^{n-1}e^{-2\pi bt}\cos(2\pi ft+\phi)$) in our experiments since they are widely used as a reasonable model of cochlear filters in auditory systems \citep{pattersonFilters}.
\textbf{Results:} 
\label{comprehensiveExpt}
The proposed framework was tested extensively against the mentioned dataset. Comprehensive results with 600 randomly selected audio snippets using 50 kernels are shown in Figure \ref{reconstructionStatistics}. The complete source code used for the experiment is available in GitHub \citep{github_repo}. 
In the experiment, kernels were normalized, and parameters for the time-varying threshold function (\ref{lockstepspikes}) were selected through a systematic grid search on a smaller dataset of 20 randomly chosen snippets. In each trial, an audio snippet of length $\approx 2.5s$ was processed with fixed parameter values, except for the refractory period, which was  gradually decreased leading to improvement in reconstructions at higher spike rates. The refractory period varied from approximately $250ms$ to $5ms$. After converting an audio snippet into a sequence of spikes, reconstruction was performed iteratively using a fixed-sized window of spikes as described in Section \ref{windowing}. The \textit{ahp} period was systematically varied, but each trial on an audio snippet was run with a single value of the \textit{ahp} period. Correspondingly, the window size was set to be inversely proportional to the \textit{ahp} period, varying from 5k to 15k as per Theorem \ref{windowthm}. Exact parameter values and window sizes are detailed in the GitHub documentation \cite{github_repo}.
The results in Figure \ref{reconstructionStatistics} show that increasing the spike rate by tuning the refractory period allows near-perfect reconstruction, aligning with our theoretical analysis. Some variability in reconstruction accuracy across signals can be attributed to dataset idiosyncrasies, e.g. certain audio samples could be noisy or ill-represented in the kernels. However, the overall trend shows promise, with an average of $\approx 20 dB$ at $1/5$th Nyquist Rate. This in conjunction with the fact that signals are represented in this scheme only via set of spike times and kernel indexes (thresholds can be inferred) shows potential for an extremely efficient coding mechanism.
Since the generation of spikes requires scanning through convolutions in a single pass, encoding is highly efficient. However, decoding is slightly more time-consuming because it involves solving the linear system $P\alpha = T$ to derive the coefficients. But then reconstruction is performed iteratively on a finite window, as described in Section \ref{windowing}. Considering $O(w^3)$ as the time complexity for inverting a $w\times w$ matrix, the overall time complexity of decoding is $O(Nw^3)$, where $N$ is the length of the signal and $w$ is the chosen window size. Thus the overall process still remains linear, making it a suitable choice for lengthy continuous-time signals.\\
\textbf{Comparison With Convolutional Orthogonal Matching Pursuit:} Our proposed framework is comparable to Convolutional Sparse Coding (CSC) techniques  \citep{CristinaCSCReview}. In CSC, the objective is to efficiently represent signals using a small number of basis functions convolved with sparse coefficients. Mathematically, for a given signal $Y$, the CSC model can be expressed as: $Y \approx \Sigma_{k=1}^{K} d_k \star x_k$ where $d_k$ are the convolutional filters, $x_k$ are the sparse feature maps, $K$ is the number of filters and $\star$ denotes convolution. The similarity with our proposed framework becomes clear when we consider our reconstruction formulation: $X^* = \sum_{i=1}^{N} \alpha_i \Phi^{j_i}(t_i-t)$. Here, the kernels $\Phi^k$ serve as the convolution filters $d_k$, and the sparse feature map $x_k$ in our framework can be likened to a vector of coefficients $\alpha_i$, where each $\alpha_i$ corresponds to the spikes of the kernel $\Phi^k$, and is time-shifted to the appropriate occurrences of spikes. In other words, our framework finds a sparse representation of a signal through its own biological spiking mechanism. Finding sparse code for signals, in general is an NP-Hard problem \citep{Davis1997}, and several heuristic-based approaches are used to address this challenge. In \citep{beyondRateCode}, our proposed framework was compared against one such heuristic-based implementation of CSC, specifically Convolutional Orthogonal Matching Pursuit (COMP), using a small dataset of about 20 audio snippets. COMP employs a greedy technique to iteratively find dictionary atoms and is relatively slow due to the orthogonalization it performs to the atoms in each step. Most of the current leading CSC algorithms \citep{AdvCSC1, AdvCSC2, AdvCSC3} use $L_1$ regularization as a relaxation of $L_0$ regularization in their sparse reconstruction objectives, making them amenable to efficient optimization methods such as the Alternating Direction Method of Multipliers (ADMM). In this paper, we extensively compare our technique  with the efficient CBPDN algorithm implemented within the state-of-art SPORCO python library \citep{sporco}.
In this experiment, we used 200 randomly chosen audio snippets from the same dataset, utilizing 10 gammatone kernels. The snippet lengths varied from 0.5s to 3.5s to compare processing times as a function of snippet length. Figure \ref{comparison} shows the reconstruction accuracy comparison between our framework and CBPDN. CBPDN, which is based on $L_1$ optimization, achieved reconstructions at different sparsity levels by varying the regularization parameter $\lambda$, while  our framework achieved different spike rates by varying the \textit{ahp} period (see \citep{github_repo} for exact parameter values). The results in Figure \ref{comparison} reveal that our framework outperforms CBPDN, particularly in the low spike rate regime, achieving better SNR values on average at consistently lower spike rates. In the high spike rate regime, CBPDN occasionally outperforms our framework. This occurs because our framework is not well suited for high spike rates, where spikes begin to overlap significantly, leading to poorer bounds as discussed in our theoretical analysis. Additionally, the limited datapoints where CBPDN outperforms (around 30 dB) make these specific results unreliable. 
Figure \ref{runtimeComp} compares the runtimes of our framework and CBPDN on a 10-core Intel(R) Xeon(R) CPU E5-2650 v3 server. The average processing time, shown as a function of snippet length, demonstrates that our framework also outperforms CBPDN in terms of runtimes, particularly in the aforementioned low spike rate regime. Although the processing time is platform-dependent, our framework offers better asymptotic complexity. Leading CSC algorithms  \citep{AdvCSC1, AdvCSC2, AdvCSC3} transform long temporal signals  to the Fourier domain, resulting in $O(N \log N)$ complexity due to  FFT, whereas our method achieves $O(Nw^3)$. 
While the overall processing times (considering both encoding and decoding) of both frameworks are comparable, our framework encodes signals into spike times very efficiently using a convolve-and-threshold mechanism in minimal time. This is particularly useful in online settings, such as when a signal needs to be streamed. In such cases, our framework allows signals to be efficiently conveyed via only the spike times and indices (with thresholds or coefficients inferred). In contrast, CBPDN and other CSC algorithms, being inherently based on optimization, cannot achieve such efficient encoding and require both the coefficients and the timing of the atoms for signal representation. Overall, our framework shows promise as a superior alternative to CSC techniques in specific scenarios, particularly for low spike rate representation of natural audio signals with biological kernels, as demonstrated in our experiments.
\section{Conclusion}
\label{Conclusion}
The experimental results establish the efficiency and robustness of the proposed spike-based encoding framework, which clearly outperforms state-of-the-art CSC techniques in the low spike rate regime. Notably, this high-fidelity coding and reconstruction is achieved through a simplified abstraction of a complex biological sensory processing system, which typically involves numerous neurons across multiple layers with diverse goals—such as feature extraction, decision making, classification and more—rather than being limited to reconstruction. For context, the human auditory system's cochlear nerve contains about 50,000 spiral ganglion cells (analogous to 50,000 kernels). The fact that our framework, with a single layer of roughly 100 neurons using a simple convolve-and-threshold model, achieves such high-quality reconstruction underscores the potential of fundamental biological signal processing principles. Our framework differs fundamentally from the Nyquist-Shannon theory, primarily in its mode of representation and coding. Instead of sampling the value of a function at uniform or non-uniform pre-specified points, our coding scheme identifies the non-uniform points where the function takes specific convolved values. This efficient coding scheme, combined with our proposed window-based fast processing of continuous-time signals, shows great promise for achieving significant compression in real-time signal communication.

\begin{figure}
\centering
  	{{\includegraphics[width=90mm, height=180mm]{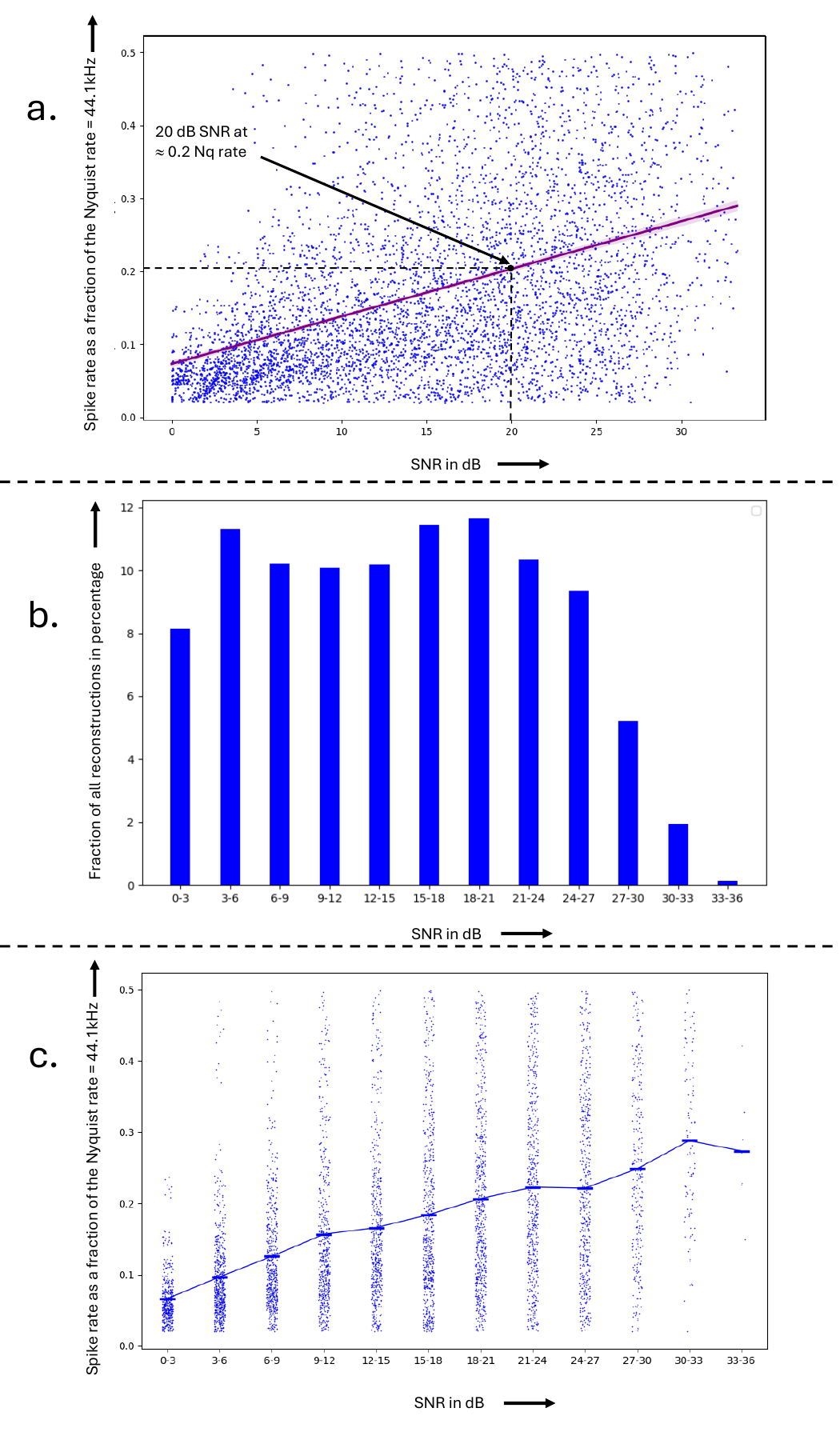}}}%
 	 \caption{Comprehensive results of experiments on 600 audio snippets with 50 kernels. \textbf{(a)} Scatter plot of reconstructions where each dot represents a single reconstruction performed on one of the 600 sound snippet for a particular setting of the \textit{ahp} parameters as described in Section ~\ref{comprehensiveExpt}. The plot shows the SNR value of the reconstructions (x-axis) against corresponding spike-rate of the ensemble (y-axis). The trend line in purple is generated using seaborn regression fit, and the black dot on the line highlights the point on the trend line at 20 dB SNR which has an average spike-rate of approximately one-fifth the Nyquist rate.
   \textbf{(b)} Bar graph showing the distribution of all reconstructions from (a) at different SNRs. Reconstructions are binned at intervals of 3dB in the range 0 to 30 dB SNR (x-axis) and the bars show the percentage of all reconstructions that fall within the corresponding bin (y-axis). As is evident, the maximum fraction of reconstructions falls in the bin of 18-21 dB. 
   \textbf{(c)} Strip plot showing the distribution of spike rates across all the bins from (b). Like (b), the x-axis shows the SNR values binned at intervals of 3dB and the y-axis shows the spike rates of reconstructions as a fraction of the Nyquist rate, 44.1kHz. The blue line, passing through the strip plot, connects the average spike rates of reconstructions in each bin (averages of each bin are shown by blue markers). 
   }%
 	\label{reconstructionStatistics}
\end{figure}

\begin{figure}
    \vspace{-10pt}
 	\centering
 	\includegraphics[width=90mm, height=120mm]{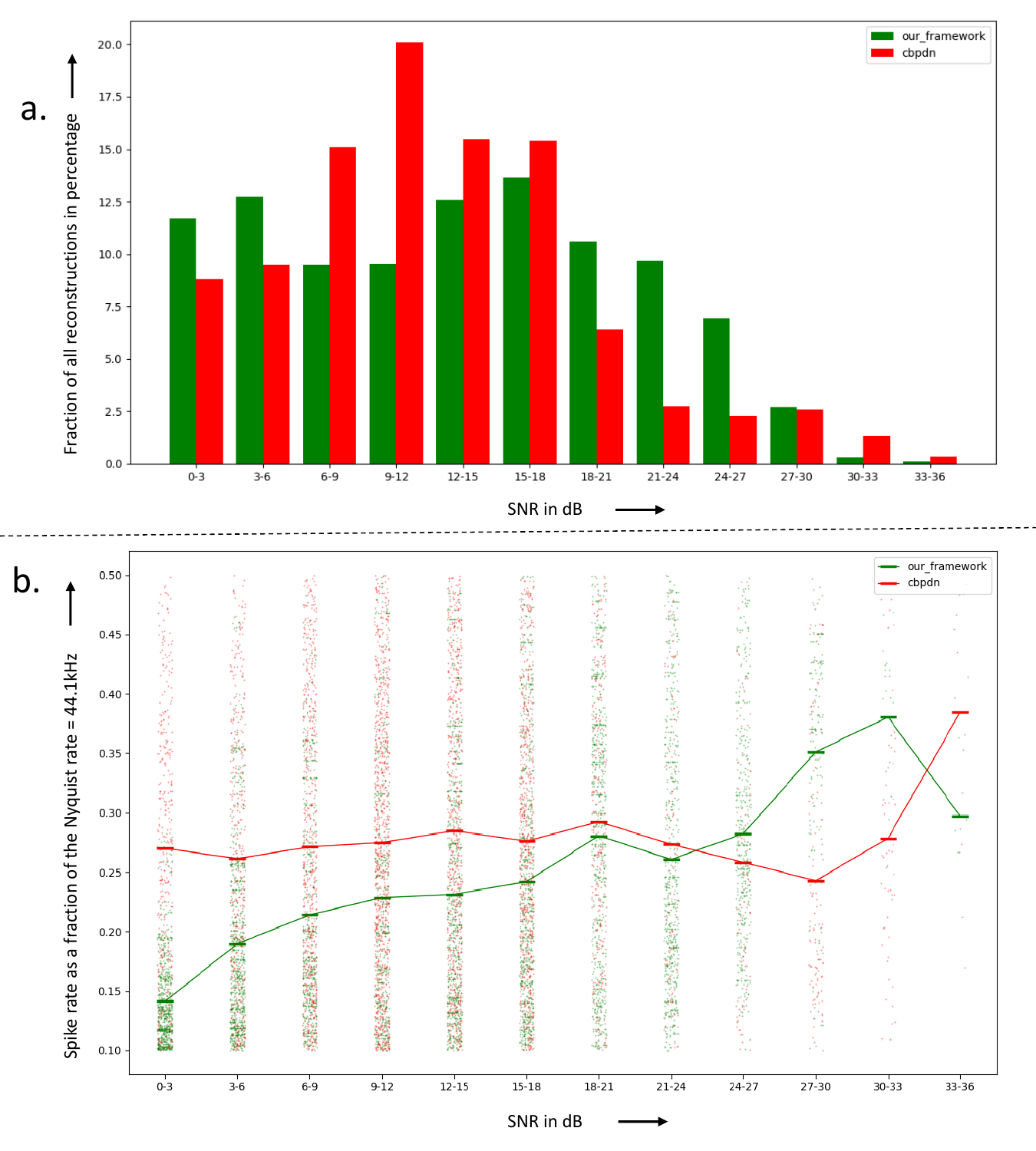}
 	\caption{Experimental results of comparison between our framework and cbpdn in terms of overall reconstruction accuracy and spike rate. \textbf{(a)} Bar plot showing a comparison of reconstruction accuracy between our framework and cbpdn. The x-axis displays SNR values binned at intervals of 3dB over the range of 0-36 dB, with each bar representing the percentage fraction of all reconstructions falling into the corresponding bin (red bars for cbpdn and green bars for our framework). As is evident, the maximum fraction of reconstructions for cbpdn falls in the 9-12dB bin as opposed to the 15-18dB bin for our framework. Also our framework produced significantly larger number of reconstructions in the high SNR region (right side of the graph where the green bars are noticeably taller than the red bars).
  \textbf{(b)} Strip plot comparing spike rates of reconstruction between our framework and cbpdn. Like (a), the x-axis bins SNR values, and the strips of dots (red dots for cbpdn and green dots for our framework) in each bin show the distribution of reconstructions for that bin. The y-axis denotes the spike-rate of the ensemble as a fraction of the Nyquist rate for each reconstruction. The two lines through the strip plot connect the corresponding average spike rates of each bin for our framework (green) and cbpdn (red). As is evident, in the low spike rate regime the green line lies well below the red line, indicating a much lower spike rate achieved by our framework. Although briefly in the higher SNR regime, cbpdn outperforms our framework, the very few reconstruction data points available for cbpdn, as evidenced in (a), makes it challenging to draw reliable conclusions.
  }%
 	\label{comparison}
\end{figure}
\begin{figure}
    \vspace{-10pt}
 	\centering
 	\includegraphics[width=90mm, height=60mm]{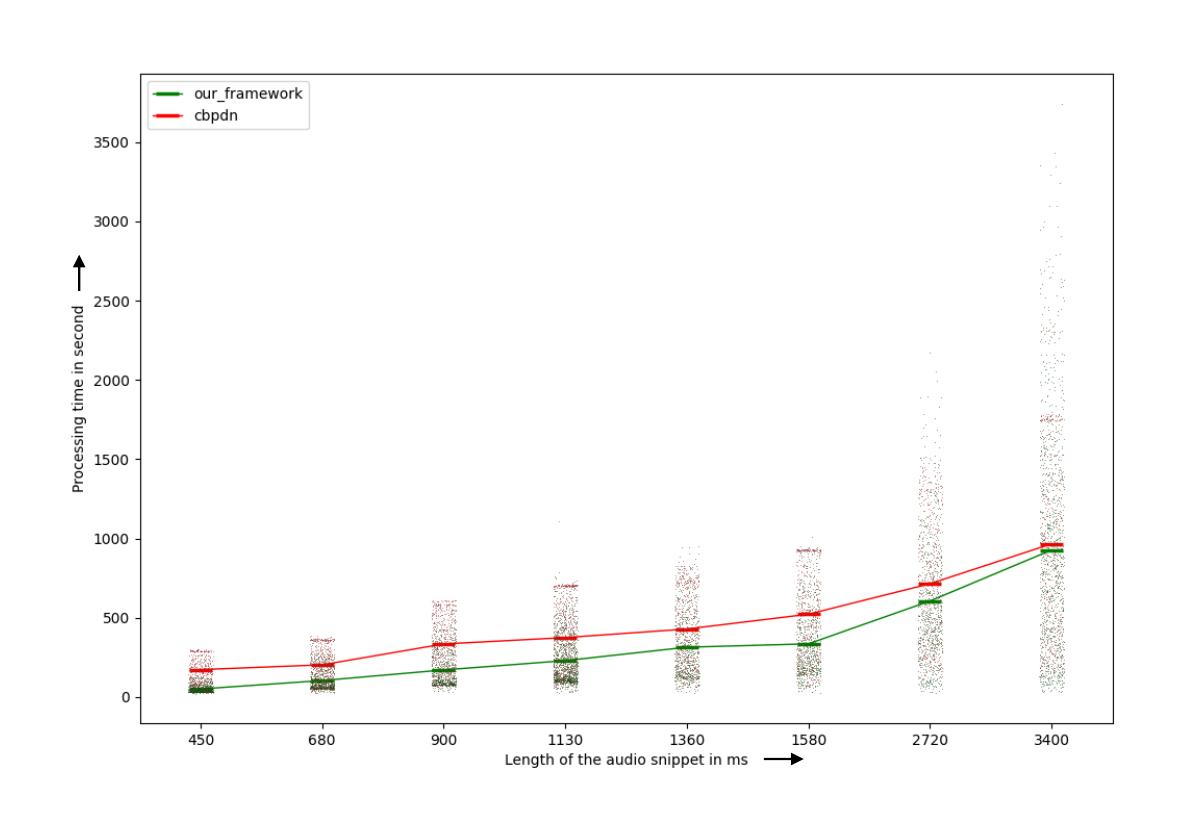}
 	\caption{Comparison of runtimes between our framework and cbpdn. The x-axis shows the different lengths of the audio snippets in msec and the strips of dots (red dots for cbpdn and green dots for our framework) represent the distribution of reconstruction runtimes corresponding to each length. The two lines connect the average processing times at each length for both our framework (green) and cbpdn (red). In terms of overall runtime our framework outperforms cbpdn. Although there are a few points produced by our framework in the high spike regime with higher processing times due to large matrix inversions for higher spike rates, this is greatly dependent on the hardware used.}%
 	\label{runtimeComp}
\end{figure}

{\appendices
\section{Proof of Corollary \ref{convolutionCorollary}}
\label{convolutionCorollarySupp}
\textbf{Corollary \ref{convolutionCorollary}:}    Let $\Phi^j $ be a function in $ C[0, \tau], \tau \in \mathbb R^{+}$ and $||\Phi^j||_2=1$. Let $X(t) \in \mathcal{F} = \{f(t)| t \in [0,\tau'], |f(t)|\leq b\},$ where $b, \tau' \in \mathbb R^+$, be the input to our model. Then: 
    (a) The convolution $C^j(t)$ between $X(t)$ and $\Phi^j(t)$ defined by $C^j(t) = \int X(t') \Phi^{j}(t-t') dt',$ for $ t\in [0, \tau+\tau']$, is a bounded and continuous function. Specifically, one can show that $|C^j(t)| < b\sqrt{\tau}$  for all $t \in [0, \tau+\tau']$ .
    (b) Suppose the parameter $M$ in the threshold equation \ref{thresholdeq} is chosen such that $M > 2b\sqrt{\tau}$. Then the interspike interval between any two spikes produced by the given neuron $\Phi^j$ is greater than $\frac{\delta}{2}$.\\\\
\textbf{Proof:}
(a)   For showing continuity, we observe that $\Phi^j \in C[0, \tau]$ is uniformly continuous. Therefore, for $t,h \in \mathbb R$ we have:
   \begin{align}
       & |C^j(t)- C^j(t+h)|  = \nonumber\\  
         &\hspace{1cm}|\int \Phi^j(t-t')X(t')dt' -\Phi^j(t+h-t')X(t')dt'| \nonumber\\
       \nonumber
    \end{align}
    \begin{align}
       & \leq (\sup_{t \in \mathbb R} |\Phi^j(t) - \Phi^j(t+h)|) \int |X(t')|dt' \nonumber\\
       &\leq ||\Phi^j(t) - \Phi^j(t+h)||_\infty||X||_1 \to 0 \text{ as } h \to 0 \nonumber\\
       \nonumber
   \end{align}
   where the convergence follows from the fact that $\Phi^j(t)$ is uniformly continuous and $||X||_1$ is a bounded quantity as the input  is a bounded function on compact support. A bound on $|C^j(t)|$ can be shown as follows-
    \begin{align}
    &|C^{j}(t)| = |\int X(\tau) (\Phi^{j}(t-\tau))d\tau| \leq \int |X(\tau)| |(\Phi^{j}(t-\tau))|d\tau\nonumber\\
    &< b \int|(\Phi^{j}(t-\tau)| d\tau \;\;(\text{since} |X(t)| < b)\nonumber\\
    &= b \int_0^\tau 1.|(\Phi^{j}(\tau)| d\tau = b||\Phi^j||_2\sqrt{\int_0^\tau 1 d\tau} \leq b\sqrt{\tau} \label{convIneq}\\
    &(\text{Using Cauchy-Schwarz inequality and } ||\Phi^j||_2=1)  \nonumber         
    \end{align}        
    
    (b) Assume that $t_0, t_1,..., t_k, t_{k+1}, ...t_L$ be a sequence of times at which kernel $\Phi^j$ produced a spike. We want to establish the above corollary by induction on this sequence. Let us denote the interspike intervals by $\delta_i$, i.e. $\delta_i = t_i - t_{i-1}$ and the corresponding differences in the spiking thresholds by $\Delta_i$, i.e. $\Delta_i= T^j(t_i)- T^j(t_{i-1}) =C^j(t_i)- C^j(t_{i-1}) , \forall i \in \{1,..., L\}$. Clearly, $\sum_{i=1}^{k} \Delta_i \geq 0,\forall k \in \{1,..., L\}$, as any deviation from this condition would result in the convolution value falling below the baseline threshold $C$, consequently rendering the system incapable of producing a spike. Also,  since the convolution $C^j(t)$ is a continuous function with $C^j(0)=0$ (by definition)  convolution value for the first spike equals the baseline threshold, i.e. $C^j(t_0)=C >0$. Then, based on equation \ref{convIneq}, $|C^j(t_k)- C^j(t_0)|= |\sum_{i=1}^{k} \Delta_i| <  b\sqrt{\tau},\forall k \in \{1,..., L\}$. The threshold equation \ref{thresholdeq} formulates that for each spike at time $t_i$, due to the \textit{ahp} effect  the threshold value is immediately incremented by $M$ and then within time $\delta$ the \textit{ahp} effect linearly drops to zero. Let $A^{i}_{k}$ denote the drop in the \textit{ahp} due to the spike at time $t_i$ in the interval $\delta_k$. Note that $A^{i}_{k} = 0$ whenever $i \geq k$ and $\Sigma_{k}A^{i}_{k} =M$, i.e. the total \textit{ahp} drop for each spike across all intervals is $M$. The proof follows by establishing the following two invariant by induction on $k$. 
    \vspace{-2pt}
    \[
    \begin{aligned}
        A^{k-1}_{k} &= M - \Sigma_{i=1}^{k} \Delta_i, \\
        \delta_{k} &< \frac{\delta}{2} \quad \forall k \in \{1,..., L\}        
    \end{aligned}
    \]
    The base case for $k=1$, is clearly true. Because, $C^j(t_0) = T^j(t_0) = C$ and  $C^j(t_1) = T^j(t_1) = C + M(1- \frac{\delta_1}{\delta}I_{(\delta_1 < \delta)})$ (by equation \ref{thresholdeq}). Here $I_{(\delta_1 < \delta)}=1$ if $\delta_1 < \delta$ and $0$ otherwise. When $\delta_1 < \delta$, we get:
    \begin{align*}
    C^j(t_1) &= C + M(1- \frac{\delta_1}{\delta})\\
    &= C^j(t_0) + \Delta_1 \\
    &\Rightarrow \delta_1 = \frac{M-\Delta_1}{M}\delta >  \frac{M-b\sqrt{\tau}}{M}\delta > \frac{\delta}{2}.
    \end{align*}
    And the \textit{ahp} drop due to spike at $t_0$ in the interval $\delta_1$ is: $A^{0}_{1} = M-(C^j(t_1) - C^j(t_0)) = M- \Delta_1$.\\ For the other case when  $\delta_1 \geq \delta$, trivially $\delta \geq \frac{\delta}{2}$. Also in that case $C^j(t_1) = C$ since there is no spike in the previous $\delta$ interval of $t_1$ and therefore $\Delta_1 = 0$.  But the \textit{ahp} drop of the spike at $t_0$ in interval $\delta_1$ is:  $A^{0}_{1} = M \text{ (since total drop is $M$)}= M- \Delta_1$. This establishes the invariants for the base case, $k =1$.\\
    
    Now for the induction step we assume that the invariants hold for all $k \in \{1, \ldots, n\}$, for some $n < L$, and we show that the invariants are true for $k = n+1$. Specifically, we assume that $\delta_k > \frac{\delta}{2}$ and $A^{k-1}_{k}= M - \Sigma_{i=1}^{k} \Delta_i$, for all $ k \leq n$. By assumption, $C^j(t_{n+1}) = C^j(t_{n}) + \Delta_{n+1} = T^j(t_{n+1})$. But $T^j(t_{n+1})$, the threshold at time $t_{n+1}$, is the sum of $T^j(t_{n})$ and the \textit{ahp} effect of the spike at time $t_{n}$ and the \textit{ahp} drop due to the spike at time $t_{n-1}$. Note that for $T^j(t_{n+1})$ we don't need to consider the \textit{ahp} effects of spikes prior to $t_{n-1}$ since the spikes prior to $t_{n-1}$ are outside the $\delta$ interval of $t_{n+1}$ by induction assumption and therefore have zero \textit{ahp} effect. Mathematically,
    \[
    \begin{aligned} 
    &T^j(t_{n+1}) = T^j(t_{n}) + M\left(1 - \frac{\delta_{n+1}}{\delta}\mathbf{I}_{(\delta_{n+1} < \delta)}\right) - A^{n-1}_{n+1}  \\
    & \Rightarrow C^j(t_{n+1})= C^j(t_{n}) + M\left(1 - \frac{\delta_{n+1}}{\delta}\mathbf{I}_{(\delta_{n+1} < \delta)}\right) - A^{n-1}_{n+1}  \\
    & \hspace{2cm} = C^j(t_{n}) + \Delta_{n+1}\\
    &\Rightarrow \Delta_{n+1} =  M\left(1 - \frac{\delta_{n+1}}{\delta}\mathbf{I}_{(\delta_{n+1} < \delta)}\right) - A^{n-1}_{n+1}. 
    \end{aligned}
\]
Therefore for the case when $\delta_{n+1} < \delta$ we get: 
\[
\begin{aligned}
& \Delta_{n+1} =  M\left(1 - \frac{\delta_{n+1}}{\delta}\right) - A^{n-1}_{n+1}\\
 & \Rightarrow  M\frac{\delta_{n+1}}{\delta} = M - \Delta_{n+1} -  A^{n-1}_{n+1} \\ 
& \geq M - \Delta_{n+1} -\sum_{i=1}^{n}\Delta_i  \quad (\text{since } A^{n-1}_{n+1}+ A^{n-1}_{n} \leq M). \\       
& \Rightarrow  M\frac{\delta_{n+1}}{\delta} \geq M - \sum_{i=1}^{n+1} \Delta_i > M - b\sqrt{\tau} \Rightarrow \delta_{n+1} > \frac{\delta}{2}.
\end{aligned}
    \] 
 Since $\delta_{n}+ \delta_{n+1} > \delta$, the total drop in \textit{ahp} during $\delta_n$ and $\delta_{n+1}$ due to spike at $t_{n-1}$ is $M$, i.e. $A^{n-1}_{n+1}+ A^{n-1}_{n} = M$. But, 
    \[
    \begin{aligned}
    & \hspace{2 cm} A^{n-1}_{n+1} + A^{n}_{n+1} = M - \Delta_{n+1} \\
    & \Rightarrow  A^{n}_{n+1} = M - \Delta_{n+1} - (M - A^{n-1}_{n}) = M - \sum_{i=1}^{n+1} \Delta_i,            
    \end{aligned}
    \]
    establishing the invariant. For the case $\delta_{n+1} \geq \delta$, trivially we get $\delta_{n+1} > \frac{\delta}{2}$. Since the \textit{ahp} effect drops to zero for every spike within time $\delta$ , the threshold at time $t_{n+1}$ must come down to the baseline value $C$, i.e. $T^j(t_{n+1})= C=C^j(t_{n+1}) \Rightarrow \sum_{i=1}^{n+1} \Delta_i = 0$. Also, since $\delta_{n+1} \geq \delta$, the \textit{ahp} drop due to the spike at $t_n$ in interval $\delta_{n+1}$ is $M$, i.e. $A^{n}_{n+1} = M = M -  \sum_{i=1}^{n+1} \Delta_i$, establishing the invariant for this case and hence, completing the proof. 
\section{Detailed Proof of the Perfect Reconstruction Theorem}
\label{perfectReconsSupp}
The following section provides a detailed proof of the Perfect Reconstruction Theorem as presented in the main text. While the main text includes the theorem and its basic proof, this supplementary section offers a full proof of Lemma 2 and includes a formal claim that is referenced in the main text but not explicitly stated there. These additions provide a more comprehensive understanding and are included here to complement the content of the main text.
\\
The theorem assumes that the input signal belongs to a more restrictive  class of signals, $\mathcal{G}$, which is a subset of the class $\mathcal{F}$. The class of input signals was initially modeled by $\mathcal{F}$  in the Coding Section of the main text. The class $\mathcal{G}$ is defined as:\\ 
$\label{restricted class}
\mathcal{G} = \{X| X \in \mathcal{F}, X= \sum_{p=1}^{N} \alpha_p \Phi^{j_p}(t_p-t), j_p \in \{1,...,m\}, \alpha_p \in {\mathbb R}, t_p \in {\mathbb R}^{+}, N \in Z^{+}\}
$. The theorem is restated below exactly as it appears in the main text, followed by its proof using two lemmas, as described in the original text.\\
\textbf{Theorem \ref{PerfectReconsThm}. (Perfect Reconstruction Theorem)}
Let $X \in \mathcal{G}$ be an input signal. Then for appropriately chosen time-varying thresholds of the kernels, the reconstruction, $X^{*}$, resulting from the proposed coding-decoding framework is accurate with respect to the $L2$ metric, i.e., $||X^{*}-X||_2 = 0$.\\ 
\textbf{Lemma \ref{reconsLemma}.}
The solution $X^*$ to the reconstruction problem Eq. (\ref{optimizationproblem}) can be written as:
$
\label{reconstructionequation}
X^* = \sum_{i=1}^{N} \alpha_i \Phi^{j_i}(t_i-t)
$
where the coefficients $\alpha_i \in {\mathbb R}$ can be uniquely solved from a system of linear equations if the set of spikes $\{\phi_i = \Phi^{j_i}(t_i-t)\}_{i=1}^{N}$ produced is \emph{linearly independent}.\\
\textbf{Proof:}
An argument similar to that of the Representer Theorem \citep{scholkopf2001} on (\ref{optimizationproblem}) directly results in:
$
\label{eq:reconsSignalEq}
X^* = \sum_{i=1}^{N} \alpha_i \Phi^{j_i}(t_i-t)
$
where the $\alpha_{i}$'s are real valued coefficients. This holds true because any component of $X^*$  orthogonal to the span of the ${\Phi^{j_i}(t_i-t)}$'s does not contribute to the convolution (inner product) constraints. In essence, $X^*$ is an orthogonal projection of $X$ on the span of the spikes $\{\phi_i = \Phi^{j_i}(t_i-t)|i \in \{1,2,...,N \}\}$. Therefore, the coefficients can be derived by solving the linear system:
$
\label{alphaeq}
P\alpha=T
$
where $P$ is the $N\times N$ Gram matrix of the spikes, i.e., $[P]_{ik} = \langle \Phi^{j_i}(t_i-t),  \Phi^{j_k}(t_k-t) \rangle$, and  $T = \langle T^{j_1}(t_1), \ldots, T^{j_N}(t_N)\rangle^T$. Furthermore, the system has a unique solution if the Gram Matrix $P$ is invertible. And the Gram Matrix $P$ would be invertible if the set of spikes $\{\phi_i\}_{i=1}^{N}$ is linearly independent, which in turn follows from the assumption \ref{assumption1}. We claim that even when the $P$-matrix is non-invertible, a unique reconstruction $X^*$ can still be obtained following Eq. \eqref{optimizationproblem} in the main text, which can be calculated using the pseudo-inverse of $P$. $\hfill\Box$\\
Next, we prove the claim we just made about the uniqueness of existence of $X^*$ as per Eq. \eqref{optimizationproblem}.  before proceeding to the next lemma and the subsequent proof of the theorem.\\
\textbf{Claim:} Let $X$ be an input signal to our framework, generating a set of $N$ spikes, $\{\phi_i = \Phi^{j_i}(t_i-t)\}_{i=1}^{N}$. Let $X_1$ and $X_2$ be two possible reconstructions of $X$ from these $N$ spikes, obtained by solving the optimization problem in Eq. \eqref{optimizationproblem} of the main text. Then $X_1 = X_2$. \\
\textbf{Proof:} The uniqueness of the reconstruction of $X$, as formulated in Eq. \eqref{optimizationproblem} follows from the fact that the reconstruction is essentially the projection of $X$ onto the span of the spikes $\{\phi_i = \Phi^{j_i}(t_i-t)\}_{i=1}^{N}$. We now provide a formal proof.
Let $S$ be the subspace of $L^2$-functions spanned by $\{\phi_i = \Phi^{j_i}(t_i-t)\}_{i=1}^{N}$ with the standard inner product. Since each $\Phi^{j_i}(t_i-t)$ is assumed to be in $L^2$, $S$ is a subspace of the larger space of all $L^2$-functions. Clearly, $S$ is a Hilbert space with $dim(S) \le N$. Therefore, there exists an orthonormal basis $\{e_1,..., e_M\}$ for $S$, where $ M\le N$. Assume for contradiction that $X_1 \neq X_2$. Then there exist coefficients
$\{a_i\}$ and $\{b_i\}$ such that $X_{1} = \sum_{i=1}^{M} a_i e_i $ and $X_{2} = \sum_{i=1}^{M} b_i e_i $, where not all $a_i$ are equal to the corresponding $b_i$. Hence, there exists some $k$ such that $a_k \neq b_k$, which implies:
\vspace{-10pt}
\\
$$\langle X_1, e_k \rangle = a_k \neq b_k =
\langle X_2, e_k \rangle $$
\vspace{-10pt}
\\
However, since $e_k$ is in the span of $\{\phi_i = \Phi^{j_i}(t_i-t)\}_{i=1}^{N}$, there exists
$ \{c_1,...,c_N\}$ such that $e_k= \sum_{i=1}^{N} c_{i} \Phi^{j_i}(t_i-t)$. Therefore: 
\vspace{-10pt}\\
$$ \langle X_1, e_k \rangle = \sum_{i=1}^{N} c_i\langle X_1, \Phi^{j_i}(t_i-t)  \rangle = \sum_{i=1}^{N} c_i T^{j_i}(t_i)$$ 
Since both $X_1 \text{ and } X_2$ are solutions to the optimization problem in Eq. \eqref{optimizationproblem}, it follows that:
$$\langle X_1, e_k \rangle= \sum_{i=1}^{N} c_i T^{j_i}(t_i)  = 
\langle X_2, e_k \rangle$$
This leads to a contradiction to the assumption that $a_k \neq b_k$. Thus, $X_1 = X_2$, completing the proof.  $\hfill\Box$
\\
\textbf{Lemma \ref{th2}.}
Let $X^*$ be the reconstruction of an input signal $X$ by our framework with $\{\phi_i = \Phi^{j_i}(t_i-t)\}_{i=1}^{N}$ being the set of generated spikes. Then, for any arbitrary signal $\Tilde{X}$ within the span of the spikes given by $\Tilde{X} = \sum_{i=1}^{N} a_i \phi_i ,\; a_i \in \mathbb{R}$, the following holds: $||X- X^*|| \leq ||X-\tilde{X}||$.\\    
\textbf{Proof:} 
\allowdisplaybreaks
\begin{align}
&||X(t)-\Tilde{X}(t)|| = ||\underbrace{X(t)-X^{*}(t)}_\text{A}+\underbrace{X^{*}(t)-\Tilde{X}(t)}_\text{B}|| \nonumber\\
& \langle A,\Phi^{j_i}(t_i-t) \rangle = \langle X(t), \Phi^{j_i}(t_i-t) \rangle \nonumber\\
& \hspace{80pt}- \langle X^{*}(t), \Phi^{j_i}(t_i-t) \rangle, \;\;\;\forall i \in \{1,2,..,N\} \nonumber\\
& = T^{j_i}(t_i) -T^{j_i}(t_i) = 0  \text{ (by Eq \eqref{optimizationproblem} \& \eqref{spikeConstraint} of main text)} \nonumber\\
&\langle A,B \rangle = \langle A, \sum_{i=1}^{N} (\alpha_i - a_i) \phi_i \rangle\; \text{(By Lemma1 $X^*(t) = \sum_{i=1}^{N} \alpha_i \phi_i$)} \nonumber\\
&= \sum_{i=1}^{N} (\alpha_i - a_i)\langle A, \phi_i \rangle = 0 \Rightarrow A \perp B 
\nonumber\\
&\text{Therefore,} \nonumber\\
& ||X(t)-\Tilde{X}(t)||^{2} = ||A+B||^{2} = ||A||^{2} + ||B||^{2} \; \text{($A \perp B$)} \nonumber\\
& \hspace{70pt}\geq ||A||^{2} = ||X(t)-X^{*}(t)||^{2} \nonumber\\
&\Rightarrow ||X(t)-\Tilde{X}(t)||
\geq ||X(t)-X^{*}(t)|| \hspace{70pt}\Box\nonumber
\end{align} 
\textbf{Proof of the Theorem \ref{PerfectReconsThm}:}
The proof of the theorem follows directly from Lemma 2. Since the input signal $X \in \mathcal{G}$, let $X$ be given by:
$
    X= \sum_{p=1}^{N} \alpha_p \Phi^{j_p}(t_p-t) \hspace{10pt}(\alpha_p \in {\mathbb R}, t_p \in {\mathbb R}^{+}, N \in Z^{+})
$.
Assume that the time varying thresholds of the kernels in our kernel ensemble $\Phi$ are set in such a manner that the following conditions are satisfied:
$
    \langle X, \Phi^{j_p}(t_p-t) \rangle = T^{j_p}(t_p) \hspace{10pt} \forall{p \in \{1,...,N\}}
$
i.e., each of the kernels $\Phi^{j_p}$ at the very least produces a spike at time $t_p$ against $X$ (regardless of other spikes at other times). Clearly then $X$ lies in the span of the set of spikes generated by the framework. Applying Lemma 2 it follows that:
$
    ||X- X^*||_2 \leq ||X-X||_2 = 0 
$.$\hfill\Box$\\
\section{Proof of Approximate Reconstruction Theorem}
\label{approxReconsThmSupp}
\textbf{Theorem \ref{approxReconsThm}.} (Approximate Reconstruction Theorem)\\
Let the input signal $X$ be represented as $X = \sum_{i=1}^{N} \alpha_i f^{p_i}(t_i-t)$, where $\alpha_i \in \mathbb{R}$ and $f^{p_i}(t)$ are bounded functions on finite support that constitute the input signal. Assume that there is at least one kernel function $\Phi^{j_i}$ in the ensemble for which $||f^{p_i}(t) - \Phi^{j_i}(t)||_{2} < \delta$ for all $i \in \{ 1,...,N\}$. Additionally, assume that each of these kernels $\Phi^{j_i}$ produces a spike within a $\gamma$ interval of $t_i$,
for some $\delta \text{ and } \gamma \in \mathbb{R^+}$,  for all $i$. Also, assume that the functions $f^{p_i}$ satisfy a frame bound type of condition: $
\sum_{k \neq i} \langle f_{p_i}(t-t_i),f_{p_k}(t-t_k) \rangle \; \leq  \eta \; \forall \, i \in \{1,...,N\},
$ and that the kernel functions are Lipschitz continuous. Under such conditions, the $L^2$ error in the reconstruction $X^{*}$ of the input $X$ has bounded SNR. Specifically one can show that $\frac{\|X(t) - X^*(t)\|^2}{\|X(t)\|^2} \leq \frac{(\delta + C\gamma)^2 (x_{\text{max}} + 1)}{1 - \eta},$ where $\eta <1$, $C$ is a Lipschitz constant, and $x_{max} \in [0,N-1]$ is a constant depending on the overlap of the components in the input representation.
 \vspace{5pt}\\
\textbf{Proof:} By hypothesis each kernel $\Phi^{j_i}$ produces a spike at time $t_{i}'\; \forall i \in \{1,...,N\}$ . Let us call these spikes as \textit{fitting spikes}. But the coding model might generate some other spikes against $X$ too. Other than the
set of \textit{fitting spikes} $\{(t_{i}', \Phi^{j_i})|i \in \{1,...,N\}\}$, let
$\{(\Tilde{t}_k, \Phi^{\Tilde{j}_k})|k \in \{1,...,M\}\}$ denote those extra set of spikes that the coding model produces for input $X$ against the bag of kernels $\Phi$ and call these extra spikes as \textit{spurious spikes}. Here,  $M$ is the number of spurious spikes.  
By Lemma1, the reconstruction of $X$, denoted $X^{*}$, can be represented as below:
\vspace{3pt}
\\
$
 X^{*} = 
\sum_{i=1}^{N} \alpha_{i} \Phi^{j_i}(t_i'-t) +\sum_{k=1}^{M} \Tilde{\alpha_k} \Phi^{\Tilde{j_k}}(\Tilde{t_{k}}-t)
$\vspace{-10pt}\\

where $\alpha_i$ and $\Tilde{\alpha_k}$ are real coefficients whose values can be formulated again from Lemma1.
Let $T_{i}$ be the thresholds at which kernel  $\Phi^{j_i}$ produced the spike at time $t_{i}'$ as given in the hypothesis.
Hence for generation of the \textit{fitting spikes} the following condition must be satisfied:
\vspace{-5pt}
\begin{align}
\langle X, \Phi^{j_i}(t_{i}'-t) \rangle = T_{i} \;\; \forall i \in \{1,2,...,N\}
\end{align}
Consider a hypothetical signal $X_{hyp}$ defined by the equations below:
\vspace{-15pt}
\begin{align}
 X_{hyp} = \sum_{i=1}^{N} a_i \Phi^{j_i}(t_{i}'-t) , a_i \in R \nonumber\\
 \text{s.t. } \langle X_{hyp}, \Phi^{j_i}(t_{i}'-t) \rangle = T_i, \forall i 
\end{align}
\vspace{-15pt}\\
Clearly this hypothetical signal $X_{hyp}$ can be deemed as if it is the
reconstructed signal where we are only considering the  \textit{fitting spikes} and ignoring all
\textit{spurious spikes}. Since, $X_{hyp}$ lies in the span of the
shifted kernels used in reconstruction of
$X$ using Lemma 3 we may now write:\\
\vspace{-15pt}
\begin{align}
&||X-X_{hyp}|| \geq ||X-X^{*}|| \label{part1}\\
& ||X-X_{hyp}||_{2}^{2} = \langle X-X_{hyp}, X-X_{hyp} \rangle  \nonumber\\
&= \langle X-X_{hyp}, X \rangle -\langle X-X_{hyp}, X_{hyp}\rangle  \nonumber\\
& =\langle X-X_{hyp},X\rangle -\Sigma_{i=1}^{N} a_{i} \langle X-X_{hyp},\Phi^{j_i}(t -t_i')\rangle  \nonumber\\ 
&= ||X||_{2}^{2} - \langle X, X_{hyp}\rangle  \nonumber\\
&( \text{Since by construction}  \langle X_{hyp},\Phi^{j_i}(t -t_i')\rangle = T_{i} \;\forall i\in\{1...N\}) \nonumber\\
&= \Sigma_{i=1}^{N}\Sigma_{k=1}^{N} \alpha_{i} \alpha_{k}\langle f_{i}(t-t_{i}),f_{k}(t-t_{k})\rangle \nonumber\\
&\hspace{40pt}- \Sigma_{i=1}^{N}\Sigma_{k=1}^{N} \alpha_{i} a_{k} \langle f_{i}(t-t_{i}),\Phi^{j_k}(t -t_k'))\rangle \nonumber\\
&= \alpha^{T}F\alpha - \alpha^{T}F_{K}a 
\label{error}
\\
&(\text{denoting } a= [a_1, a_2, ..., a_N]^T
, 
 \alpha = [\alpha_{1},\alpha_{2},...,\alpha_{N}]^T,\nonumber\\
&F = [F_{ik}]_{N\times N} \text{, an $N\times N$ matrix, }  \text{where } F_{ik} =  \nonumber\\
& \langle  f_{i}(t-t_{i}),f_{k}(t-t_{k}) \rangle \text{ and } F_{K} = [(F_{K})_{ik}]_{NXN} \text{ where } \nonumber\\ 
& (F_{K})_{ik}  = \langle f_{i}(t-t_{i}), \Phi^{j_k}(t-t_{k}')\rangle) \nonumber\\
&  \text{ But using Lemma1 $a$ can be written as: } \nonumber\\
&a = P^{-1}T \text{, } P = [P_{ik}]_{NXN}, P_{ik} = \langle \Phi^{j_i}(t-t_{i}'),\Phi^{j_k}(t-t_{k}')\rangle \nonumber\\
& \text{And, } T = [T_{i}]_{N\times 1} \text{ where } T_{i} = \langle X(t), \Phi^{j_i}(t-t_{i}') \rangle \nonumber\\
& = \Sigma_{k=1}^{N}\alpha_{k} \langle f_{k}(t-t_k), \Phi^{j_i}(t-t_{i}') \rangle = F_{K}^{T}\alpha\nonumber\\
&\Rightarrow a = P^{-1}F_{K}^{T}\alpha \nonumber\\
&\text{Plugging this expression of $a$ in equations \eqref{error} we get,} \nonumber\\
&||X-X_{hyp}||_{2}^{2} = \alpha^{T}F\alpha - \alpha^{T}F_{K}P^{-1}F_{K}^{T}\alpha \label{expandedError}\\\nonumber
& \text{But,} (F_{K})_{ik} = \langle f_{i}(t-t_i), \Phi^{j_k}(t-t_k') \rangle \nonumber\\
&= \langle \Phi^{j_i}(t-t_i'), \Phi^{j_k}(t-t_k') \rangle \nonumber\\
& \hspace{60pt}- \langle \Phi^{j_i}(t-t_i')-f_i(t-t_i), \Phi^{j_k}(t-t_k') \rangle \nonumber\\
&= (P)_{ik} - (\mathcal{E}_{K})_{ik} \label{FKform}\\
&(\text{denoting } \mathcal{E}_{K} = [(\mathcal{E}_{K})_{ik}]_{N\times N}, \nonumber\\
& \text{where } (\mathcal{E}_{K})_{ik} = \langle \Phi^{j_i}(t-t_i')-f_i(t-t_i), \Phi^{j_k}(t-t_k') \rangle) \nonumber\\
&\text{Also, }(F)_{ik} = \langle f_{i}(t-t_{i}), f_{k}(t-t_{k}) \rangle \nonumber\\
  & = \langle f_{i}(t-t_{i})-\Phi^{j_i}(t-t_i')+\Phi^{j_i}(t-t_i'), \nonumber\\ 
  &\hspace{50pt}f_{k}(t-t_{k})-\Phi^{j_k}(t-t_k')+\Phi^{j_k}(t-t_k') \rangle \nonumber\\
&=  (\mathcal{E})_{ik} - (\mathcal{E}_{K})_{ik} - (\mathcal{E}_{K})_{ki} + (P)_{ik} \label{Fform}\\
& \text{Combining \eqref{expandedError}, \eqref{FKform} and \eqref{Fform} we get,} \nonumber\\
&||X-X_{hyp}||_{2}^{2} = \alpha^{T}F\alpha - \alpha^{T}F_{K}P^{-1}F_{K}^{T}\alpha \nonumber\\
&= \alpha^{T} \mathcal{E} \alpha - \alpha^{T}\mathcal{E}_{K}\alpha - \alpha^{T}\mathcal{E}_{K}^{T}\alpha + \alpha^{T}P\alpha \nonumber\\
&- \alpha^{T}P\alpha + \alpha^{T}\mathcal{E}_{K}\alpha + \alpha^{T}\mathcal{E}_{K}^{T}\alpha - \alpha^{T}\mathcal{E}_{K}P{-1}\mathcal{E}_{K}^{T}\alpha \nonumber\\ 
&= \alpha^{T} \mathcal{E} \alpha - \alpha^{T}\mathcal{E}_{K}P^{-1}\mathcal{E}_{K}^{T}\alpha \leq \alpha^{T} \mathcal{E} \alpha \nonumber\\
&\text{(Since, P is an SPD matrix, $\alpha^{T}\mathcal{E}_{K}P^{-1}\mathcal{E}_{K}^{T}\alpha > 0$)} \label{compactError}
\end{align}
We seek for a bound for the above expression. For that we observe the following:
\begin{align}
&|(\mathcal{E})_{ik}| = |\langle f_{i}(t-t_i)-\Phi^{j_i}(t-t_{i}^{'}),f_{k}(t-t_k)-\Phi^{j_k}(t-t_{k}^{'})\rangle| \nonumber\\
&=||f_{i}(t-t_i)-\Phi^{j_i}(t-t_{i}^{'})||_{2} ||f_{k}(t-t_k)-\Phi^{j_k}(t-t_{k}^{'})||_{2}.x_{ik} \nonumber\\
& \text{(where $x_{ik} \in [0,1]$. We also note that $x_{ik}$ is close to $0$ when}\nonumber\\
&\text{ the overlaps in the supports of the two components and their} \nonumber\\ 
& \text{corresponding fitting kernels are minimal.)} \nonumber\\
&\Rightarrow (\mathcal{E})_{ik} \leq(||(f_{i}(t-t_i)-\Phi^{j_i}(t-t_{i})||+\nonumber\\ 
&\hspace{25pt}||\Phi^{j_i}(t-t_{i})-\Phi^{j_i}(t-t_{i}^{'}))||). \nonumber\\ 
&\hspace{35pt}(||f_{k}(t-t_k)-\Phi^{j_k}(t-t_{k})|| \nonumber\\ 
&\hspace{45pt}+||\Phi^{j_k}(t-t_{k})-\Phi^{j_k}(t-t_{k}^{'})||).x_{ik}\nonumber\\
&\Rightarrow (\mathcal{E})_{ik} \leq x_{ik}.(\delta + C\gamma)^{2} \label{eVal}\\
&\text{(Assuming $C$ is a Lipschitz constant that each kernel $\Phi^{j_i}$} \nonumber\\
&\text{satisfies, and by assumption we have $|t_{i}- t_{i}^{'}| < \delta$ for all $i$.)}\nonumber\\
&\text{Now, using Gershgorin circle theorem, the maximum eigen} \nonumber\\
&\text{value of $\mathcal{E}$ can be obtained as follows:} \nonumber\\
&\Lambda_{max}(\mathcal{E}) \leq max_{i} ((\mathcal{E})_{ii} + \Sigma_{k \neq i}|(\mathcal{E})_{ik}|) \nonumber\\
&\hspace{38pt}\leq (\delta + C\gamma)^2 (x_{max}+1) 
\hspace{20pt}\text{(Using \eqref{eVal})} \label{Eeig}\\
&\text{(where $x_{max} \in [0,N-1]$ is a positive number that depends  } \nonumber\\
&\text{on the maximum overlap of the supports of the component} \nonumber\\
&\text{signals and their fitting kernels.)}\nonumber\\
&\text{Similarly, the minimum eigen value of $F$ is:}
\nonumber\\
&\Lambda_{min}(F) =  min_{i} ((F)_{ii} - \Sigma_{i \neq k} |\langle f_{p_i}(t-t_i), f_{p_k}(t-t_k)\rangle|) \nonumber\\
&\hspace{40pt}\geq 1- \eta \label{Feig}\\
& \text{(By assumption $\Sigma_{i \neq k} |<f_{p_i}(t-t_i), f_{p_k}(t-t_k)>|  \leq \eta$ )} \nonumber\\
&\text{Combining the results from \eqref{compactError}, \eqref{Eeig} and \eqref{Feig} we get:} \nonumber\\
&\nicefrac{||X(t)- X_{hyp}(t)||^{2}}{||X(t)||^{2}} \leq \alpha^{T} \mathcal{E} \alpha / \alpha^{T} F \alpha \nonumber\\
&\leq \Lambda_{max}(\mathcal{E})/\Lambda_{min}(F) \nonumber\\
&\leq  (\delta + C\gamma)^2 (x_{max}+1) /(1- \eta)\\
&\text{Finally using \eqref{part1} we conclude,} \nonumber\\
&\nicefrac{||X(t)- X^{*}(t)||^{2}}{||X(t)||^{2}} \leq
\nicefrac{||X(t)- X_{hyp}(t)||^{2}}{||X(t)||^{2}} \nonumber\\
&
\hspace{25pt}\leq(\delta + C\gamma)^2 (x_{max}+1) /(1- \eta) \nonumber\\
&\text{where $\eta <1$, and $x_{max} \in [0,N-1]$, a constant depending on} \nonumber\\
&\text{the overlap of the components in the input representation.$\hspace{8pt}\Box$}
\nonumber
\end{align}
\section{Analysis of Overlapping Kernels in Figure \ref{sineKernelsOverlap}}
\label{sineKernelSupp}
This section contains an analysis of the overlapping sine kernels scenario depicted in Figure \ref{sineKernelsOverlap}. The assumption \ref{assumption1} from the main text states that the norm of the projection of each spike onto the span of all previous spikes is bounded from above by some constant strictly less than 1. In other words, each incoming spike has a component orthogonal to the span of all previous spikes, the norm of which is bounded below by some constant strictly greater than 0. However, this assumption does not necessarily imply that each spike is poorly represented in the span of the set of all other spikes, including both past and future spikes. Figure \ref{sineKernelsOverlap} provides a counterexample to this by constructing a sequence of spikes where, despite each spike having a bounded orthogonal component with respect to all previous spikes, one can show that a particular spike in this sequence can be almost perfectly represented by all others. Specifically, the projection of one spike onto the span of the others approaches a norm of 1 as the sequence length increases. 
In the figure, each spike is generated by a kernel function that consists of a single cycle of a sine wave. The spikes are arranged so that tail of one spike perfectly overlaps with the head of the previous spike. In this setup, each spike overlaps only with  one previous spike and one subsequent spike. \\
\textbf{Claim:}
Let $\phi_1, ..., \phi_N$ be the sequence of spikes arranged in time as shown in Figure \ref{sineKernelsOverlap}, where each spike is generated by a normalized sine wave kernel. In this setup, the projection of any spike $\phi_n$ onto the span of all other spikes in the sequence satisfies 
$||\mathcal{P}_{\mathcal{S}(\bigcup_1^{N}\{\phi_i\} \setminus \{\phi_n\})}(\phi_n)|| \rightarrow 1$ for $n \rightarrow \infty \text{ and } N-n \rightarrow \infty$.\\
\textbf{Proof:}
 Each spike $\phi_i$ can be decomposed into two components: the positive half wave at the tail of the spike, denoted $\phi_{i}^{t}$, and the negative half-wave at the head of the spike, denoted $\phi_{i}^{h}$. By assumption, each spike is normalized, so the norm of each component is $\frac{1}{\sqrt{2}}$. Due to how the spikes are aligned, $\phi_{i}^h =-\phi_{i+1}^t$ for all $i \in [1, N-1]$, which means that $\langle \phi_i, \phi_j \rangle = -\frac{1}{2}$ for $|i-j| =1$ and 0 for $|i-j|>1$. 
 To show that a spike $\phi_n$ in this sequence is almost perfectly represented by the others, we calculate the projection of $\phi_n$ onto the span of all the other spikes. This projection can be expressed as the sum of two components: one projection onto the span of the preceding spikes and one onto the span of the succeeding spikes. The coefficients of these projections can be determined by solving a system of linear equations involving the Gram matrix formed by the inner products of the spikes similar to the approach used Lemma \ref{reconsLemma} from the main text. The projection can be written as follows:
 \vspace{-20pt}\\
 \begin{align}
     & \mathcal{P}_{\mathcal{S}(\bigcup_1^{N}\{\phi_i\} \setminus \{\phi_n\})}(\phi_n)
     = \mathcal{P}_{\mathcal{S}(\bigcup_1^{n-1}\{\phi_i\} )}(\phi_n) \nonumber\\
     &\hspace{4cm} + \mathcal{P}_{\mathcal{S}(\bigcup_{n+1}^{N}\{\phi_i\} )}(\phi_n)\nonumber\\
     & \text{(As $\mathcal{S}(\cup_{n+1}^{N}\{\phi_i\} ) \perp \mathcal{S}(\cup_1^{n-1}\{\phi_i\} )$ due to disjoint support)} \nonumber\\
     &\Rightarrow \mathcal{P}_{\mathcal{S}(\bigcup_1^{N}\{\phi_i\} \setminus \{\phi_n\})}(\phi_n) = \sum_{i =1}^{n-1} \alpha_i \phi_i +\sum_{i =n+1}^{N} \alpha_i \phi_i; \alpha_i \in \mathbb R  \nonumber\\
     \nonumber
 \end{align}
 \vspace{-35pt}
 \\
 The values of $\alpha_i \text{ for }  i \in [1,n-1]$, following Lemma \ref{reconsLemma} from the main text, can be obtained by solving a linear system of equations of the form $P_1\alpha = T$ where $P_1$ is an $(n-1)\times (n-1)$ matrix and $\alpha = [\alpha_{n-1}, \cdots,\alpha_1]^T$. Similarly, the values of $\alpha_i \text{ for } i \in [n+1, N]$ can be obtained by solving a linear system of equations of the form $P_2\alpha' = T$, where $P_2$ is an $(N-n)\times (N-n)$ matrix and $\alpha' = [\alpha_{n+1}, \cdots,\alpha_N]^T$. Specifically we have the following.
 \begin{align}
     & P_1\alpha =
     \begin{bmatrix}
1 & -\frac{1}{2} & 0 & \cdots & 0 \\
-\frac{1}{2} & 1 & -\frac{1}{2} & \cdots & 0 \\
0 & -\frac{1}{2} & 1 & \cdots & 0 \\
\vdots & \vdots & \vdots & \ddots & -\frac{1}{2} \\
0 & 0 & 0 & -\frac{1}{2} & 1
\end{bmatrix} \begin{bmatrix}
\alpha_{n-1}\\
\vdots \\
\alpha_1
\end{bmatrix}
= \begin{bmatrix}
-\frac{1}{2} \\
0 \\
0 \\
\vdots \\
0
\end{bmatrix}
\label{sineAlpha} \end{align}
In the above equation \eqref{sineAlpha}, the matrix $P_1$ is a symmetric tridiagonal matrix with unit diagonal elements and constant off-diagonal entries, the inverse of which can be calculated using \textit{Chebyshev polynomials of the second kind} \citep{da2001explicit}. Specifically, the elements of $P_1^{-1}$ are given by:
\begin{align}
    &\hspace{-2cm}(P_1^{-1})_{i,j} =  (-1)^{i+j+1} \cdot 2 \cdot \frac{U_{i-1}\left(-1\right) \cdot U_{(n-1)-j}\left(-1\right)}{U_{n-1}\left(-1\right)}, \; i \leq j \nonumber\\
    &\hspace{-2cm} \text{ where $U_k$ denotes order-k Chebyshev polynomial of second } \nonumber\\
    &\hspace{-2cm} \text{kind. This simplifies to:}\nonumber\\
    \Rightarrow 
    (P_1^{-1})_{i,j} & = 
    \begin{cases}
    2 \cdot \frac{i(n-j)}{n}  & \text{ for } i \leq j \\
    2 \cdot \frac{j(n-i)}{n} & \text{ for } i > j \;\;\text{(using symmetry)}
    \end{cases}\nonumber\\
& \hspace{-1.5cm} \text{Now, using this and Eq. \eqref{sineAlpha} we get:}\nonumber\\
     \begin{bmatrix}
\alpha_{n-1}\\
\vdots \\
\alpha_1
\end{bmatrix} & = P_1^{-1}  
\begin{bmatrix}
-\frac{1}{2} \\
0 \\
0 \\
\vdots \\
0
\end{bmatrix}
= 
\begin{bmatrix}
-\frac{1}{2}(P_1^{-1})_{1,1} \\
-\frac{1}{2}(P_1^{-1})_{2,1} \\
\vdots \\
-\frac{1}{2}(P_1^{-1})_{n-1,1}
\end{bmatrix}
= 
\begin{bmatrix}
-\frac{n-1}{n} \\
-\frac{n-2}{n} \\
\vdots \\
-\frac{1}{n}
\end{bmatrix}
\nonumber\\
\Rightarrow \sum_{i =1}^{n-1} \alpha_i \phi_i &= -\sum_{i =1}^{n-1} \frac{i}{n} \phi_i 
= -\sum_{i =1}^{n-1} \frac{i}{n} (\phi_i^h + \phi_i^t) \nonumber\\
& \hspace{-2cm} = -\sum_{i =2}^{n-1} \frac{i}{n} (\phi_i^h - \phi_{i-1}^h) - \frac{1}{n} \phi_1 \nonumber\\
&\hspace{-2cm}  = - \frac{n-1}{n} \phi_{n-1}^h + \frac{1}{n}\sum_{i =1}^{n-2} \phi_{n-2}^h - \frac{1}{n}\phi_1^t \nonumber\\
&\hspace{-2cm} \text{(since $\phi_{i}^h =-\phi_{i+1}^t$ for all $i \in [1, N-1]$ by construction)}\nonumber\\
&\hspace{-2cm} \text{Therefore,} \nonumber\\
& \hspace{-2cm} ||\sum_{i =1}^{n-1} \alpha_i \phi_i||^2 = (\frac{(n-1)^2}{n^2} + \frac{n-2}{n^2} + \frac{1}{n^2})||\phi_1^t||^2 = \frac{n-1}{n}\cdot \frac{1}{2} \nonumber\\
&\hspace{-2cm}\text{(Since $\phi_i^t$s are mutually orthogonal and}\nonumber\\
&\hspace{-2cm}\text{ $||\phi_{n-1}^h||= ||\phi_{n-2}^h||=\cdots = ||\phi_{1}^h|| = ||\phi_{1}^t|| = \frac{1}{\sqrt{2}}$)} \nonumber\\
& \hspace{-1.5cm} \Rightarrow  ||\sum_{i =1}^{n-1} \alpha_i \phi_i||^2 \rightarrow \frac{1}{2} \text{  as  } n \rightarrow \infty \label{proj1}\\
& \hspace{-1.5cm}\text{Likewise using symmetrical arguments we can show that:}\nonumber\\
&\hspace{-2cm} ||\sum_{i =n+1}^{N} \alpha_i \phi_i||^2 = \frac{N-n-1}{N-n}\cdot\frac{1}{2} \rightarrow \frac{1}{2} \text{ as } N -n \rightarrow \infty \label{proj2}\\
&\hspace{-2cm}\text{Therefore combining \eqref{proj1} and \eqref{proj2} we get:}  \nonumber\\
& \hspace{-2cm}||\mathcal{P}_{\mathcal{S}(\bigcup_1^{N}\{\phi_i\} \setminus \{\phi_n\})}(\phi_n)||^2 \nonumber\\
& \hspace{-2cm} = ||\sum_{i =1}^{n-1} \alpha_i \phi_i||^2 + ||\sum_{i =n+1}^{N} \alpha_i \phi_i||^2 \rightarrow 1 
\text{ for} n  , N-n \rightarrow \infty \;\;\Box\nonumber\\
\nonumber 
\end{align}
\section{Proof of Lemma \ref{windowCondLemma}}
\label{windowCondLemmaSupp}
\textbf{Lemma \ref{windowCondLemma}.}
    Let $S = \{\phi_i\}_{i=1}^{N}$ denote the set of spikes generated by our framework, satisfying Assumption \ref{assumption2}, i.e., $\forall n \in \{1,..., N\}$, $||\mathcal{P}_{\mathcal{S}(\bigcup_{i=1}^{N}\{\phi_i\} \setminus \{\phi_n\})}(\phi_n)|| \leq \beta$, where $\beta \in \mathbb{R}$ is a constant strictly less than 1. Consider a subset $V \subseteq S$ of a finite size $d$, $d < N$. Then, for every $v \in \mathcal{S}(V)$ with $||v|| = 1$,  $\exists \beta_d < 1$, such that $||\mathcal{P}_{\mathcal{S}(S \setminus V)}(v)|| \leq \beta_d$ where $\beta_d$ is a real constant that depends on $\beta$ and $d$. Specifically, we can show that $\beta_d^2 \leq (1+ \frac{1-\beta^2}{d^2 \beta^2})^{-1} < 1$. \\
    \textbf{Proof:} Let $ S \supseteq V = \{ \phi_{v_1}, ...\phi_{v_d}\}$ where $\phi_{v_1}, ..., \phi_{v_d}$ are $d$ distinct spikes from $S$. Also, $\forall i \in \{1,...,d\}$ let us assume that:
    $\phi_{v_i}^\parallel =  \mathcal{P}_{\mathcal{S}(S\setminus V)}(\phi_{v_i}) \text{ and } \Hat{\phi}_{v_i} =\phi_{v_i} -  \phi_{v_i}^\parallel $, 
    i.e. each $\phi_{v_i}$ is decomposed into two components,  $\phi_{v_i}^\parallel$ - the component  which is the projection of $\phi_{v_i}$ in the span of $S\setminus V$ and $\Hat{\phi}_{v_i}$ - the component of $\phi_{v_i}$ which is the orthogonal complement of $\phi_{v_i}^\parallel$.
    By assumption we have the following hold true:\\
    $ ||\phi_{v_i}^\parallel || \leq \beta $ and $||\Hat{\phi}_{v_i}||^2 \geq (1- \beta^2) $.\\ 
    Also for any $v \in V, ||v|| =1$ let us assume the following: 
    \vspace{-5pt}
     \[
     \begin{aligned}
    & v = \Sigma_{i=1}^{d} \alpha_i \phi_{v_i} = \underbrace{\Sigma_{i=1}^{d} \alpha_i \phi_{v_i}^\parallel}_{Y} + \underbrace{\Sigma_{i=1}^{d} \alpha_i \Hat{\phi}_{v_i}}_{Z} \\
    & \text{ s.t. } ||v||^2 = ||\Sigma_{i=1}^{d} \alpha_i \phi_{v_i}^\parallel ||^2 +  ||\Sigma_{i=1}^{d} \alpha_i \Hat{\phi}_{v_i}||^2 = 1 \\
    \end{aligned}
     \vspace{-5pt}
    \]
   Here by definition $Y =\Sigma_{i=1}^{d} \alpha_i \phi_{v_i}^\parallel = \Sigma_{i=1}^{d} \alpha_i \mathcal{P}_{\mathcal{S}(S\setminus V)}(\phi_{v_i}) =\mathcal{P}_{\mathcal{S}(S\setminus V)}(v)  $, i.e. $Y$ is the projection of $v$ in the span of $S\setminus V $ and $Z= \Sigma_{i=1}^{d} \alpha_i \Hat{\phi}_{v_i}$ is the orthogonal complement with respect $S\setminus V $. And hence the objective of the Lemma is to establish an upper bound on $||Y||$. 
   Assume that $|\alpha_m| = max(|\alpha_1|,..., |\alpha_d|) ,\text{ for some } m \in \{1, ..., d\} $.
 \vspace{-5pt}
  \begin{align}
   &||Y||^2 = ||\Sigma_{i=1}^{d} \alpha_i \mathcal{P}_{\mathcal{S}(S\setminus V)}(\phi_{v_i})||^2 \nonumber \\
   &\leq (\Sigma_{i=1}^{d} |\alpha_i| ||\mathcal{P}_{\mathcal{S}(S\setminus V)}(\phi_{v_i})|| )^2 \;\;\;\text{(Using Triangle inequality)} \nonumber\\     
   & \leq \beta^2 (\Sigma_{i=1}^{d}|\alpha_i|)^2 \leq d^2\beta^2\alpha_m^2 \;\; \label{yinequality}\\
   &\hspace{2cm} \text{(Since $\forall n, ||\mathcal{P}_{\mathcal{S}(\bigcup_1^{N}\{\phi_i\} \setminus \{\phi_n\})}(\phi_n)|| \leq \beta$)}\nonumber\\
  & \text{Again, }  ||Z||^2 = || \alpha_m \Hat{\phi}_{v_m}+  \Sigma_{i \in \{1,...,d\}\setminus\{m\}} \alpha_i \Hat{\phi}_{v_i}||^2 \label{Zexp}\\
  & \text{But, $\Hat{\phi}_{v_m}$ can be decomposed into two components:  $\Hat{\phi}_{v_m}^\parallel = $ } \nonumber\\
  &  \mathcal{P}_{\mathcal{S}(\cup_{i}\{\Hat{\phi}_{v_i}\} \setminus\{\Hat{\phi}_{v_m}\})}(\Hat{\phi}_{v_m}) \text{ and} \Hat{\phi}_{v_m}^\perp = \Hat{\phi}_{v_m} -\Hat{\phi}_{v_m}^\parallel , \text{the ortho-} \nonumber\\
  &\text{gonal complement of $\Hat{\phi}_{v_m}^\parallel$. But $\Hat{\phi}_{v_m}^\perp $ can be written as :}\nonumber\\
  &\Hat{\phi}_{v_m}^\perp = \Hat{\phi}_{v_m} - \Hat{\phi}_{v_m}^\parallel = \phi_{v_m} - \phi_{v_m}^\parallel-  \Hat{\phi}_{v_m}^\parallel \;\;\; \nonumber\\ 
  & = \phi_{v_m} -\mathcal{P}_{\mathcal{S}(S\setminus V)}(\phi_{v_m}) - \mathcal{P}_{\mathcal{S}(\cup_{i}\{\Hat{\phi}_{v_i}\} \setminus\{\Hat{\phi}_{v_m}\})}(\Hat{\phi}_{v_m}) \nonumber\\
  & = \phi_{v_m} - \mathcal{P}_{\mathcal{S}((S\setminus V)\cup (\cup_{i}\{\Hat{\phi}_{v_i}\} \setminus\{\Hat{\phi}_{v_m}\}))}(\phi_{v_m}) \nonumber\\
  &\hspace{2cm} \text{(since every $\Hat{\phi}_{v_i}$ is orthogonal to $S\setminus V$)}\nonumber\\ \nonumber
  & \shortintertext{But $\mathcal{S}((S\setminus V)\cup (\cup_{i}\{\Hat{\phi}_{v_i}\} \setminus\{\Hat{\phi}_{v_m}\})) \subseteq \mathcal{S}(S\setminus \{\phi_{v_m}\})$
  since every $\Hat{\phi}_{v_i} = \phi_{v_i} -\mathcal{P}_{\mathcal{S}(S\setminus V)}(\phi_{v_i}) $, except for $\Hat{\phi}_{v_m}$, can be written as a linear combination of spikes in $S\setminus \{\phi_{v_m}\}$.
  Hence, $||\Hat{\phi}_{v_m}^\perp ||^2 = 1- ||\mathcal{P}_{\mathcal{S}((S\setminus V)\cup (\cup_{i}\{\Hat{\phi}_{v_i}\} \setminus\{\Hat{\phi}_{v_m}\}))}(\phi_{v_m})||^2 \geq 
  1- ||\mathcal{P}_{\mathcal{S}(S\setminus \{\phi_{v_m}\})}(\phi_{v_m})||^2  \geq 1- \beta^2$ \;\;\;\; (by assumption)} \label{orthocompVm}\\ \nonumber
  & \text{Combining \eqref{Zexp} and \eqref{orthocompVm} we get:}\nonumber\\
  & ||Z||^2 = || \alpha_m \Hat{\phi}_{v_m}+  \Sigma_{i \in \{1,...,d\}\setminus\{m\}} \alpha_i \Hat{\phi}_{v_i}||^2 \nonumber\\ 
  & \geq \alpha_m^2 ||\Hat{\phi}_{v_m}^\perp ||^2  \geq \alpha_m^2(1-\beta^2) \label{zinequality} \\
  &\text{From  \eqref{yinequality} and \eqref{zinequality} we get: $\frac{||Y||^2}{d^2\beta^2} \leq \alpha_m^2  \leq \frac{||Z||^2}{1-\beta^2}$. But,}\nonumber\\
  &1 = ||Y||^2+ ||Z||^2 \leq ||Z||^2 + d^2\beta^2\alpha_m^2 \leq \frac{d^2\beta^2}{1-\beta^2}||Z||^2 + \nonumber\\
 & ||Z||^2 \Rightarrow ||Z||^2 \geq \frac{1-\beta^2}{1+ (d^2-1)\beta^2} \Rightarrow \beta_d^2 = ||\mathcal{P}_{\mathcal{S}(S \setminus V)}(v)||^2  \nonumber\\
 & = 1- ||Z||^2 \leq  1- \frac{1-\beta^2}{1+ (d^2-1)\beta^2} \Rightarrow \beta_d^2 \leq   
(1+ \frac{1-\beta^2}{d^2 \beta^2})^{-1} \nonumber\\\nonumber
&\text{ Here $\beta_d$ is a quantity strictly less than $1$ when $\beta < 1$ and }\nonumber\\
&\text{$d$ is a finite positive integer.} \hspace{13em} \Box \nonumber  \\ 
\nonumber
\end{align}%

\vspace{-10pt}
\section{Complete Proof of Windowing Theorem}
\label{windowingTheoremSupp}
The following section presents a detailed proof of the \textit{Windowing Theorem} as discussed in the main text. While the main text includes the theorem and its basic proof, this supplementary section provides a comprehensive proof of \textit{Claim \ref{orthovectorsclaim}} and \textit{Claim \ref{psiknormboundClaim}}, as well as a complete derivation for the final part of the theorem's proof. These additions are intended to enhance understanding and complement the content of the main text. The theorem is restated below exactly as it appears in the main text, followed by its proof with the help of Lemma \ref{windowLemma}, as referenced in the original text.\\\\
\textbf{Theorem \ref{windowthm}} (Windowing Theorem).
For an input signal $X$ with bounded $L_2$ norm, suppose our framework produces a set of $n+1$ successive spikes $S = \{\phi_1, ..., \phi_{n+1}\}$, sorted by their time of occurrence and satisfying Assumption \ref{assumption2}. The error in the iterative reconstruction of $X$ with respect to the last spike $\phi_{n+1}$ due to windowing, as formulated in Eq. \ref{windowingEq}, is bounded. Specifically, 
\vspace{-5pt}
\begin{align}
 & \forall \epsilon >0, \exists w_0 > 0 \text{ s.t. } ||\mathcal{P}_{\phi^{\perp}_{n+1,w}}(X)- \mathcal{P}_{\phi^{\perp}_{n+1}}(X)|| < \epsilon, \nonumber\\
& \forall w\geq w_0\text{ and } w \leq n \label{windowThmEqun}   
\\
& \hspace{-10pt}\text{ where  $w_0$ is independent of $n$ for arbitrarily large $n \in \mathbb N$.}\nonumber\\
\nonumber 
\end{align}
\textbf{Import and Proof Idea}:  The Theorem implies that the error from windowing can be made arbitrarily small by choosing a sufficiently large window size $w$, independent of $n$, when $n$ is arbitrarily large. At first glance, one might think that the condition in Equation \ref{windowThmEqun} is trivially satisfied by choosing $w = n$, i.e., a window inclusive of all spikes. However, the key aspect of the theorem is that $ w $ should be independent of $ n $ when $ n $ is arbitrarily large. This allows us to use the same window size regardless of the number of previous spikes, even for large signals producing many spikes, thereby maintaining the condition number of the overall solution as per Theorem \ref{conditionNumberThm}. Our proof demonstrates this by showing that the reconstruction error converges geometrically as a function of the window size, depending only on the spike rate which in turn depends on the \textit{ahp} parameters in Equation \ref{thresholdeq}, but not on $n$.
The proof hinges on a central lemma showing that the $L_2$ norm of the difference between $\phi_{n+1,w}^{\perp}$ and  $\phi_{n+1}^{\perp}$ decreases steadily as \( w \) increases, based on the assumptions stated in the theorem. This ensures that \( \phi_{n+1,w}^{\perp} \) becomes a good approximation of $\phi_{n+1}^{\perp}$. The proof of the theorem then follows by establishing a bound on $ \|\mathcal{P}_{\phi^{\perp}_{n+1,w}}(X) - \mathcal{P}_{\phi^{\perp}_{n+1}}(X)\| $ for a given choice of $w$ for any bounded input $X$.
The lemma is provided below:\\
\textbf{Lemma \ref{windowLemma}.}
Under the conditions of Theorem \ref{windowthm}, for any $\delta > 0$, there exists $w_0 \in \mathbb{N}^+$ such that $\|\phi^{\perp}_{n+1,w} - \phi^{\perp}_{n+1}\| < \delta\; \forall w \geq w_0, w \leq n$, where the choice of $w_0$ is independent of $n$ for arbitrarily large $n \in \mathbb{N}^+$. \\
\textbf{Proof Idea:} The proof leverages Corollary \ref{spikerateCorollary}, which states that the maximum number of spikes overlapping in time is bounded by a constant $d \in \mathbb{N}^+$, dependent on the \textit{ahp} parameters. Using this corollary, we partition the set of all spikes in time into a chain of subsets, where each subset overlaps only with its neighboring subsets and is disjoint from all others. Each subset contains at most $d$ spikes. The proof then demonstrates that the error in approximating $\phi^{\perp}_{n+1}$ due to windowing, i.e., $\|\phi^{\perp}_{n+1,w} - \phi^{\perp}_{n+1}\|$, decreases faster than a geometric sequence as more of these partitions are included within the window. This convergence is illustrated schematically in Figure \ref{windowedOverlap}.\\ 
\textbf{Proof of Lemma \ref{windowLemma}:}
    Let the set of spikes $S = \{\phi_1, ..., \phi_n\}$ be partitioned into a chain of subsets of spikes $v_1, ..., v_m$ in descending order of time, defined recursively. The first subset $v_1$ consists of all previous spikes overlapping with the support of $\phi_{n+1}$. Recursively, $v_{i+1}$ is the set of spikes overlapping with the support of any spike in $v_{i}$ for all $i\leq m$. This process continues until the first spike $\phi_1$ is included in the final subset $v_m, \text{ where } m\leq n$. An example of this partitioning is illustrated in Figure \ref{windowedOverlap}.
    The individual spikes in each partition are indexed as follows:\\
    \vspace{-20pt}
    \begin{align*}
        v_i = \{ \phi_{p_i}, ..., \phi_{p_{i-1}-1}\}, \forall i \leq m \\ 
        \text{ where } 1 = p_m < ...<p_1 < p_0 = n+1. 
    \end{align*}  
    That is, the $i^{th}$ partition $v_i$ consists of spikes indexed  from $p_i$ to $p_{i-1} -1$. Since the spikes $\phi_1, ... , \phi_n$ are  sorted in order of their occurrence time, if both $\phi_{p_i} \text{ and } \phi_{p_{i-1}-1}$ are in partition $v_i$, then by construction,  $\phi_j \in v_i$ for all $p_i \leq j \leq p_{i-1}-1$. \vspace{5pt}\\
   \textbf{Claim \ref{partitionClaim}.}
             The number of spikes in every partition, $v_i \; \forall i \in \{1, ..., m\}$, is bounded by some constant $d \in \mathbb{N}^+$.\\
    \textbf{Proof:} This corollary follows from the observation that all spikes in a given partition $v_{i+1}$ overlap in time with at least one spike from the preceding partition $v_i$, specifically the spike whose support extends furthest into the past, for all $ i \in \{2,..,m\}$ (see fig \ref{windowedOverlap}). In the case of the partition $v_1$, each spike overlaps with $\phi_{n+1}$. Then the proof  follows from the corollary \ref{spikerateCorollary}. $\hfill\Box$\\  
    Now, let  $V_1, ... , V_m$  be subspaces spanned by the subsets of spikes $v_1, ..., v_m $ respectively.
    Before proceeding with the rest of the proof, we introduce the following notations. 
    \begin{align*}
        \Tilde{\phi_i} = \phi_i - \mathcal{P}_{\mathcal{S} (\bigcup\limits_{j = i+1}^{n}\{ \phi_{j}\})}(\phi_i) \;\;\;
        \forall i \in \{1,..., n\}
    \end{align*}
    where $ \Tilde{\phi_i}$ denotes the orthogonal complement of the spike $\phi_i$ with respect to all the future spikes up to $\phi_n$. We also denote,  
    \[
        \Tilde{V_k} = \mathcal{S}(\bigcup\limits_{j=p_k}^{p_{k-1} -1}\{\Tilde{\phi}_j\}) \;\;\;
        \forall k \in \{1,...,m\}
    \]   
    where $\Tilde{V_k}$ is the subspace spanned by spikes in the partition $v_k$, with each spike orthogonalized with respect to all future spikes.\\
\textbf{Claim \ref{orthovectorsclaim}.}    
         For the subspaces defined on the partitions as above, the following holds:\\
            $\Tilde{V_k} = \{ x- \mathcal{P}_{\sum_{j =1}^{k-1} V_j} (x) | x \in V_k\} =  
            \mathcal{S}(\bigcup\limits_{j=p_k}^{p_{k-1} -1}\{\Tilde{\phi}_j\}). 
       $ \\
    \textbf{Proof:} Denote $A = \{ x- \mathcal{P}_{\sum_{j =1}^{k-1} V_j} (x) | x \in V_k\}$ and $B = \mathcal{S}(\bigcup\limits_{j=p_k}^{p_{k-1} -1}\{\Tilde{\phi}_j\})$. We need to show that $A=B$. Before showing that, we rewrite $ \Tilde{\phi_i}$ for $ p_k\leq i \leq p_{k-1} -1$ as follows:
    \begin{align}
       & \Tilde{\phi_i} = \phi_i - \mathcal{P}_{\mathcal{S} (\bigcup\limits_{j = i+1}^{n}\{ \phi_{j}\})}(\phi_i) \;\;\; ( p_k\leq i \leq p_{k-1} -1) \nonumber\\
       & = \phi_i - \mathcal{P}_{\mathcal{S} ((\bigcup\limits_{j = p_{k-1}}^{n}\{ \phi_{j}\})\bigcup  (\bigcup\limits_{j = i+1}^{p_{k-1}-1}\{ \Tilde {\phi_{j}}\}))}(\phi_i) \nonumber\\
       &= \phi_i - \mathcal{P}_{\sum_{j =1}^{k-1} V_j}(\phi_i) - \sum_{j =i+1}^{p_{k-1} -1}\mathcal{P}_{\Tilde{\phi_j}}(\phi_i) \label{tildeExp}\\
       &(\text{As } \sum_{j =1}^{k-1} V_j \text{ and all } \Tilde{\phi_i}  \text{'s are mutually orthogonal} \nonumber\\
       & \text{ for }   p_k\leq i \leq p_{k-1} -1) \nonumber\\
       \nonumber
    \end{align}
    \vspace{-30pt}\\
   Using the formulation in \ref{tildeExp}, we will show $A=B$. First, we show that $A \subseteq B$. Let there be any $y \in A$. Then by assumption, we can write $y$ as follows:
   \vspace{-20pt}\\
    \begin{align}
       & y = x- \mathcal{P}_{\sum_{j =1}^{k-1} V_j} (x) \text{, for some } x = \sum_{i=p_k}^{p_{k-1}-1} \alpha_i \phi_i, \;\; \alpha_i \in \mathbb R \nonumber\\
       & \Rightarrow y = \sum_{i=p_k}^{p_{k-1}-1} \alpha_i \phi_i - \mathcal{P}_{\sum_{j =1}^{k-1} V_j}(\sum_{i=p_k}^{p_{k-1}-1} \alpha_i \phi_i) \nonumber \\     
       & = \sum_{i=p_k}^{p_{k-1}-1} \alpha_i (\phi_i - \mathcal{P}_{\sum_{j =1}^{k-1} V_j}(\phi_i)) \;\;\text{(by linearity of projection)} \nonumber\\
       & \Rightarrow y = \sum_{i=p_k}^{p_{k-1}-1} \alpha_i (\Tilde{\phi_i} +  \sum_{j =i+1}^{p_{k-1} -1}\mathcal{P}_{\Tilde{\phi_j}}(\phi_i))
       \text{ (using \eqref{tildeExp})} \label{yExp} 
    \end{align}
    \vspace{-10pt}\\
        Equation \eqref{yExp} shows that $y$ is written as a linear combination of $\Tilde{\phi_i} \text{ for } p_k\leq i \leq p_{k-1} -1$, and hence $y \in B$. This proves that $A \subseteq B$. \\
        To prove $B \subseteq A$, we observe that any $z \in B$ can be written as a linear combination of $\Tilde{\phi_i} \text{ for } p_k\leq i \leq p_{k-1} -1$. So it suffices to show that $\Tilde{\phi_i} \in A \text{ for } p_k\leq i \leq p_{k-1} -1$. We do that via induction on the set $\{\Tilde{\phi}_{p_{k-1}-1}, ..., \Tilde{\phi}_{p_k}\}$, starting with $\Tilde{\phi}_{p_{k-1}-1}$ as the base case. The base case is trivially true, i.e., $\Tilde{\phi}_{p_{k-1}-1} \in A$, because by equation \ref{tildeExp}, we immediately obtain $\Tilde{\phi}_{p_{k-1}-1} = \phi_{p_{k-1}-1} - \mathcal{P}_{\sum_{j =1}^{k-1} V_j}({\phi}_{p_{k-1}-1})$. \\
        For the induction step, assume that $\Tilde{\phi_i} \in A$ for all $i \text{such that }  i \in [ n, p_{k-1}-1] \text{ for some} n \in \mathbb{N}^+ \text{ and } p_k < n \leq p_{k-1}-1$. We need to show that $\Tilde{\phi}_{n-1} \in A$. For that, we again use equation \eqref{tildeExp} to observe the following:
        \begin{align}
       & \Tilde{\phi}_{n-1} = (\phi_{n-1} - \mathcal{P}_{\sum_{j =1}^{k-1} V_j}(\phi_{n-1})) - \sum_{j =n}^{p_{k-1} -1}\mathcal{P}_{\Tilde{\phi_j}}(\phi_{n-1}) \label{phiInduction}
        \end{align}
Analyzing the expression for $\Tilde{\phi}_{n-1}$ on the right-hand side of equation \eqref{phiInduction}, we observe that the expression $ (\phi_{n-1} - \mathcal{P}_{\sum_{j =1}^{k-1} V_j}(\phi_{n-1})) \in A$ by construction, and the expression $\sum_{j =n}^{p_{k-1} -1}\mathcal{P}_{\Tilde{\phi_j}}(\phi_{n-1})$, a linear combination of $\Tilde{\phi_i}, \; n \leq i \leq p_{k-1} -1$, is also in $A$ because by induction hypothesis, each $\Tilde{\phi_i}, \text{ for } n \leq i \leq p_{k-1} -1$ is in $A$. Thus, we obtain $\Tilde{\phi}_{n-1} \in A$. This proves that $B \subseteq A$. Therefore, $A=B$.
    $\hfill\Box$\\
        Following the claim, we define the subspace $U_{k}$ as below:
        \begin{align}
            &U_{k} = \sum_{i=k}^{m} \Tilde{V}_i = \mathcal{S}(\bigcup\limits_{j=1}^{p_{k-1} -1} \{\Tilde{\phi}_{j}\})
            \label{lhs}\\
            & = \{ x- \mathcal{P}_{\sum_{j =1}^{k-1} V_j} (x) | x \in \sum_{i=k}^{m} V_i \} \label{rhs}
        \end{align}
        where the equality between \ref{lhs} and \ref{rhs} follows from the Claim \ref{orthovectorsclaim}. Lastly, define $\phi^\perp_{n+1, v_k}$ as:
        $$\phiperp{v_k} = \phi_{n+1} - \mathcal{P}_{\sum_{i =1}^{k} V_i}(\phi_{n+1})$$
        i.e. $\phiperp{v_k}$ is the orthogonal complement of $\phi_{n+1}$ with respect to window of spikes up to partition $v_k$. Now we proceed to quantify the norm of the difference of $\phicomp$ and $\phiperp{v_k}$.  For that denote $e_k$ as follows:
         \begin{align}
            & e_k = ||\phiperp{v_k} - \phicomp||  \nonumber\\
            & = ||\mathcal{P}_{\sum_{i =1}^{k} V_i} (\phi_{n+1}) -\mathcal{P}_{\sum_{i =1}^{m} V_i} (\phi_{n+1})|| \nonumber\\
            & = ||\mathcal{P}_{\sum_{i =1}^{k} V_i} (\phi_{n+1}) -\mathcal{P}_{\sum_{i =1}^{k} V_i +U_{k+1}} (\phi_{n+1})|| \;(\text{by def. of $U_{k}$})\nonumber\\
            & = ||\mathcal{P}_{\sum_{i =1}^{k} V_i} (\phi_{n+1}) -(\mathcal{P}_{\sum_{i =1}^{k} V_i } (\phi_{n+1}) + \mathcal{P}_{U_{k+1}} (\phi_{n+1}))|| \nonumber\\            
            & \hspace{2cm}(\text{Since by construction $U_{k+1} \perp \Sigma_{i=1}^{k}V_i$}) \nonumber\\
            & \Rightarrow e_k= ||\mathcal{P}_{U_{k+1}} (\phi_{n+1})) ||  \label{e_k_exp}\\
            & \text{Note that by definition $e_m =0$ and by Assumption \ref{assumption2} we get,} \nonumber\\
            & e_k = ||\mathcal{P}_{U_{k+1}} (\phi_{n+1})) || \leq |\beta| < 1, \;\forall k \in \{1,...,m\} \label{boundary} \\  
            &\text{Now Assume that:} \nonumber\\
            & \mathcal{P}_{U_{k+1}} (\phi_{n+1})) = \alpha_k \Tilde{\psi}_k, \text{where } \alpha_k \in \mathcal{R}, \Tilde{\psi}_k \in U_{k+1} \nonumber\\
            &\text{Then it follows from the definition of $U_{k+1}$ and Claim \ref{orthovectorsclaim}} \nonumber\\ 
            & \text{that $\Tilde{\psi}_k$ is of the form:}\nonumber\\
            & \Tilde{\psi}_k =  \psi_k - \mathcal{P}_{\sum_{j =1}^{k} V_i } (\psi_k) \text{ for some } \psi_k \in \sum_{i =k+1}^{m} V_i \nonumber\\ 
            &\text{W.L.O.G.  assume $||\psi_k|| =1$ by appropriately choosing $\alpha_k$.}\nonumber\\
            & \text{Since $\alpha_k \Tilde{\psi}_k$ is a projection of $\phi_{n+1}$ we obtain:}\nonumber\\
            & \alpha_k = \frac{\langle \phi_{n+1}, \Tilde{\psi}_k\rangle}{||\Tilde{\psi}_k||^2}  \label{coeff}\\
            & \text{For $k>0$ from \ref{coeff} and \ref{e_k_exp} we obtain,}\nonumber\\
            & e_k = ||\alpha_k \Tilde{\psi}_k|| = |\alpha_k|||\Tilde{\psi}_k|| = \frac{|\langle \phi_{n+1}, \Tilde{\psi}_k\rangle|}{||\Tilde{\psi}_k||}\nonumber\\
            & = \frac{|\langle \phi_{n+1}, \psi_k - \mathcal{P}_{\sum_{j =1}^{k} V_j } (\psi_k)\rangle|}{||\Tilde{\psi}_k||}\nonumber\\
            & = \frac{|\langle \phi_{n+1}, \mathcal{P}_{\sum_{j =1}^{k} V_j} (\psi_k)\rangle|}{||\Tilde{\psi}_k||} 
            =\frac{|\langle \phi_{n+1}, \mathcal{P}_{\sum_{j =1}^{k-1} V_j + \Tilde{V}_k} (\psi_k)\rangle|}{||\Tilde{\psi}_k||}
            \nonumber\\
            & \text{(for $k>0, \psi_k \in\sum_{i =k+1}^{m} V_i \perp \phi_{n+1}$ due to disjoint support)} \nonumber\\     
            &  
            =\frac{|\langle \phi_{n+1}, \mathcal{P}_{\sum_{j =1}^{k-1} V_j} (\psi_k)+ \mathcal{P}_{\Tilde{V}_k} (\psi_k)\rangle|}{||\Tilde{\psi}_k||} \;\text{(by def. $\Tilde{V}_k \perp \sum_{j =1}^{k-1} V_j$)}\nonumber\\
            & \Rightarrow e_k = \frac{|\langle \phi_{n+1}, \mathcal{P}_{\Tilde{V}_k} (\psi_k)\rangle|}{||\Tilde{\psi}_k||} \;
            \text{($k>0,\psi_k \in \sum_{i=k+1}^{m} V_i \perp \sum_{j=1}^{k-1} V_j$)}
            \label{finalE_k}\\ 
            & \text{Likewise, } e_{k-1} = ||\mathcal{P}_{U_{k}} (\phi_{n+1})) ||  = ||\mathcal{P}_{U_{k+1}+ \Tilde{V}_k} (\phi_{n+1})) ||\nonumber\\
            & \Rightarrow e_{k-1}^2 = ||\mathcal{P}_{U_{k+1}} (\phi_{n+1})) ||^2 + ||\mathcal{P}_{\Tilde{V}_k} (\phi_{n+1})) ||^2
            \nonumber\\
            & \hspace{2cm} \text{(Since by construction $U_{k+1} \perp \Tilde{V}_k$)} \nonumber\\
            & \Rightarrow e_{k-1}^2 = e_k^2 +  ||\mathcal{P}_{\Tilde{V}_k} (\phi_{n+1})) ||^2 \label{ek11}\\
            & \text{Assume that } \mathcal{P}_{\Tilde{V}_k} (\phi_{n+1}) = \beta_k \Tilde{\theta}_k, \text{ where }  \Tilde{\theta}_k \in \Tilde{V}_k, \beta_k \in \mathcal{R} \nonumber\\
            & \Rightarrow ||\mathcal{P}_{\Tilde{V}_k} (\phi_{n+1})) || = ||\beta_k \Tilde{\theta}_k|| = \frac{|\langle \phi_{n+1}, \Tilde{\theta}_k \rangle|}{||\Tilde{\theta}_k||} \label{ptheta}\\
            &\text{From \ref{ek11} \& \ref{ptheta} we get, } e_{k-1}^2 = e_k^2 + \frac{|\langle \phi_{n+1}, \Tilde{\theta}_k \rangle|^2}{||\Tilde{\theta}_k||^2} \label{e_k-1}\\
            &\text{Again, for $k >0$ we further analyze $e_k$ from \ref{finalE_k} to obtain:}\nonumber\\
            & e_k =  \frac{|\langle \phi_{n+1}, \mathcal{P}_{\Tilde{V}_k} (\psi_k)\rangle|}{||\Tilde{\psi}_k||} 
            \nonumber\\
            &= \frac{|\langle \phi_{n+1}, \mathcal{P}_{\mathcal{S}(\{\Tilde{\theta}_k\})} (\psi_k)
            + \mathcal{P}_{\Tilde{V_k}/{\mathcal{S}(\{\Tilde{\theta}_k\})}} (\psi_k)\rangle|}{||\Tilde{\psi}_k||}
            \nonumber\\ \nonumber
            \vspace{-20pt}
             &\text{The last line above is essentially written by breaking $\psi_k$ }\nonumber\\
             &\text{into two mutually orthogonal subspaces: $\mathcal{S}(\{\Tilde{\theta}_k\})$, subspace }\nonumber\\
             &\text{of $\Tilde{V_k}$ spanned by $\Tilde{\theta}_k$, and $\Tilde{V_k}/\mathcal{S}(\{\Tilde{\theta}_k\}$, the subspace of $\Tilde{V_k}$}\nonumber\\
             &\text{orthogonal to the subspace $\mathcal{S}(\{\Tilde{\theta}_k\})$. Also, observe that} \nonumber\\ 
             & \phi_{n+1} \perp \Tilde{V_k}/\mathcal{S}(\{\Tilde{\theta}_k\}) \text{ since } \mathcal{P}_{\Tilde{V}_k} (\phi_{n+1})) = \beta_k \Tilde{\theta}_k \text{. Therefore,}\nonumber\\   
               & e_k = \frac{|\langle \phi_{n+1}, \mathcal{P}_{\mathcal{S}(\{\Tilde{\theta}_k\})} (\psi_k)|}{||\Tilde{\psi}_k||}\;\; 
             = \frac{|\langle \phi_{n+1}, \frac{\langle \psi_k,\Tilde{\theta}_k\rangle \Tilde{\theta}_k}{||\Tilde{\theta}_k||^2}\rangle|}{||\Tilde{\psi}_k||}
             \nonumber\\
            & \Rightarrow e_k= \frac{|\langle \phi_{n+1}, \Tilde{\theta}_k\rangle|}{||\Tilde{\theta}_k||} \frac{|\langle \psi_k,\frac{\Tilde{\theta}_k}{||\Tilde{\theta}_k||}\rangle|}{||\Tilde{\psi}_k||}\label{e_k}\\
            &\text{Combining \ref{e_k-1} and \ref{e_k} we obtain:}\nonumber\\ 
            & e_{k-1}^2 = e_k^2 + e_k^2 \frac{||\Tilde{\psi}_k||^2}{|\langle \psi_k, \frac{\Tilde{\theta}_k}{||\Tilde{\theta}_k||}\rangle|^2}\label{sequenceRel}\\
            &\text{Since both $\psi_k$ and $\frac{\Tilde{\theta}_k}{||\Tilde{\theta}_k||}$ are unit norm $|\langle \psi_k, \frac{\Tilde{\theta}_k}{||\Tilde{\theta}_k||}\rangle| \leq1$.}\nonumber\\
            & \Rightarrow e_{k-1}^2 \geq e_k^2(1+ ||\Tilde{\psi}_k||^2)\label{geoDrop}\\\nonumber
    \end{align}
    \vspace{-30pt}\\
           Eq. \ref{geoDrop} demonstrates that the sequence $\{e_k\}$ converges geometrically
        to 0 when $(1+ ||\Tilde{\psi}_k||^2) > 1$, i.e. $||\Tilde{\psi}_k|| >0 $. To complete the proof of Lemma \ref{windowLemma}, we need to establish a positive lower bound on $||\Tilde{\psi}_k||$, ensuring the convergence of the sequence ${e_k}$. The following Corollary provides the necessary lower bound on $||\Tilde{\psi}_k||$.
        \\
        \textbf{Claim \ref{psiknormboundClaim}.}
            Following Claim \ref{partitionClaim}, if the number of spikes in each partition $v_i$ is bounded by $d$, then $|| \Tilde{\psi}_k||^2 \geq 1-  \beta_d^2$, where $\beta_d >0 $ is as defined in Lemma \ref{windowCondLemma}. 
        \\        
        \vspace{-10pt}
    \begin{align}
             \textbf{Proof:} &\hspace{1cm}\Tilde{\psi}_k = \psi_k - \mathcal{P}_{\Tilde{V}_k}(\psi_k),   ||\psi_k|| =1. 
             \nonumber\\&
            \text{Let,  } \mathcal{P}_{\Tilde{V}_k}(\psi_k) = \beta_k \Tilde{\eta}_k
             \text{ for some } \Tilde{\eta}_k \in \Tilde{V}_k, ||\Tilde{\eta}_k|| =1  \nonumber\\
             & \Rightarrow \beta_k = \langle \psi_k, \Tilde{\eta}_k\rangle \nonumber\\
             \nonumber
    \end{align}  
    \vspace{-60pt}\\
     \begin{align}
     & \Rightarrow \mathcal{P}_{\psi_k}(\Tilde{\eta}_k) = \langle \psi_k, \Tilde{\eta}_k\rangle \psi_k, \;\; (\text{Since, } ||\psi_k|| = 1 )\\
     &\text{But } \Tilde{\eta}_k \in \Tilde{V}_k \implies \Tilde{\eta}_k = \eta_k -  \mathcal{P}_{\sum_{i =1}^{k-1} V_i}( \eta_k)\nonumber\\
     &\text{for some $\eta_k \in V_k $. Also,}\nonumber\\
     & ||\eta_k||^2 = 1+ ||\mathcal{P}_{\sum_{i =1}^{k-1} V_i}( \eta_k)||^2 \label{etaNorm}\\
     &\text{We observe,} \mathcal{P}_{\psi_k}(\Tilde{\eta}_k) = \mathcal{P}_{\psi_k}(\eta_k -\mathcal{P}_{\sum_{i =1}^{k-1} V_i}( \eta_k))\nonumber\\
     & = \mathcal{P}_{\psi_k}(\eta_k) \;\;\;\;\; (\text{Since }\psi_k \perp \sum_{i =1}^{k-1} V_i)\\\nonumber
\end{align}
\vspace{-30pt}\\
        Now, consider the vector $\frac{\eta_k}{||\eta_k||} \in V_k$. Here $\frac{\eta_k}{||\eta_k||}$ is an unit vector in $V_k$, the span of a finite partition of spikes $v_k$, the size of which is bounded from above by some constant $d$ according to Claim \ref{partitionClaim}. Therefore, based on the assumption stated in the theorem \ref{windowthm}, we can infer that the norm of any projection of $\frac{\eta_k}{||\eta_k||}$ on any subspace spanned by a set of spikes other than those in $v_k$, would bounded from above by some constant $\beta_d >1$. Hence, we can write the following.
        \begin{align}
            &||\mathcal{P}_{\mathcal{S}(\{\psi_k\})+\sum_{i =1}^{k-1} V_i} (\frac{\eta_k}{||\eta_k||})||^2 \leq \beta_d^2 \;\;\; (|\beta_d| <1)\nonumber\\
            &\Rightarrow \frac{||\mathcal{P}_{\psi_k}(\eta_k)||^2 + ||\mathcal{P}_{\sum_{i =1}^{k-1} V_i}(\eta_k)||^2}{||\eta_k||^2} \leq \beta_d^2 \nonumber\\
            &\Rightarrow \frac{||\mathcal{P}_{\psi_k}(\eta_k)||^2 + ||\mathcal{P}_{\sum_{i =1}^{k-1} V_i}(\eta_k)||^2}
            {1+ ||\mathcal{P}_{\sum_{i =1}^{k-1} V_i}(\eta_k)||^2} \leq \beta_d^2 \;\;\;(\text{using } \ref{etaNorm})\nonumber\\
            &\Rightarrow ||\mathcal{P}_{\psi_k}(\eta_k)||^2 \leq \beta_d^2(1+||\mathcal{P}_{\psi_k}(\eta_k)||^2) - ||\mathcal{P}_{\psi_k}(\eta_k)||^2 \nonumber\\
            &\Rightarrow ||\mathcal{P}_{\psi_k}(\eta_k)||^2 \leq \beta_d^2 - (1- \beta_d^2)||\mathcal{P}_{\psi_k}(\eta_k)||^2 \nonumber\\
            &\Rightarrow ||\mathcal{P}_{\psi_k}(\eta_k)||^2 \leq \beta_d^2 \;\;\; (\text{Since} |\beta_d| <1)\nonumber\\
            &\Rightarrow ||\mathcal{P}_{\psi_k}(\Tilde{\eta}_k)||^2=||\mathcal{P}_{\psi_k}(\eta_k)||^2 \leq \beta_d^2 \nonumber\\
            &\Rightarrow ||\mathcal{P}_{\Tilde{\eta}_k}(\psi_k)||=||\mathcal{P}_{\psi_k}(\Tilde{\eta}_k)|| \leq \beta_d^2 \;\nonumber\\
             &(\text{Since both $\Tilde{\eta}_k$ and $\psi_k$ are unit vectors} )\nonumber\\
             &\text{Therefore, } || \Tilde{\psi}_k||^2 = ||\psi_k - \mathcal{P}_{\Tilde{V}_k}(\psi_k)||^2
            = 1- ||\mathcal{P}_{\Tilde{\eta}_k}(\psi_k)||^2 \;\nonumber\\
            &\;\;\;\;\;(\text{Since by assumption $\mathcal{P}_{\Tilde{V}_k}(\psi_k) = \beta_k \Tilde{\eta}_k $ and $||\psi_k||=1$})\nonumber\\
            &\Rightarrow || \Tilde{\psi}_k||^2 \geq 1-  \beta_d^2 \;\;\;\;\;\;\;\;(|\beta_d| < 1) \hspace{100pt}\hfill\Box \nonumber\\
            \nonumber
            \end{align}
            \vspace{-25pt}\\
Finally, combining \ref{geoDrop} with Claim \ref{psiknormboundClaim}, we obtain:
\vspace{-5pt}
        \begin{equation}            
         e_{k-1}^2 \geq e_k^2(1+ (1-  \beta_d^2)) \Rightarrow e_k^2 \leq \frac{e_{k-1}^2}{\gamma^2}\label{finalGeoSeq}
        \end{equation}
        \vspace{-13pt}
        \\
    where $\gamma^2 = (1+ (1-  \beta_d^2))$ is a constant strictly greater than 1. Since $e_m =0$, Eq. \ref{finalGeoSeq} shows that the sequence $\{e_k\}_{k=1}^m$ converges to 0 faster than geometrically. 
    Thus, $\forall \delta >0, \exists k_0 \in \mathbb{N}^+ \text{such that } e_k < \delta, \text{ for all } k\geq k_0 \text{ and } m\geq k$. Given the geometric drop in Eq. \ref{finalGeoSeq} and the bound $e_1 <1$ (Eq. \ref{boundary}), for arbitrarily large $m$ (hence $n$) the choice of $k_0$ is independent of $m$ and depends only on $\beta_d$, which is determined by the \textit{ahp} parameters in Eq. \ref{thresholdeq}). Since the number of spikes in each partition is bounded by d (Claim \ref{partitionClaim}), choosing a window size $w_0 = k_0*d$ ensures $\forall \delta >0, \exists w_0 \in \mathbb{N}^+  \text{ such that }  ||\phi^{\perp}_{n+1,w} -\phi^{\perp}_{n+1}|| < \delta \; \text{ for all } w \geq w_0 \text{ and } w \leq n$. For arbitrarily large $n$, the choice of $w_0$ is independent of $n$.$\hfill\Box$\\
    \textbf{Proof of Theorem \ref{windowthm}:} Having established a bound on the norm of the difference between \(\phi^{\perp}_{n+1}\) and \(\phi^{\perp}_{n+1,w}\), we now need to bound the norm of the difference between the projections of the input signal \(X\) with respect to these two vectors. Specifically, we seek to bound \(||\mathcal{P}_{\phi^{\perp}_{n+1,w}}(X) - \mathcal{P}_{\phi^{\perp}_{n+1}}(X)||\) based on the window size. We use the following notations:
    $$\Tilde{X} = \mathcal{P}_{\mathcal{S}(\{\phi^{\perp}_{n+1,w}, \phi^{\perp}_{n+1}\})}(X)$$
    $$X_u = \mathcal{P}_{\phi^{\perp}_{n+1}}(X)$$
    $$X_v = \mathcal{P}_{\phi^{\perp}_{n+1,w}}(X)$$
    So that $\Tilde{X}, X_u$ and $X_v$ lie in the same plane (see Fig. \ref{windowedProjection}). Assume the angle between $\Tilde{X}$ and $\phi^{\perp}_{n+1}\})$ is $a$, and the angle between $\Tilde{X}$ and $\phi^{\perp}_{n+1,w}\})$ is $b$. Hence, the angle between $\phi^{\perp}_{n+1,w}$ and $\phi^{\perp}_{n+1}$ is $a-b$ (see Fig. \ref{windowedProjection}). Note that the input $X$ may not necessarily lie in the same plane as $X_u$ and $X_v$. Hence, for our calculation, we consider the projection $\Tilde{X}$ of $X$ onto this plane. We also denote:
    \begin{align}
       & p_w =\phi^{\perp}_{n+1,w} - \phi^{\perp}_{n+1} \\
       & = (\phi_{n+1} - \mathcal{P}_{\mathcal{S}(\{\phi_{n-w+1}, ..., \phi_{n}\})}(\phi_{n+1})) \nonumber\\
       & \hspace{1.5cm} -(\phi_{n+1} - \mathcal{P}_{\mathcal{S}(\{\phi_{1}, ..., \phi_{n}\})}(\phi_{n+1})) \nonumber\\
       & \text{(by definition) and thus:}\nonumber\\
      & p_w = \mathcal{P}_{\mathcal{S}(\{\phi_{n-w+1}, ..., \phi_{n}\})}(\phi_{n+1}) \nonumber\\
      & \hspace{2cm}-  \mathcal{P}_{\mathcal{S}(\{\Tilde{\phi}_1,..., \Tilde{\phi}_w\} \cup \{\phi_{n-w+1}, ..., \phi_{n}\})}(\phi_{n+1}) \nonumber\\
       & =  \mathcal{P}_{\mathcal{S}(\{\Tilde{\phi}_1,..., \Tilde{\phi}_w\}}(\phi_{n+1}) \nonumber\\
        & \hspace{2cm}(\text{Since } \mathcal{S}(\{\Tilde{\phi}_1,..., \Tilde{\phi}_w\} \perp \mathcal{S}(\{\phi_{n-w+1}, ..., \phi_{n}\})\nonumber\\
        &= \mathcal{P}_{\mathcal{S}(\{\Tilde{\phi}_1,..., \Tilde{\phi}_w\}}(\phi_{n+1} -  \mathcal{P}_{\mathcal{S}(\{\phi_{n-w+1}, ..., \phi_{n}\})}(\phi_{n+1}))\nonumber\\
       & =  \mathcal{P}_{\mathcal{S}(\{\Tilde{\phi}_1,..., \Tilde{\phi}_w\}}(\phi^{\perp}_{n+1,w})\nonumber\\
           & \Rightarrow p_w = \phi^{\perp}_{n+1,w} - \phi^{\perp}_{n+1} = \mathcal{P}_{\mathcal{S}(\{\Tilde{\phi}_1,..., \Tilde{\phi}_w\}}(\phi^{\perp}_{n+1,w}) \nonumber\\
       & \Rightarrow p_w \perp \phi^{\perp}_{n+1}\nonumber\\
       &\text{(Since $\mathcal{P}_S(x) \perp (x- \mathcal{P}_S(x))$ for any $x \in \mathcal{H}, S \subseteq \mathcal{H}$)}\nonumber\\
       & \text{Therefore, from Figure \ref{windowedProjection}, we observe:}\nonumber\\
       & \sin^2(a-b) = \frac{||p_w||^2}{||\phiperp{w}||^2} \label{sinValue}\\ \nonumber
    \end{align} 
    \vspace{-20pt}\\
    Now, we can quantify the norm of the difference in two the projections of $X$ with respect to $\phicomp$ and $\phiperp{w}$ as follows:
    \begin{align}
      &||\mathcal{P}_{\phi^{\perp}_{n+1,w}}(X)- \mathcal{P}_{\phi^{\perp}_{n+1}}(X)|| = ||X_u - X_v||^2 \nonumber\\
      &=||X_u||^2 + ||X_v||^2 -2 ||X_u||||X_v||cos(a-b) \nonumber\\
      &=||\Tilde{X}||^2[ \cos^2a + \cos^2b - 2\cos a\cos b\cos(a-b) ] \nonumber\\
      & \hspace{1cm} (\text{using the geometry of Figure \ref{windowedProjection}})\nonumber\\
      &= ||\Tilde{X}||^2[ \frac{1+\cos2a}{2} + \frac{1+\cos2b}{2}  \nonumber\\
      &\hspace{2cm} -(\cos(a-b)+ \cos(a+b))\cos(a-b)] \nonumber\\
      & = ||\Tilde{X}||^2 [\frac{1}{2} - \frac{\cos2(a-b)}{2}] = ||\Tilde{X}||^2 \sin ^2 (a-b) \nonumber\\
      & = ||\Tilde{X}||^2\frac{||p_w||^2}{||\phiperp{w}||^2} \;\; (\text{from Eq. \eqref{sinValue}}) \nonumber\\
      & \leq \frac{||X||^2}{1- \beta ^2} ||\phiperp{w}-\phicomp||^2  \nonumber\\
      & (\text {Since $||\Tilde{X}|| \leq||X|| $ and $||\phiperp{w}|| \geq ||\phicomp|| \geq (1-\beta^2)$ } \nonumber\\
      & \text{by assumption)} \nonumber\\
      \nonumber
    \end{align}
    \vspace{-30pt}\\
     Finally, since the input signal has a bounded $L_2$ norm (i.e. $||X||$ is bounded),  setting $\delta = \epsilon\frac{\sqrt{1-\beta^2}}{||X||}$ in Lemma \ref{windowLemma} gives a window size $w_0$ such that $||\mathcal{P}_{\phi^{\perp}_{n+1,w}}(X)- \mathcal{P}_{\phi^{\perp}_{n+1}}(X)|| \leq \frac{||X||}{\sqrt{1- \beta ^2}} ||\phiperp{w}-\phicomp|| < \epsilon, \forall \epsilon >0, \forall w \geq w_0\text{ and } w\leq n $. The choice of $w_0$ is independent of $n$ for arbitrarily large $n \in \mathbb{N}$. This completes the proof of the theorem. $\hfill\Box$

}

 
%

\bibliographystyle{IEEEtran}
\bibliography{main.bib}

\vfill

\end{document}